\documentclass[11pt, letterpaper]{article}
\usepackage[T1]{fontenc}
\usepackage{subcaption}
\usepackage{arxiv}
\usepackage{times}
\usepackage{tikz}
\usepackage{adjustbox}
\usepackage{stmaryrd}
\usetikzlibrary{positioning, arrows.meta}
\definecolor{linkblue}{HTML}{3B78A6}

\hypersetup{
    colorlinks=true,
    linkcolor=linkblue,
    citecolor=linkblue,
    urlcolor=linkblue
}
\title{\LARGE Learning with Synthetic Data via SGD in High-Dimensional\\ Linear Regression}

\newcommand{\emailfont}[1]{{\fontfamily{ntxtt}\selectfont #1}}

\author{
Jichu Li \\
 Renmin University of China\\
\emailfont{lijichu52@gmail.com}
\and
Difan Zou \\
The University of Hong Kong\\
\emailfont{dzou@hku.hk}
}

\date{}

\date{}

\begin{document}

\maketitle

\begin{abstract}
Synthetic data has become a promising way to scale model training beyond limited human-generated data but it may also induce \emph{strong model collapse}~\citep{dohmatob2024strong}, where any fixed fraction of synthetic data prevents model performance from improving under data scaling, leaving a non-vanishing excess risk floor. In this paper, we study how synthetic data affects the generalization of one-pass SGD in high-dimensional linear regression with model shift. 
We establish finite-sample risk bounds for mixed and two-stage training, separating standard bias and variance from source-mismatch effects, namely fluctuation and persistent drift under mixing and filtered initialization bias under two-stage. These bounds reveal a sharp contrast: mixed training induces strong model collapse, while two-stage training avoids the floor by using synthetic data only in the first stage, showing that collapse is not inevitable under a simple data curriculum. 
Under a random sketch model, we further obtain scaling laws for both protocols, with tight results for mixed training in the optimization-saturated regime. These laws show that larger models may amplify synthetic-induced degradation under mixing, and quantify how high-quality synthetic pretraining may reduce bias in two-stage training. 
Finally, we establish an exact finite-sample necessary-and-sufficient condition for two-stage training to strictly outperform real-only training under the same real-data budget and identical real-stage updates. Overall, our results highlight that synthetic data is neither inherently harmful nor beneficial; its effect depends critically on both its quality and the training protocol used to incorporate it.
\end{abstract}

\section{Introduction}

Large language models (LLMs) have demonstrated remarkable capabilities, primarily driven by large-scale, high-quality pre-training data~\citep{radford2019language,brown2020language}. Scaling laws suggest that increasing both model size and data volume consistently leads to improved performance~\citep{kaplan2020scaling,hoffmann2022training,muennighoff2023scaling}. However, high-quality pre-training data is limited, and its supply is rapidly diminishing, posing a significant challenge for continued model scaling~\citep{muennighoff2023scaling,villalobos2024position}. Consequently, synthetic data, often generated by existing models, has emerged as a promising supplement to human-generated data for training, particularly in scenarios where labeled data is scarce~\citep{muennighoff2023scaling,chen2024diversity,long2024llms,havrilla2024surveying}. The training paradigm thus evolves to operate on a mixture of real data from $\cP_1$ and synthetic data from $\cP_2$, while our ultimate interest remains performance on the real data distribution $\cP_1$ alone.

This raises a natural question: Does synthetic data truly benefit model training, and what limitations or adverse effects might arise? On the one hand, a large body of work has demonstrated the utility of synthetic data. By amplifying real datasets, synthetic samples effectively increase the total training volume, enabling models to learn richer patterns and improve generalization~\citep{gunasekar2023textbooks,li2023textbooks,liu2024best,maini2024rephrasing,long2024llms,abdin2024phi,qin2025scaling}. On the other hand, \citet{shumailov2023curse} showed that training with synthetic data can induce \textit{model collapse}, a phenomenon where the iterative use of generated data for training leads to a critical degradation in model performance. This phenomenon arises as the model gradually overfits to patterns found in synthetic data, which may not fully represent the richness or variability of real-world data, ultimately leading to a deterioration of model quality over time, as extensively validated by prior empirical works~\citep{hataya2023will,bohacek2023nepotistically,alemohammad2023self,guo2024curious,briesch2024largelanguagemodelssuffer,shumailov2024ai}. Understanding training with synthetic data has become central to large-scale model development.

Theoretical studies of learning with synthetic data typically formulate the problem as a regression task under distribution mismatch between real and synthetic data. Some of them focus on the phenomenon of model collapse, showing how training with synthetic data can degrade model performance~\citep{dohmatob2024tale,dohmatob2024model,seddik2024bad,jain2024scaling} while other studies investigate strategies to mitigate model collapse or to effectively leverage synthetic data to improve learning~\citep{seddik2024bad,firdoussi2024maximizing,barzilai2025models,garg2025preventing,ferbach2024self}. Of particular relevance to our work, \citet{dohmatob2024strong} introduce the concept of \textit{strong model collapse}: even a small proportion of synthetic data prevents model performance from improving under data scaling, leaving a non-vanishing excess risk floor, which serves as a key motivation for our analysis. However, most of these works focus on static estimators and asymptotic regimes, without capturing the dynamics of practical training procedures based on iterative optimization algorithms (e.g., SGD), nor providing a detailed finite-sample characterization. 

In this paper, we study the effect of synthetic data on model generalization through the lens of one-pass SGD in high-dimensional linear regression. We consider a two-source setting where real and synthetic data share the same covariate distribution but differ in their labeling functions, capturing common synthetic-supervision pipelines where real inputs are labeled by a biased teacher, and analyze SGD under both mixed and two-stage training protocols. We establish finite-sample upper and lower bounds showing that mixed training leads to \emph{strong model collapse} under SGD, while a simple two-stage protocol avoids this issue and achieves vanishing excess risk. Further, we derive scaling laws under a random sketch model, characterizing how model size may affect synthetic-induced degradation and how synthetic data may benefit learning through bias reduction. Finally We give an exact finite-sample condition for two-stage training to strictly outperform real-only training with the same real-data budget and real-stage updates. \newline
\noindent Our contributions can be summarized as follows:
\begin{itemize}[leftmargin=*]
    \item In Section~\ref{sec:mix-train}, we derive finite-sample upper and lower bounds for last-iterate one-pass SGD under mixed training. The upper bound decomposes the synthetic-data-induced degradation into $\mathsf{fluctuation}$ and $\mathsf{drift}$, characterizing how source mismatch affects optimization dynamics. Together, the upper and lower bounds establish a non-vanishing excess risk floor under data scaling, which is exactly \emph{strong model collapse}. Moreover, we show that the same behavior persists for iterate-averaged SGD, indicating that strong model collapse is fundamentally induced by sample mixing in the training protocol, rather than an artifact of a specific optimization scheme.
    \item In Section~\ref{sec:two-stage-train}, we show that strong model collapse is not inevitable. We analyze a two-stage training protocol, where the model is first trained on synthetic data and then on real data, and show that this simple data curriculum avoids the non-vanishing excess risk floor. This highlights that the ordering and scheduling of data can be as important as the data itself. 
    \item In Section~\ref{sec:scaling-law}, we further establish scaling-law bounds for the two training protocols under a random sketch model, with tight results for mixed training in the optimization-saturated regime, showing that larger models may amplify synthetic-data-induced degradation in mixed training, and providing an explicit characterization of how high-quality synthetic data may reduce bias in two-stage training, thereby benefiting learning.
    \item In Section~\ref{sec:synthetic_benefit}, we establish an exact finite-sample condition under which two-stage training strictly outperforms real-data-only training with a fixed real-data budget and identical real-stage training; we also use a one-dimensional toy model to illustrate the roles of synthetic sample size, source mismatch, and label noise.
\end{itemize}

\paragraph{Notations.}
For two positive-valued functions $f(x)$ and $g(x)$, we write  $f(x)\lesssim g(x)$ (and $f(x)= \mathcal{O} (g(x))$) or $f(x)\gtrsim g(x)$ (and $f(x) = \Omega(g(x))$) if $f(x) \le cg(x)$ or $f(x) \ge cg(x)$ holds for some absolute (if not otherwise specified) constant $c>0$ respectively. 
We write $f(x) \eqsim g(x)$ (and $f(x)=\Theta(g(x))$) if $f(x) \lesssim g(x) \lesssim f(x)$.
For two vectors $\ub$ and $\vb$ in a Hilbert space, we denote their inner product by $\langle\ub, \vb\rangle$ or  $\ub^\top \vb$.
For two matrices $\Ab$ and $\Bb$ of appropriate dimensions, we define their inner product by$\langle \Ab, \Bb \rangle := \tr(\Ab^\top \Bb)$.
We use $\|\cdot\|$ to denote the operator norm for matrices and $\ell_2$-norm for vectors.
For a positive semi-definite (PSD) matrix $\Ab$ and a vector $\vb$ of appropriate dimension, we write $\|{\vb}\|_{\Ab}^2 := \vb^\top \Ab \vb$. Kronecker/tensor product is denoted by $\otimes$.

\section{Related Work}
\paragraph{Synthetic data and model collapse.}
Training with synthetic data has been shown to induce \textit{model collapse}~\citep{shumailov2023curse}, a phenomenon in which model performance degrades significantly when consistently trained with poor-quality synthetic data, as has been empirically verified by a number of subsequent works~\citep{hataya2023will,bohacek2023nepotistically,alemohammad2023self,guo2024curious,briesch2024largelanguagemodelssuffer,shumailov2024ai}. Theoretical studies of model collapse typically frame it as a regression problem under distribution mismatch between real and synthetic data~\citep{dohmatob2024tale,dohmatob2024model,seddik2024bad}, showing how synthetic data degrade performance, while others explore how to avoid collapse or leverage synthetic data to benefit learning~\citep{seddik2024bad,firdoussi2024maximizing,barzilai2025models,garg2025preventing}. In particular, \citet{dohmatob2024strong} introduce the concept of \textit{strong model collapse}, where even a small proportion of synthetic data prevents model performance from improving under data scaling, leaving a non-vanishing excess risk floor. Moreover, such collapse cannot be mitigated by simple strategies such as data reweighting (\citep{jain2024scaling,ferbach2024self}) unless the proportion of synthetic data vanishes asymptotically. However, most of them do not capture the dynamics of practical training, where the order and interplay of data sources matter. In this paper, we study finite-sample one-pass SGD and show that while mixed training induces strong model collapse, it can be avoided by a simple two-stage training protocol.

\paragraph{Learning with multiple data sources.}
Learning from multiple data sources has been widely studied in transfer learning, domain adaptation, and continual learning under covariate or model shift~\citep{blitzer2007learning,ben2010theory,wang2015generalization,lei2021generalization,kpotufe2018marginal,hanneke2019value,pathak2022new,li2023fixed,tahir2024features,li2025memory,kim2025transfer}. However, most of them focus on statistical estimators rather than optimization dynamics. More recently, several works have analyzed SGD under covariate shift. \citet{wu2022power} study pretraining-finetuning under covariate shift via SGD and establish excess risk bounds. \citet{ding2024understanding} extend the analysis to continual learning settings and show how covariate shifts across tasks affect SGD dynamics and forgetting. \citet{liu2025optimal} characterize the minimax optimality of SGD-type methods for high-dimensional linear regression under covariate shift. Beyond standard covariate shift, \citet{deng2025mixed} propose an adaptive mixed-sample SGD that further accounts for model shift while achieving transfer-optimal target risk. In this paper, we consider different data sources with shared covariates but mismatched labeling functions, as commonly arises with synthetic data, and characterize how synthetic data affects finite-sample SGD dynamics.   

\paragraph{SGD in high-dimensional linear regression.}
The generalization of stochastic gradient descent (SGD) in linear regression has been extensively studied. In the classical underparameterized regime, a large number of works have studied its behavior~\citep{polyak1992acceleration,bach2013non,defossez2015averaged,dieuleveut2017harder,jain2017markov,jain2018parallelizing,pillaud2018statistical}. One-pass SGD under different iterate schemes (e.g., last iterate, averaging) has also been rigorously studied in the overparameterized setting~\citep{dieuleveut2016nonparametricstochasticapproximationlarge,ge2019step,berthier2020tight,varre2021last,wulast,zou2023benign,zhang2024optimality,zhang2025learning}, providing framework for analyzing the generalization of SGD in high-dimensional regime. In parallel, Several works have extended the analysis to multipass SGD~\citep{lin2017optimal,pillaud2018statistical,mucke2019beating,lei2021generalization,zou2022risk}. More recently, a line of research has focused on the theoretical scaling laws of SGD in high-dimensional linear regression. \citet{linscaling} analyzed the last-iterate test error of one-pass SGD in a random sketch model, providing the first systematic derivation of a finite-sample joint scaling law. \citet{lin2025improved} further extended their analysis to multi-epoch SGD, followed by an even finer-grained characterization from \citet{yan2025larger}. \citet{li2025functional} established functional scaling laws and analyzed optimal learning rate and batch size schedules~\citep{li2026optimal,wang2026fast}. Overall, this paper builds on high-dimensional SGD and scaling-law analyses, and extend them from single-source learning to a two-source setting where source mismatch induces new drift and fluctuation effects.

\section{Preliminaries}\label{sec:setup}
\paragraph{Problem Setup.} We use $\xb \in \cH$ to denote a feature vector, where $\cH$ is a finite d-dimensional or countably infinite dimensional Hilbert space, and $y \in \mathbb{R}$ to denote its label.  \textit{Linear Regression} concerns the following objective:
\[
\min_{\wb} \R(\wb), \ \text{where}\ \R(\wb):= \frac{1}{2}\EE\big(\langle\xb,\wb\rangle-y\big)^2
\]
Here $\wb \in \cH$ is a weight vector to be learned and the expectation is over $(\xb,y)\sim \cP$ for some distribution $\cP$ on $\cH \times \mathbb{R}$.

Following the prior literature \citep{bartlett2020benign,wulast,zou2023benign,linscaling}, we make the following assumptions on the data distribution.

\begin{assumption}[Data Covariance]\label{assump:second-moment}
Let $\Hb := \EE[\xb \xb^\top]$ be the data covariance matrix and assume that $\tr(\Hb)$ and all entries of $\Hb$ are finite. In addition, we assume that all parameter vectors involved in the analysis have finite $\Hb$-norm, i.e., $\|\wb\|_{\Hb} < \infty$.
\end{assumption}

\begin{assumption}[Fourth moment conditions]\label{assump:fourth-moment}
(1)\label{item:fourth-moement-upper} There exists constant $\alpha > 0$ such that for every PSD matrix $\Ab$ it holds that $\EE[\xb \xb^\top \Ab \xb \xb^\top] \preceq \alpha  \tr (\Hb \Ab) \Hb$; (2)\label{item:fourth-moement-lower} There is a constant $\beta > 0$, such that for every PSD matrix $\Ab$, we have $\EE [\xb \xb^\top \Ab \xb \xb^\top] - \Hb \Ab \Hb \succeq \beta  \tr (\Hb \Ab)  \Hb$.
\end{assumption}
In this paper, motivated by \citep{dohmatob2024strong}, we consider a two-source setting, where data are drawn from either a real distribution $\cP_1$ or a synthetic distribution $\cP_2$. We assume that both distributions share the same feature marginal $\xb \sim \cD$, and differ only in their labeling functions, which allows us to isolate the effect of model mismatch and obtain a clean characterization of the learning dynamics. This formulation naturally captures the prevalent self-training and knowledge distillation paradigms, where a massive pool of unlabeled real-world features is annotated by a potentially biased teacher model to generate supervisory signals.

\begin{assumption}[Well-specified models]\label{assump:noise}
For $\ell \in \{1,2\}$, data from distribution $\cP_\ell$ follow a linear model $y = \langle \xb, \wb^*_\ell \rangle + \xi_\ell,$
where $\wb^*_\ell \in \cH$ is the underlying parameter, and the noise $\xi_\ell \sim N(0, \sigma_\ell^2)$ is independent of $\xb$. Define the model shift between the two sources as $\bm{\delta} := \wb^*_2 - \wb^*_1.$
\end{assumption}
While training may involve samples from both $\cP_1$ and $\cP_2$, our goal is to minimize the population risk on the real distribution $\R(\wb) := \frac{1}{2}\EE_{(\xb,y)\sim \cP_1}(\langle \xb, \wb \rangle - y)^2.$

\paragraph{Training protocol.} We consider two training protocols for combining real and synthetic data.

\textit{(1) Mixed training.}
We are given a dataset consisting of $M+N$ samples, among which $M$ are drawn from the synthetic distribution $\cP_2$ and $N$ from the real distribution $\cP_1$. Let $p := M/(M+N)$ denote the fraction of synthetic data. The training samples are presented in a random order. An equivalent fixed-budget formulation key to the proof is detailed in Appendix~\ref{ap:mixed-sampling}.

\textit{(2) Two-stage training.}
We first train on $M$ synthetic samples drawn from $\cP_2$, followed by training on $N$ real samples drawn from $\cP_1$.

\paragraph{SGD iteration.}
We consider training via \textit{stochastic gradient descent} (SGD) with step sizes $\{\gamma_t\}$:
\[
\wb_{t+1} = \wb_t - \gamma_t \big( \langle \xb_t, \wb_t \rangle - y_t \big) \xb_t\quad t\ge1 .
\]
where $\{(\xb_t,y_t)\}_{t=1}^{M+N}$ follow one of the two protocols above, and $\{\gamma_t\}_{t=1}^{M+N}$ are the step sizes. We adopt a well-used geometric decaying stepsize scheduler \citep{wulast,wu2022power,linscaling,lin2025improved}:
\begin{equation}\label{eq:geometry-tail-decay-lr}
    \text{for } t = 1, \dots, T, \quad \gamma_t := \gamma/2^{\ell_t}, \text{ and } \ell_t = \lfloor t/(T/\log(T)) \rfloor. 
\end{equation}
Here the initial stepsize $\gamma$ is a hyperparameter for the algorithm and $T$ denotes the total number of iterations. In the mixed-training setting, we take $T = M+N$, while in the two-stage setting the same schedule is applied separately within each stage.

The output of the algorithm is the last iterate $\wb_{M+N}$ and we study the \textit{excess risk} on the real distribution:
\[
\mathsf{Excess} := \mathcal{E}_1(\wb_{M+N})=\R(\wb_{M+N}) - \R(\wb^*_1).
\]

\paragraph{Additional Notations.} let $\Hb=\sum_i \lambda_i\vb_i\vb_i^\top$ be the eigen-decomposition of $\Hb$, where $\{\lambda_i\}^{\infty}_{i=1}$ are the eigenvalues of $\Hb$ sorted in non-increasing order and $\vb_i$ are the corresponding eigenvectors. Following \citep{zou2023benign,wu2022power}, we denote:
\begin{equation*}
\Hb_{0:k}
:=
\sum_{i=1}^k \lambda_i \vb_i \vb_i^\top,
\quad
\Hb_{k:\infty}
:=
\sum_{i>k} \lambda_i \vb_i \vb_i^\top,
\quad
\bm{\delta}
=
\sum_i \delta_i \vb_i.
\end{equation*}

\section{Main Results}
\subsection{Mixed training}\label{sec:mix-train}
In this subsection, we study SGD trained on a mixture of real and synthetic data. We show that, in this setting, the presence of synthetic data induces a non-vanishing excess risk floor, leading to the phenomenon of strong model collapse~\cite{dohmatob2024strong}. Detailed proof is deferred to Appendix~\ref{ap:mix-upper} and~\ref{ap:mix-lower}.

\begin{theorem}[An upper bound for mixed training]\label{thm:generalization_error}
Given $M$ synthetic and $N$ real samples. Consider last iterate SGD with stepsize scheme \eqref{eq:geometry-tail-decay-lr}.
Suppose Assumptions~\ref{assump:second-moment}, ~\ref{assump:fourth-moment} and ~\ref{assump:noise} hold.
Let $ \widetilde{N}_{\mathtt{eff}}:= (M+N) / \log (M+N)$ and $\bar \sigma^2=(1-p)\sigma_1^2+p\sigma_2^2$.
Suppose $\gamma < 1/(4\alpha\tr(\Hb))$. We have 

\begin{align*}
    \mathbb{E} [\mathcal{E}_1({\mathbf{w}}_{M+N})] & \lesssim \underbrace{\big\| \prod_{t=1}^{M+N}(\mathbf{I}-\gamma_t \mathbf{H}) (\mathbf{w}_0 - \mathbf{w}_1^*) \big\|^2_{\mathbf{H}} +\alpha\|\mathbf{w}_0 - \mathbf{w}_1^* \|^2_{\frac{\mathbf{I}_{0:k^*}}{\gamma \widetilde{N}_{\mathtt{eff}}}+\mathbf{H}_{k^*:\infty}}\frac{\widetilde{D}_{\mathtt{eff}}}{\widetilde{N}_{\mathtt{eff}}}}_{\mathsf{BiasError}}\\
    &+\underbrace{\bar \sigma^2 \frac{\widetilde{D}_{\mathtt{eff}}}{\widetilde{N}_{\mathtt{eff}}}}_{\mathsf{VarError}}
    +  \underbrace{\alpha\, p \|\bm{\delta}\|^2_{\mathbf{H}} \frac{\widetilde{D}_{\mathtt{eff}}}{\widetilde{N}_{\mathtt{eff}}}}_{\mathsf{FlucError}}
    + \underbrace{ p^2\Big\| \big(\mathbf{I}-\prod_{t=1}^{M+N}(\mathbf{I}-\gamma_t\mathbf{H})\big)\bm{\delta}\Big\|^2_{\mathbf{H}}}_{\mathsf{DriftError}},
\end{align*}
where  $k^* = \max \{k: \lambda_k \ge \frac{1}{\gamma \widetilde{N}_{\mathtt{eff}}}\}$ and $\widetilde{D}_{\mathtt{eff}} := k^* + \gamma^2 \widetilde{N}_{\mathtt{eff}}^2 \sum_{i>k^*}\lambda_i^2.$
\end{theorem}

\textbf{Interpretation.} This upper bound decomposes the excess risk into four terms: a bias term, a variance term, and two additional terms induced by the presence of synthetic data, which we refer to as the fluctuation error and the drift error. The bias and variance terms match those in classical analyses of SGD with geometrically decaying step sizes \citep{wu2022power}, capturing the optimization error and stochastic noise, respectively. The fluctuation error is upper bounded by terms involving $p(1-p),p^2$ and $p$, and is therefore summarized at the order level as $\mathcal{O}(p \|\bm{\delta}\|_{\Hb}^2 \cdot \frac{\widetilde{D}_{\mathtt{eff}}}{\widetilde{N}_{\mathtt{eff}}}).$ It arises from the stochastic mixing of the two data sources. Particularly, when $\|\bm{\delta}\|^2_{\Hb}\eqsim\bar\sigma^2$, this term is of the same order as the standard variance term and behaves similarly to it. In contrast, the drift error
$p^2 \| (\mathbf{I}-\prod_{t=1}^{M+N}(\mathbf{I}-\gamma_t\mathbf{H}))\bm{\delta}\|^2_{\Hb}$
captures a systematic bias in the optimization trajectory toward the synthetic model, which reflects the accumulation of updates aligned with $\wb^*_2$, leading the iterate to drift away from the target $\wb^*_1$.

Together, these two terms reveal a fundamental trade-off: while incorporating synthetic data can increase sample size to reduce bias and variance, it also introduces a bias proportional to the model mismatch $\bm{\delta}$, which may dominate and degrade performance. Further, when the total sample size grows to infinity with a fixed synthetic proportion p, all terms except the drift term vanish under the condition $\widetilde{D}_{\mathtt{eff}}=o(\widetilde{N}_{\mathtt{eff}})$. As a result, the upper bound converges to a non-zero floor of order $p^2\|\bm{\delta}\|^2_{\Hb}$. Next, we complement this upper bound with a lower bound, showing that this excess risk floor cannot be avoided in general.

\begin{theorem}[A lower bound for mixed training]\label{thm:tail-decay-lower-bound}
Given $M$ synthetic and $N$ real samples. Consider last iterate SGD with stepsize scheme \eqref{eq:geometry-tail-decay-lr}.
Suppose Assumptions~\ref{assump:second-moment}, ~\ref{assump:fourth-moment} and ~\ref{assump:noise} hold.
Let $ \widetilde{N}_{\mathtt{eff}}:= (M+N) / \log (M+N)$ and $\bar \sigma^2=(1-p)\sigma_1^2+p\sigma_2^2$.
Suppose $M+N \ge 500$ and $\gamma < 1 / \lambda_1$, then 
\begin{align*}
    \EE [\mathcal{E}_1(\wb_{M+N}) ] & \gtrsim \underbrace{ \Big\| \prod_{t=1}^{M+N}(\Ib-\gamma_t\Hb) (\wb_0 - \wb_1^*) +p\big(\Ib-\prod_{t=1}^{M+N}(\Ib-\gamma_t\Hb)\big)\bm{\delta}\Big\|^2_{\Hb}}_{\mathsf{Non-vanishing\ term}} \\
    &+\big(\beta\norm{\wb_0-\wb^*_1-p\bm{\delta}}_{\Hb_{k^*:\infty}}^2+\bar \sigma^2\big) \frac{\widetilde{D}_{\mathtt{eff}}}{\widetilde{N}_{\mathtt{eff}}} + p(1-p)\gamma^2 \widetilde{N}_{\mathtt{eff}} \norm{\bm{\delta}}_{\Hb^3_{k^*:\infty}}^2
\end{align*}
where  $k^* = \max \{k: \lambda_k \ge \frac{1}{\gamma \widetilde{N}_{\mathtt{eff}}}\}$ and $\widetilde{D}_{\mathtt{eff}} := k^* + \gamma^2 \widetilde{N}_{\mathtt{eff}}^2 \sum_{i>k^*}\lambda_i^2.$
\end{theorem}

\textbf{Interpretation.} The lower bound complements Theorem~\ref{thm:generalization_error} by showing that the excess risk cannot vanish in the large-sample regime. In particular, its first term captures the mean optimization error, jointly determined by the residual initialization error and the drift toward the synthetic model. At finite sample sizes, the spectral filter determines which directions remain dominated by the initialization and which have accumulated synthetic-induced bias. The remaining terms further quantify label-noise variance and mismatch-induced fluctuations, providing a more complete finite-sample characterization. Comparing with the upper bound, we observe that both bounds contain a term governed by the quantity $p\big(\Ib-\prod_{t=1}^{M+N}(\Ib-\gamma_t\Hb)\big)\bm{\delta},$ which characterizes the drift of the optimization trajectory toward the synthetic model. While the upper bound isolates this effect as the drift error, the lower bound shows that such a term is unavoidable as the initialization is progressively forgotten. Crucially, when the total sample size scales to infinity with a fixed synthetic proportion $p$, and under the condition $\widetilde{D}_{\mathtt{eff}}=o(\widetilde{N}_{\mathtt{eff}})$, all remaining terms vanish, and both bounds converge to the same non-zero limit of order $p^2 \|\bm{\delta}\|_{\Hb}^2.$ This establishes that the excess risk admits an non-vanishing floor, which is entirely induced by the mismatch between real and synthetic data.

This directly leads to the following corollary, which formalizes the phenomenon of \textit{strong model collapse}\citep{dohmatob2024strong}.
\begin{corollary}[Strong model collapse]\label{cor:strong-model-collapse}
Consider last iterate SGD with stepsize scheme \eqref{eq:geometry-tail-decay-lr}.
Suppose Assumptions~\ref{assump:second-moment}, ~\ref{assump:fourth-moment} and ~\ref{assump:noise} hold. Suppose $\gamma < 1 / 4\alpha\tr(\Hb)$ and $\widetilde{D}_{\mathtt{eff}}=o(\widetilde{N}_{\mathtt{eff}})$. As the total sample size $M+N$ scales to infinity with a fixed synthetic proportion $p=M/(M+N)$, we have
\[
\lim_{M+N\xrightarrow{}\infty}\EE[\mathcal{E}_1(\wb_{M+N})] \eqsim p^2\|\bm{\delta}\|_{\Hb}^2.
\]
\end{corollary}

Corollary~\ref{cor:strong-model-collapse} shows that the strong model collapse phenomenon persists under SGD with decaying stepsizes. In particular, even when the total sample size $M+N$ tends to infinity, as long as the proportion of synthetic data $p$ remains non-vanishing, the excess risk does not converge to zero, but instead admits a strictly positive limit of order $p^2 \|\bm{\delta}\|_{\Hb}^2$. This demonstrates that strong model collapse is not merely a property of batch estimators or scaling-law limits, but also arises intrinsically from the optimization dynamics of SGD. Here the condition $\widetilde{D}_{\mathtt{eff}}
=o(\widetilde{N}_{\mathtt{eff}})$ requires the effective spectral dimension to grow sublinearly with the effective sample size, and holds for many standard eigendecay regimes; see \citep{zou2023benign,wulast}.

Our result should be compared with \citet{dohmatob2024strong}, who established strong model collapse for regression models trained on mixtures of real and synthetic data. In their classical linear estimator setting, the error floor is of order $p^2 \EE[\|\wb_2^*-\wb_1^*\|_\Hb^2].$ Hence, their expression can be viewed as taking an additional expectation over the model mismatch, while ours characterizes the floor for a fixed $\bm{\delta}$. In this sense, our result is fully consistent with the estimator-level picture of \citet{dohmatob2024strong}, but recast in stochastic optimization dynamics. However, a key distinction is that their analysis focuses on classical batch estimators (ridge/OLS) under a proportional high-dimensional asymptotic regime, and the resulting error floor is an estimator-level asymptotic characterization. In contrast, we analyze SGD under finite-sample training and establish non-asymptotic upper and lower bounds. Our corollary is also asymptotic, but it follows from these finite-sample bounds, showing that the same non-vanishing floor arises from the optimization dynamics of SGD.

More importantly, we show that \textit{strong model collapse} is fundamentally induced by sample mixing in the training protocol, rather than an artifact of the optimization scheme. In particular, the same phenomenon arises for constant stepsize SGD with iterate averaging, where $\gamma_t=\gamma$ and the output is $\bar\wb_{M+N}=\frac{1}{M+N}\sum_{i=0}^{M+N-1}\wb_i$. In this case, we have
$
\lim_{M+N\xrightarrow{}\infty}\EE[\mathcal{E}_1(\bar\wb_{M+N})] \eqsim p^2\|\bm{\delta}\|_{\Hb}^2,
$
demonstrating that the same non-vanishing error floor persists; see Appendix~\ref{ap:constant-sgd} for details.

This naturally raises an important question:
\[
\textit{\textbf{Q:} Is strong model collapse inevitable with fixed synthetic data proportion?}
\]
In the next subsection, we answer this question by showing that it can in fact be avoided. Specifically, we demonstrate that a two-stage training protocol eliminates the non-vanishing error floor, leading to vanishing excess risk, highlighting the importance of how synthetic and real data are scheduled.

\subsection{Two-stage Training}\label{sec:two-stage-train}
In this subsection, we study a two-stage training protocol, where the model is first trained on synthetic data and then further trained on real data. We show that, in contrast to the mixed-training setting, this training scheme avoids strong model collapse and leads to vanishing excess risk under mild conditions. The detailed proof is deferred to Appendix~\ref{ap:two-stage}.

\begin{theorem}[Two-stage training]\label{thm:two-stage}
Given $M$ synthetic samples and $N$ real samples. Consider last iterate SGD with stepsize scheme \eqref{eq:geometry-tail-decay-lr}.
Suppose Assumptions \ref{assump:second-moment},~\ref{assump:fourth-moment} and~\ref{assump:noise} hold.
Let $\Neff := N / \log N$ and $\mathcal{E}_{2}(\wb) := \frac{1}{2}\EE_{(\xb,y)\sim\cP_2}[(\langle\xb,\wb\rangle - y)^2] - \frac{1}{2}\EE_{(\xb,y)\sim\cP_2}[(\langle\xb,\wb_2^*\rangle - y)^2]$ denote the excess risk on the synthetic distribution.
Suppose $\gamma < 1/(4\alpha\tr(\Hb))$. Then we have 
\begin{align*}
\EE[\mathcal{E}_1(\wb_{M+N})]
&\lesssim
\underbrace{\vphantom{\prod_{t=1}^{N} \frac{\DIM}{\Neff}} \EE[\mathcal{E}_2(\wb_{M})]
+
\Big\|\prod_{t=1}^{N}(\mathbf{I}-\gamma_t \mathbf{H})\bm{\delta}\Big\|_{\Hb}^2}_{\mathsf{EffectiveBias}}
+
\underbrace{\vphantom{\prod_{t=1}^{N} \frac{\DIM}{\Neff}} \Big(
\alpha\,\EE[\mathcal{E}_2(\wb_{M})]
+
\alpha\|\bm{\delta}\|_{\Hb}^2
+
\sigma_1^2
\Big)\frac{\DIM}{\Neff}}_{\mathsf{EffectiveVariance}}.
\end{align*}
where  $k^* = \max \{k: \lambda_k \ge \frac{1}{\gamma \Neff}\}$ and $\DIM := k^* + \gamma^2 \Neff^2 \sum_{i>k^*}\lambda_i^2.$
\end{theorem}

As a direct consequence of Theorem \ref{thm:two-stage}, by considering the asymptotic regime where the sample size goes to infinity, we can establish that strong model collapse does not occur:
\begin{corollary}[No strong model collapse]
\label{cor:no-collapse-two-stage}
Under the assumptions of Theorem~\ref{thm:two-stage}. As the total sample size $M+N$ scales to infinity with a fixed synthetic proportion $p=M/(M+N)$, suppose that $\DIM=o(\Neff)$, then $\EE[\mathcal{E}_2(\wb_{M})]\xrightarrow{M\xrightarrow{}\infty}0,$ and we have
\[
\lim_{M+N\to\infty}\EE[\mathcal{E}_1(\wb_{M+N})]=0.
\]
\end{corollary}

Corollary~\ref{cor:no-collapse-two-stage} shows that, in sharp contrast to the mixed-training setting, strong model collapse can be avoided under a two-stage training protocol. The key distinction is that synthetic data no longer appears throughout the entire optimization trajectory. Instead, the synthetic stage only determines the initialization of the second stage, after which all subsequent updates are performed on real data. As a result, the synthetic model does not induce a persistent drift in the trajectory. Broadly, this suggests that successfully leveraging synthetic data may hinge crucially on the training curriculum, as the specific ordering and scheduling of data may be just as important as the data itself. 

Further, Theorem~\ref{thm:two-stage} sheds light on the mechanism through which synthetic data can be beneficial. The main benefit does \emph{not} come from reducing variance: the $\mathsf{EffectiveVariance}$ term still scales as $\DIM/\Neff$, which is the same order as in the standard single-source SGD bound trained on real data only \citep{wu2022power}. Instead, the main benefit of synthetic data lies in reducing the \emph{bias}. First, the synthetic training phase error $\EE[\mathcal{E}_2(\wb_M)]$ can be driven to zero as $M$ increases. Second, the residual mismatch term
$\|\prod_{t=1}^{N}(\mathbf{I}-\gamma_t \mathbf{H})\bm{\delta}\|_{\Hb}^2$
replaces the real-only initialization term
$\|\prod_{t=1}^{N}(\mathbf{I}-\gamma_t \mathbf{H})(\wb_0-\wb_1^*)\|_{\Hb}^2.$
Thus, two-stage training effectively replaces random initialization error by the source mismatch $\bm{\delta}=\wb_2^*-\wb_1^*$. When the synthetic data is of high quality so that $\|\bm{\delta}\|_{\Hb}^2$ is small, first-stage places the model substantially closer to the target $\wb_1^*$, yielding a smaller bias after second-stage training.

Consequently, synthetic data primarily acts as a \emph{bias-reduction mechanism}: it improves the initialization for the real-data optimization stage, thereby accelerating convergence and reducing the optimization error when real data is limited. This highlights that the effectiveness of synthetic data critically depends on its quality (i.e., the magnitude of $\bm{\delta}$), as well as on the training protocol used to incorporate it. In the next section, we make this intuition more precise by analyzing a random sketch model under a power-law eigendecay assumption on the data covariance, which allows us to characterize more explicitly when and how synthetic data improves performance.

\textbf{Learning-rate restart.} Throughout the main text, including the scaling-law analysis below, we restart the learning-rate schedule to treat the real-data phase as a separate fine-tuning stage, with a larger cumulative stepsize that accelerates forgetting of the synthetic initialization. Without restart, the mismatch filter becomes
$\prod_{t=M+1}^{M+N}(\Ib-\gamma_t\Hb),$
which generally forgets $\bm{\delta}$ more slowly. Nevertheless, restart does not affect the qualitative behavior: for fixed
$p\in(0,1)$, the cumulative real-stage stepsize diverges and the excess risk still has no non-vanishing floor; see detailed discussion in Appendix~\ref{ap:two-stage-no-restart}.

\subsection{Random Sketch Model \& Scaling laws}\label{sec:scaling-law}

The results in the previous subsections establish a sharp qualitative distinction between mixed training and two-stage training. However, these results do not yet reveal how the effect of synthetic data interacts with model size. To study this question and establish scaling laws for the two training protocols, we adopt the random sketch model with Gaussian sketch matrix~\citep{linscaling,lin2025improved,zhang2026scaling}. Detailed proof of this subsection is deferred to Appendix~\ref{ap:random-sketch}.

Concretely, following \citet{linscaling}, we introduce a Gaussian random sketch operator $\Sb:\cH\to\mathbb{R}^D$, where entries of $\Sb$ are independently sampled from $N(0,1/D)$, and replace the original feature vector $\xb$ in our problem setup by its $D$-dimensional sketch $\Sb\xb$. We then train a linear predictor with $D$ trainable parameters in the sketched space,
\[
f_{\vb}(\xb)=\langle \vb,\Sb\xb\rangle,\qquad \vb\in\mathbb{R}^D,
\]
starting from zero initialization $\vb_0 = \mathbf{0}$ and using the same two-source data model and training protocols introduced in Section~\ref{sec:setup}. Accordingly, the population risk on the real distribution becomes
\[
\R_D(\vb):=\frac12\EE_{(\xb,y)\sim\cP_1}\big(\langle \vb,\Sb\xb\rangle-y\big)^2.
\]
where the expectation is conditioned on $\Sb$. The corresponding optimal parameters are
\[
\vb_\ell^*:=(\Sb\Hb\Sb^\top)^{-1}\Sb\Hb\wb_\ell^*,\qquad \ell\in\{1,2\}.
\]
Thus, the sketch dimension $D$ plays the role of model size, allowing us to study how model scaling interacts with source mismatch and synthetic data. As in \citet{linscaling}, we can decompose the risk into irreducible, approximation, and excess terms:
\[
\R_D(\vb_{M+N})
=
\underbrace{\min \R(\cdot)}_{\mathsf{Irreducible}}
+
\underbrace{\min \R_D(\cdot)-\min \R(\cdot)}_{\mathsf{Approx}}
+
\underbrace{\R_D(\vb_{M+N})-\min \R_D(\cdot)}_{\mathsf{Excess}}.
\]
Following the prior literature \citep{linscaling,lin2025improved,li2025functional,zhang2026scaling}, we further make the following additional assumptions to establish scaling law upper bounds for the two training protocols and to explicitly characterize the effects of synthetic data and model capacity.

\begin{assumption}[Distributional conditions]\label{ass:distribution-condition}
We assume $\xb \sim \mathcal{N}(0,\Hb)$, and the random vectors $\wb_1^*$ and $\bm{\delta}:=\wb_2^*-\wb_1^*$ satisfy $\EE[\wb_1^*(\wb_1^*)^\top] = \mathbf I$, $\EE[\bm{\delta}\bm{\delta}^\top] = \bm{\Sigma}$, and $\EE[\wb_1^*\bm{\delta}^\top] = \mathbf 0$.
\end{assumption}

\begin{assumption}[Power-law spectrum and source condition]\label{ass:power-law-source}
Let $(\lambda_i,\vb_i)_{i\ge1}$ be the eigenvalue-eigenvector pairs of $\Hb$ with $(\lambda_i)_{i\ge1}$ in non-increasing order. We assume $\lambda_i\eqsim i^{-a}$ for some $a>1$, and for the mismatch $\bm{\delta}$ we assume $\EE[\langle \vb_i,\bm{\delta}\rangle\langle \vb_j,\bm{\delta}\rangle]=0$ for $i\neq j$ and $\EE[\lambda_i\langle \vb_i,\bm{\delta}\rangle^2]\eqsim i^{-b}$ for some $b>1$.
\end{assumption}
The exponent $b$ in Assumption~\ref{ass:power-law-source} characterizes how hard the mismatch $\bm{\delta}$ is relative to the data spectrum $\Hb$. A larger $b$ implies a simpler task, as the mismatch energy decays more rapidly along $\Hb$'s eigen-directions.

\begin{theorem}[Scaling law for mixed training]\label{thm:mix-scaling}
Given $M$ synthetic and $N$ real samples. Consider last iterate SGD with stepsize scheme \eqref{eq:geometry-tail-decay-lr}. Suppose that Assumptions~\ref{assump:second-moment},~\ref{assump:fourth-moment},~\ref{assump:noise},~\ref{ass:distribution-condition} and~\ref{ass:power-law-source} hold. Let $\widetilde{N}_{\mathtt{eff}}=(M+N)/\log (M+N)$. Suppose $\gamma<1/4\alpha\tr(\Sb\Hb\Sb^\top)$, $\sigma_1^2,\sigma_2^2\eqsim1$ and choose $\gamma \eqsim 1$, then with probability at least $1-e^{-\Omega(D)}$ over the randomness of the sketch matrix $\Sb$, we have
\begin{align*}
    \EE\R_D(\vb_{M+N})-\R(\wb_1^{*})
    &\lesssim  
    \underbrace{
        \frac{1}{D^{a-1}}+ \frac{1}{(\widetilde{N}_{\mathtt{eff}})^{1-1/a}}
    }_{\mathsf{Approx+Bias}} + 
    \underbrace{
        \big(\frac{M}{M+N}\big)^2(1-\frac{1}{D^{b-1}}) 
        \vphantom{\frac{1}{D^{a-1}}+ \frac{1}{(\Neff+\Meff)^{1-1/a}}}
    }_{
        \mathsf{Drift\ Floor} \vphantom{\mathsf{Approx+Bias}}
    }
\end{align*}
Moreover, if the sketched problem is optimization-saturated that
$\gamma\widetilde N_{\mathtt{eff}}\gtrsim D^a$, then the preceding
upper bound is sharp up to constants:
\begin{align*}
    \EE\R_D(\vb_{M+N})-\R(\wb_1^{*})
    \eqsim  D^{1-a}+ \widetilde N_{\mathtt{eff}}^{\,1/a-1}+
    p^2(1-D^{1-b}).
\end{align*}
\end{theorem}

\textbf{Interpretation.}
The first part of Theorem~\ref{thm:mix-scaling} provides a clean upper-bound scaling picture for mixed training under the random sketch model. When
$\gamma\widetilde N_{\mathtt{eff}}\gtrsim D^a$, the cumulative stepsize is
large enough to learn even the smallest sketched eigendirection, whose
eigenvalue is of order $D^{-a}$, and a matching lower bound holds in this optimization-saturated regime. The fluctuation and variance terms are absorbed into higher-order contributions and can be bounded by the same order as the $\mathsf{Approx+Bias}$ term. Hence, they do not affect the leading-order behavior, which is dominated by $\mathsf{Approx+Bias}$ and the drift floor. Second, synthetic data improves the reducible part of the risk at the pre-asymptotic level. Compared with the real-only scaling $\frac{1}{(\Neff)^{1-1/a}}$ in \citet{linscaling}, the term $\frac{1}{(\widetilde{N}_{\mathtt{eff}})^{1-1/a}}$ shows that synthetic and real samples contribute jointly, effectively increasing the sample size and reducing $\mathsf{Approx+Bias}$. However, this benefit comes with a fundamental trade-off. While increasing the sketch dimension $D$ decreases the approximation term $D^{1-a}$ and therefore improves $\mathsf{Approx+Bias}$, it simultaneously increases the drift-floor term $p^2(1-\frac{1}{D^{b-1}}).$
Indeed, as $D$ grows, the factor $D^{1-b}$ vanishes and the drift contribution approaches its limiting level of order $p^2$. Therefore, larger models fit the real signal better, but they may amplify synthetic-induced degradation by exposing the mismatch between the real and synthetic sources more fully, revealing a fundamental trade-off in mixed training.

\begin{theorem}[Scaling law upper bound for two-stage training]\label{thm:two-stage-scaling}
Given $M$ synthetic and $N$ real samples. Consider last iterate SGD with stepsize scheme \eqref{eq:geometry-tail-decay-lr}. Suppose that Assumptions~\ref{assump:second-moment},~\ref{assump:fourth-moment},~\ref{assump:noise},~\ref{ass:distribution-condition} and~\ref{ass:power-law-source} hold. Let $\Neff=N/\log N$ and $\Meff=M/\log M$. Suppose $\gamma<1/4\alpha\tr(\Sb\Hb\Sb^\top)$ and $\sigma_1^2,\sigma_2^2\eqsim1$. Choose $\gamma \eqsim 1$, and assume $b < a+1$ for simplicity and clarity, then with probability at least $1-e^{-\Omega(D)}$ over the randomness of the sketch matrix $\Sb$, we have
\begin{align*}
    \EE\R_D(\vb_{M+N})-\R(\wb^{*}_1) &\lesssim \underbrace{\frac{1}{D^{a-1}}}_{\mathsf{Approx}}+  \underbrace{\frac{\min\{D,(\Neff)^{1/a}\}}{\Neff}}_{\mathsf{Variance}}\\
    & +\underbrace{e^{-\Omega(\frac{\Neff}{D^a})}\cdot\max\big\{\frac{1}{D^{a\wedge b-1}},\frac{1}{(\Meff)^{\frac{a\wedge b-1}{a}}}  \big\}+\max\big\{\frac{1}{D^{b-1}},\frac{1}{(\Neff)^{\frac{b-1}{a}}}  \big\}}_{\mathsf{Bias}} 
\end{align*}
\end{theorem}
\textbf{Interpretation.}
The variance term is of the same order as in the real-only scaling law of \citet{linscaling}, namely $\frac{\min\{D,\Neff^{1/a}\}}{\Neff},$
and therefore synthetic data does not improve the variance, which is consistent with our earlier analysis. The bias term is split into two parts, reflecting two distinct effects of synthetic training phase. The first term, $e^{-\Omega(\Neff/D^a)}\max\{D^{1-a\wedge b},\Meff^{-(a\wedge b-1)/a}\},$
captures the residual error from the first stage, and is exponentially forgotten during training on real data. The second term, $\max\{D^{1-b},\Neff^{-(b-1)/a}\},$ captures the remaining source mismatch after the second stage real-data training. Compared with the real-only bias upper bound $\max\{D^{1-a},\Neff^{-(a-1)/a}\}$\citep{linscaling},
this term is strictly smaller when $b>a$, corresponding to high-quality synthetic data. In this regime, the mismatch is easier than the original task, so synthetic training phase places the model closer to the target and reduces the bias more effectively than random initialization. So both bias contributions can be improved at the upper-bound level than in the real-only case. Therefore, unlike mixed training, two-stage training can genuinely benefit from synthetic data: high-quality synthetic data can reduce the bias without introducing any non-vanishing drift floor.

This subsection establishes only an upper bound for two-stage training, since
a uniform termwise matching lower bound is generally impossible, because upper-bound analysis separately bounds the propagated first-stage error and the source-mismatch contribution, although their cross term may be negative and substantially cancel in the
actual final iterate. Accordingly,
Theorem~\ref{thm:two-stage-scaling} should be interpreted as an upper-bound
scaling characterization of how
high-quality synthetic pretraining may reduce bias rather than claiming
that it uniformly outperforms real-only training.

\subsection{Finite-Sample Conditions for Synthetic Benefit}\label{sec:synthetic_benefit}

We now ask when the potential synthetic bias-reduction translates into a strict improvement in the actual risk. In this subsection, under a fixed real-data budget and identical real-stage training, we establish an exact finite-sample necessary-and-sufficient condition for two-stage training to outperform real-only training, and then use a one-dimensional toy model to characterize the roles of synthetic sample size, label noise, and source mismatch. Detailed proof is deferred to Appendix~\ref{ap:synthetic-benefit}.

\begin{theorem}[Exact condition for strict synthetic benefit]
\label{thm:exact-synthetic-benefit}
Let $\wb_N^{\mathrm{real}}$ and $\wb_{M+N}$ denote the outputs of real-only
and two-stage SGD, respectively. Both procedures start from $\wb_0$; the
two-stage procedure additionally uses $M$ synthetic samples before the two
procedures undergo the same $N$ real updates, with identical samples, order,
and stepsizes. Define $\Phi_N:=\prod_{t=N}^{1}(\Ib-\gamma_t\xb_t\xb_t^\top)$ and
$\mathbf{K}_N:=\EE_{\xb_1,\ldots,\xb_N}[\Phi_N^\top\Hb\Phi_N].$ Under Assumptions~\ref{assump:second-moment}
and~\ref{assump:noise},
\begin{align*}
2\left(\EE[\mathcal{E}_1(\wb_N^{\mathrm{real}})]-\EE[\mathcal{E}_1(\wb_{M+N})]\right)=
\|\wb_0-\wb_1^*\|_{\mathbf{K}_N}^{2}
-\|\EE[\wb_M]-\wb_1^*\|_{\mathbf{K}_N}^{2}-
\tr\left(\mathbf{K}_N\operatorname{Cov}(\wb_M)
\right).
\end{align*}
Consequently, two-stage training strictly outperforms real-only training if
and only if
\begin{equation*}
\|\EE[\wb_M]-\wb_1^*\|_{\mathbf{K}_N}^{2}
+\tr\left(\mathbf{K}_N\operatorname{Cov}(\wb_M)
\right)
<\|\wb_0-\wb_1^*\|_{\mathbf{K}_N}^{2}.
\end{equation*}
\end{theorem}

This condition compares three quantities.
$\|\wb_0-\wb_1^*\|_{\mathbf{K}_N}^{2}$ is the real-only initialization error that remains after
$N$ real updates. $\|\EE[\wb_M]-\wb_1^*\|_{\mathbf{K}_N}^{2}$ is the remaining bias of the synthetic initialization; it includes both finite-$M$ optimization bias and
source mismatch. $\tr\left(\mathbf{K}_N\operatorname{Cov}(\wb_M)
\right)$ is the remaining uncertainty generated by
finite synthetic training, including synthetic sampling and label noise. The role of $\mathbf{K}_N$ follows from
$\|\ub\|_{\mathbf{K}_N}^2=\EE\|\Phi_N\ub\|_{\Hb}^2$, which measures the expected error remaining after applying the same $N$ real-data updates to an initial error $\ub$. Therefore, synthetic pretraining is beneficial exactly when $\wb_M$ provides a better starting point than $\wb_0$ after both are
followed by the same real-data training and it must improve the initialization in directions that cannot be efficiently corrected by the available real data.

Theorem~\ref{thm:exact-synthetic-benefit} gives an exact criterion once the
mean and covariance of $\wb_M$ are characterized. In general, however, these
quantities depend jointly on the synthetic sample size, label noise, source
mismatch, data spectrum, and SGD dynamics, making their individual effects
difficult to separate. To disentangle these factors and obtain an explicit sample-size threshold, we next specialize the condition to a one-dimensional toy model.

\textbf{One-dimensional illustration.}
Let $x$ be Rademacher, independent of the label noises, and consider
\[
    y^{\mathrm{real}}=\theta x+\xi_1,
    \qquad
    y^{\mathrm{syn}}=(\theta+\delta)x+\xi_2,
    \qquad
    \EE[\xi_\ell]=0,
    \quad
    \operatorname{Var}(\xi_\ell)=\sigma_\ell^2.
\]
Starting from $w_0$, the synthetic stage performs $M$ SGD-updates with
stepsize $1/s$ at iteration $s$. Afterward, both procedures use the same
$N$ real samples and stepsizes $0<\gamma_t<1$.

\begin{corollary}[One-dimensional strict-benefit condition]
\label{cor:one-dimensional-benefit}
For every $M,N\geq1$,
\begin{align*}
\EE[\mathcal{E}_1(w_{M+N})]-\EE[\mathcal{E}_1(w_N^{\mathrm{real}})]=
\frac{1}{2}
\left(\prod_{t=1}^{N}(1-\gamma_t)
\right)^2\left[\delta^2+\frac{\sigma_2^2}{M}-(w_0-\theta)^2\right].
\end{align*}
Hence, synthetic pretraining strictly improves the expected risk if and only
if $\delta^2+\sigma_2^2/M<(w_0-\theta)^2$. Such an improvement is achievable
for some $M$ if and only if $\delta^2<(w_0-\theta)^2$; under this condition,
the exact requirement on the synthetic sample size is $M>\sigma_2^2/\big((w_0-\theta)^2-\delta^2\big)$.
\end{corollary}

This condition displays the bias-variance trade-off directly. The mismatch $\delta^2$ is the error that remains even with infinitely many synthetic samples, whereas $\sigma_2^2/M$ is the uncertainty caused by finite and noisy synthetic data. Thus, increasing $M$ can eliminate the synthetic variance but cannot compensate for a mismatch larger than the
original initialization error. The required synthetic sample size increases linearly with $\sigma_2^2$ and diverges as $\delta^2$ approaches $(w_0-\theta)^2$.

\begin{figure}[t]
    \centering
    \includegraphics[width=1.01\linewidth]{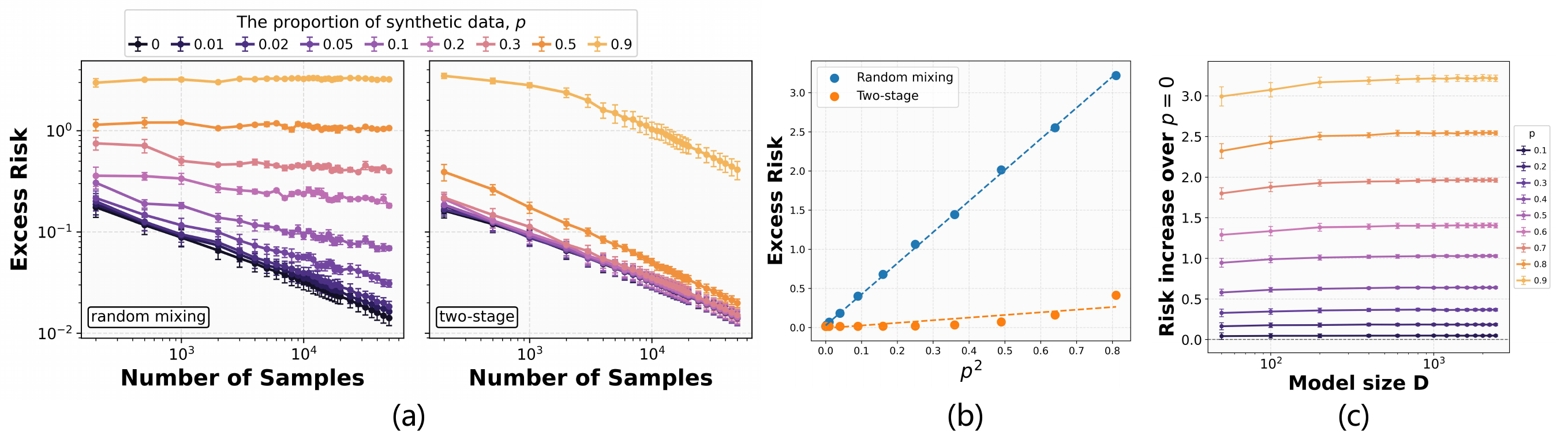}
    \caption{
\textbf{Simulation results (i).}
(a) Excess risk of two protocols in full high-dimensional linear model: mixed training exhibits a clear $p$-dependent error floor as $T$ scales, while two-stage training continues to improve without saturation.
(b) Excess risk at fixed $T=50000$ versus $p^2$: under mixing, the excess risk scales approximately linearly with $p^2$, consistent with the theoretical prediction $p^2 \|\bm\delta\|_\Hb^2$, whereas two-stage training shows a much weaker dependence, matching our theory.
(c) Risk increase over $p=0$ under mixed training in the random sketch model at fixed $T=50000$: the degradation grows with model size $D$ and saturates, consistent with the theoretical scaling $p^2(1 - D^{1-b})$, indicating that larger models amplify the contamination from synthetic mixed training.
}

    \label{fig:results}
\end{figure}

\section{Experiments}\label{sec:experiments}
\paragraph{Simulation.}
We first validate our theory in a Gaussian high-dimensional linear regression setting with dimension $d=5000$. Features are drawn from a Gaussian distribution with covariance spectrum $\lambda_i \propto i^{-1.5}$, and the synthetic teacher satisfies $\bm w_2^*=\bm w_1^*+\bm\delta$ with source condition $\lambda_i \mathbb{E}[\delta_i^2]\propto i^{-b}$. We train last-iterate SGD under the mixed and two-stage protocols in Section~\ref{sec:setup}, and evaluate excess risk on the real distribution. This experiments contain two parts. The first part varies the synthetic proportion $p\in(0,1)$, sample size $T\in[50,50000]$, and sketch width $D$ to validate the predicted synthetic-induced error floor. The second part keeps the same high-dimensional setup but varies the synthetic quality $b$, synthetic sample size $M$, and synthetic label noise $\sigma_2^2$ to study when synthetic data helps. All experiments are repeated with 5 random seeds, , and the results of the two parts are reported in Figure~\ref{fig:results} and Figure~\ref{fig:synthetic-quality}, respectively.

\begin{figure}[t]
\centering
\includegraphics[width=1.03\linewidth]{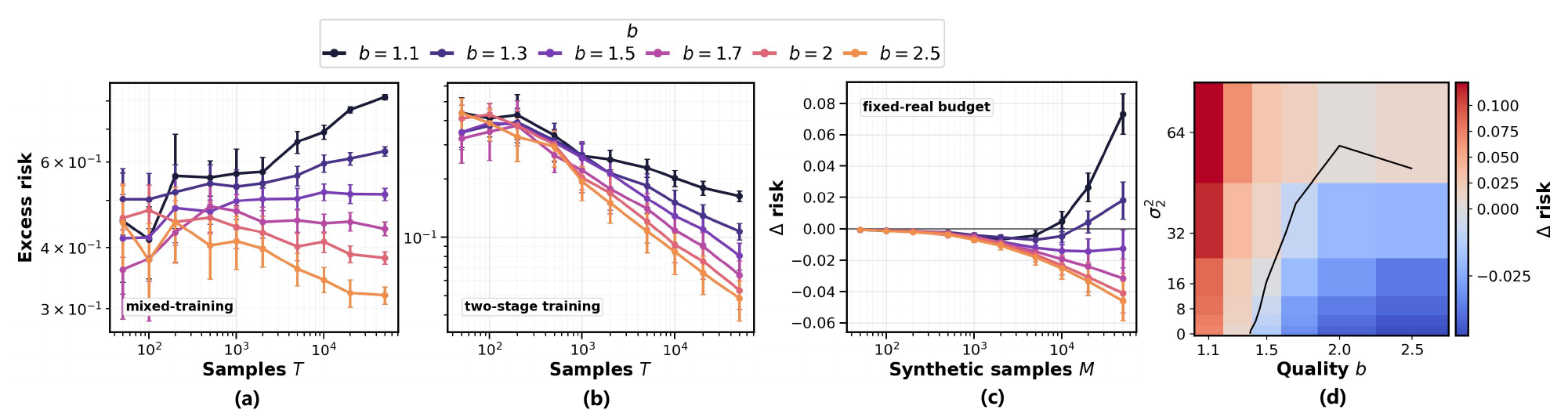}
\caption{
\textbf{Simulation results (ii).}
(a) Excess risk of mixed training under different synthetic qualities: low-quality synthetic data induces a larger error floor, while increasing $b$ substantially reduces the degradation.
(b) Excess risk of two-stage training under different synthetic qualities: two-stage training continues to improve with $T$ for all $b$, and higher-quality synthetic data leads to lower final risk.
(c) Risk change from adding synthetic samples under a fixed real-data budget: synthetic data helps when $b$ is large, but low-quality synthetic data can eventually hurt as $M$ increases.
(d) Benefit region over synthetic quality and label noise: additional synthetic data improves real risk only when the teacher mismatch is sufficiently benign and $\sigma_2^2$ is not too large.
}
\label{fig:synthetic-quality}
\end{figure}

\paragraph{Real-world data.}
We further extend the experiments to validate our theory on $\mathtt{CIFAR10}$~\citep{krizhevsky2009learning} using a 5-layer CNN and on $\mathtt{WikiText-103}$ via language model training. Synthetic data is generated by a pretrained teacher model. We observe the same qualitative behavior as predicted by our theory. Full experimental details are provided in Appendix~\ref{app:real-data}.

\section{Conclusion}
We studied how synthetic data affects the generalization of one-pass SGD in high-dimensional linear regression with model shift. We establish finite-sample risk bounds for mixed and two-stage training that characterize how source mismatch enters the SGD dynamics under different training protocols. As a consequence, mixed training exhibits \textit{strong model collapse} while two-stage training avoids this by using synthetic data only in the first stage, showing that collapse is not inevitable through a simple data curriculum. Our random sketch model analysis further establishes scaling-law bounds for both protocols, showing that larger models may amplify synthetic-induced degradation and giving an explicit characterization of how high-quality synthetic pre-training may reduce bias. We also establish an exact finite-sample condition under which two-stage training strictly outperforms real-data-only training with a fixed real-data budget and identical real-stage training. Overall, our results highlight that synthetic data is not inherently harmful or beneficial; its effect depends critically on both its quality and how it is incorporated during training.
\paragraph{Limitations.} Our analysis focuses on one-pass SGD with shared covariates. Extensions to joint covariate and model shift, as well as multi-epoch SGD, are natural next steps beyond our current setting. In addition, the mixed-training scaling law is tight only in the optimization-saturated regime, while the two-stage result remains an upper bound; this disparity between the two protocols leaves tight characterizations for data-limited and two-stage settings open.

\bibliography{reference}

\bibliographystyle{plainnat}

\newpage

\appendix

\newpage
\tableofcontents
\resettocdepth{}

\tableofcontents
\section{Technical Operator Preliminaries}

\paragraph{Operators.}
We first summarize the linear operators (on symmetric matrices) to be used in the proof:
\begin{gather*}
    \cI = \Ib \otimes \Ib,\quad
    \cM = \EE [ \xb \otimes \xb \otimes \xb \otimes \xb ],\quad
    \tilde{\cM} = \Hb \otimes \Hb, \\
    \cT_t = \Hb \otimes \Ib + \Ib \otimes \Hb - \gamma_t\cM, \quad
    \tilde \cT_t = \Hb \otimes \Ib + \Ib \otimes \Hb - \gamma_t\Hb\otimes\Hb.
\end{gather*}
We use the notation $\mathcal{O}\circ \Ab$ to denotes the operator
$\mathcal{O}$ acting on a symmetric matrix $\Ab$. For example, with these definitions, we have that for a symmetric matrix $\Ab$, \citep{zou2023benign}:
\begin{gather}
\cI \circ \Ab = \Ab, \ \ \ \cM\circ \Ab = \EE [ (\xb^\top \Ab \xb) \xb \xb^\top ], \ \ \  \tilde{\cM} \circ \Ab = \Hb \Ab \Hb, \notag \\     
(\cI - \gamma_t \cT_t) \circ \Ab = \EE [ (\Ib - \gamma_t \xb \xb^\top)\Ab (\Ib - \gamma_t \xb \xb^\top) ], \ \  (\cI-\gamma_t\tilde\cT_t)\circ\Ab = (\Ib-\gamma_t\Hb)\Ab(\Ib-\gamma_t\Hb). \label{eq:0005}
\end{gather}

For the linear operators we have the following technical lemma from \citet{zou2023benign}.
\begin{lemma}[Lemma B.1, \citet{zou2023benign}]\label{lemma:operators}
An operator $\mathcal{O}$ defined on symmetric matrices is called PSD mapping, if $\Ab \succeq 0$ implies $\mathcal{O} \circ \Ab \succeq 0$.
Then we have
\begin{enumerate}
    \item $\cM$ and $\widetilde\cM$ are both PSD mappings.
    \item $\cI-\gamma\cT_t$ and $\cI-\gamma\widetilde\cT_t$ are both PSD mappings.
    \item $\cM - \widetilde\cM$ and $\widetilde \cT_t - \cT_t$ are both PSD mappings.
\item  If $0 < \gamma_t < 1/\lambda_1$, then $\widetilde{\cT_t}^{-1}$ exists, and is a PSD mapping.
\item  If $0 < \gamma_t < 1/(\alpha\tr(\Hb))$, then $\cT_t^{-1}\circ \Ab$ exists for PSD matrix $\Ab$, and $\cT_t^{-1}$ is a PSD mapping.
\end{enumerate}
\end{lemma}
\begin{proof}
See proof of Lemma B.1 in \citet{zou2023benign}.
\end{proof}

\section{Mixed Training Upper Bound Analysis}\label{ap:mix-upper}

\paragraph{Local notation.}
For notational simplicity, throughout this section we depart slightly from the
notation in the main text and use $N$ to denote the total sample size. Thus,
the numbers of synthetic and real samples are $pN$ and $(1-p)N$, respectively,
where $p\in[0,1]$ is the synthetic-data proportion. We assume that $pN$ and
$(1-p)N$ are integers; replacing them by the corresponding floor or ceiling
does not affect any of our results.

\paragraph{Fixed-budget sampling model.}
\label{ap:mixed-sampling}
We formalize the random-order mixed-training protocol used in our analysis.
Let
\[
    \mathcal{Z}_{N,p}
    :=
    \left\{
        \mathbf{z}=(z_1,\ldots,z_N)\in\{0,1\}^{N}
        :
        \sum_{t=1}^{N}z_t=pN
    \right\},
\]
where $z_t=1$ indicates that the sample at iteration $t$ is synthetic and
$z_t=0$ indicates that it is real. We first draw the source schedule
$\mathbf{z}$ uniformly from $\mathcal{Z}_{N,p}$. Conditional on the complete source schedule $\mathbf{z}$, the training samples
are generated according to
\[
    \xb_t\stackrel{\mathrm{i.i.d.}}{\sim}\cD,
    \qquad
    y_t
    =
    \left\langle
        \xb_t,\wb_1^*+z_t\bm{\delta}
    \right\rangle
    +\xi_t,
    \qquad
    t=1,\ldots,N,
\]
where $\bm{\delta}:=\wb_2^*-\wb_1^*$ and
\[
    \xi_t\mid z_t
    \sim
    \begin{cases}
        N(0,\sigma_1^2), & z_t=0,\\
        N(0,\sigma_2^2), & z_t=1.
    \end{cases}
\]
Conditional on $\mathbf{z}$, the noises $\{\xi_t\}_{t=1}^{N}$ are independent
across iterations and independent of all covariates.

This construction is equivalent in distribution to independently drawing
$pN$ complete samples from $\cP_2$ and $(1-p)N$ complete samples from $\cP_1$,
and then uniformly shuffling their union. Indeed, a uniform shuffle induces a
source schedule that is uniform over $\mathcal{Z}_{N,p}$. Conditional on any
such schedule, the observations at positions with $z_t=1$ are independent
samples from $\cP_2$, while those at positions with $z_t=0$ are independent
samples from $\cP_1$. Since samples within each source are exchangeable, the
two constructions induce the same joint distribution over the ordered
training sequence.

Our analysis does not assume independence of the source indicators. To identify the conditional independence used in the SGD recursions, define
the filtration
\[
    \mathcal{F}_{t-1}:=\sigma\!\left( \mathbf{z},\xb_1,\xi_1,\ldots,\xb_{t-1},\xi_{t-1}\right).
\]
Since the real and synthetic distributions share the same covariate marginal
$\cD$, the current covariate $\xb_t$ is independent of
$\mathcal{F}_{t-1}$ and satisfies
\[
\EE[\xb_t\xb_t^\top\mid\mathcal{F}_{t-1}] =\Hb.
\]
Moreover, the sample-wise noise construction gives
\[
    \EE[\xi_t\mid\mathcal{F}_{t-1},\xb_t]=0.
\]
Consequently, the noise cross terms appearing in the bias--variance
recursions vanish. Averaging over the source schedule further yields
\[
    \EE[\xi_t^2\xb_t\xb_t^\top]
    =\bar{\sigma}^2\Hb,\qquad
    \bar{\sigma}^2:=(1-p)\sigma_1^2+p\sigma_2^2.
\]
Therefore, the covariance and noise recursions used below remain exact under
the fixed-budget sampling protocol. They rely on fresh covariates and
sample-wise conditionally independent noises, rather than on independence of
the source indicators.

\subsection{Excess Risk Decomposition}
Starting from $\wb_0\in\mathbb{R}^\cH$, SGD on the squared loss performs
\[
\wb_t
=
\wb_{t-1}
+
\gamma_t (y_t-\xb_t^\top\wb_{t-1})\xb_t ,
\qquad t\ge1 .
\]

Substituting the label model yields
\[
\wb_t
=
\wb_{t-1}
+
\gamma_t\!\left(
\xb_t^\top(\wb_1^*+z_t\bm{\delta}-\wb_{t-1})
+
\xi_t
\right)\xb_t .
\]

Let $\bm{\eta}_t:=\wb_t-\wb_1^*$ denote the parameter error relative to the real-data optimum. Then equation~\eqref{eq:error-recursion} serves as the starting point of the analysis:
\begin{equation}
\bm{\eta}_t
=
(\Ib-\gamma_t\xb_t\xb_t^\top)\bm{\eta}_{t-1}
+
\gamma_t z_t\xb_t\xb_t^\top\bm{\delta}
+
\gamma_t\xi_t\xb_t .
\label{eq:error-recursion}
\end{equation}
With as light abuse of probability spaces, one can view the centered SGD
iterates as the sum of four random processes:
\[
\bm{\eta}_t=\bm{\eta}^{\bias}_t+\bm{\eta}^{\var}_t+\bm{\eta}^{\drift}_t+\bm{\eta}^{\fluct}_t, \qquad t=1,2,...N.
\]

where
\[
\begin{aligned}
&\begin{cases}
 \bm{\eta}^{\bias}_t = (\Ib - \gamma_t \xb_t \xb_t^\top) \bm{\eta}^{\bias}_{t-1}; \\ \bm{\eta}^{\bias}_0 = \wb_0 - \wb_1^*,
\end{cases}
&&\quad 
\begin{cases}
\bm{\eta}^{\var}_t = (\Ib - \gamma_t \xb_t \xb_t^\top) \bm{\eta}^{\var}_{t-1}+\gamma_t\xi_t\xb_t; \\ \bm{\eta}^{\var}_0 = \boldsymbol{0},
\end{cases}\\[2ex] 
&\begin{cases}
 \bm{\eta}^{\drift}_t = (\Ib - \gamma_t \xb_t \xb_t^\top) \bm{\eta}^{\drift}_{t-1}+\gamma_tp\Hb\bm{\delta}; \\ \bm{\eta}^{\drift}_0 = \boldsymbol{0},
\end{cases}
&&\quad
\begin{cases}
 \bm{\eta}^{\fluct}_t = (\Ib - \gamma_t \xb_t \xb_t^\top) \bm{\eta}^{\fluct}_{t-1}+\gamma_t z_t \xb_t\xb_t^\top\bm{\delta}-\gamma_tp\Hb\bm{\delta}; \\ \bm{\eta}^{\fluct}_0 = \boldsymbol{0},
\end{cases}
\end{aligned}
\]
Moreover, we denote
\begin{align*}
    &\Bb_t = \EE [\bm{\eta}^{\bias}_t \otimes \bm{\eta}^{\bias}_t], \quad
\Vb_t = \EE [\bm{\eta}^{\var}_t \otimes \bm{\eta}^{\var}_t],\\
&\Db_t = \EE [\bm{\eta}^{\drift}_t \otimes \bm{\eta}^{\drift}_t], \quad
\Fb_t = \EE [\bm{\eta}^{\fluct}_t \otimes \bm{\eta}^{\fluct}_t].
\end{align*}


Using the definition of $\Bb_t,\Vb_t,\Db_t,\Fb_t$, we can then decompose Excess Risk using Cauchy--Schwarz Inequality:
\begin{align*}
    \EE [ \R(\wb_N) - \R(\wb_1^*) ] &= \frac{1}{2}\big\langle\Hb, \EE[\bm{\eta}_N \otimes \bm{\eta}_N] \big\rangle \\
    &= \frac{1}{2}\big\langle\Hb, \EE[(\bm{\eta}_N^\bias + \bm{\eta}_N^\var + \bm{\eta}_N^\drift + \bm{\eta}_N^\fluct) \otimes (\bm{\eta}_N^\bias + \bm{\eta}_N^\var + \bm{\eta}_N^\drift + \bm{\eta}_N^\fluct)]\big \rangle \\
    &\le 2 \big\langle \Hb, \EE[\bm{\eta}_N^\bias \otimes \bm{\eta}_N^\bias] \big\rangle + 2 \big\langle \Hb, \EE[\bm{\eta}_N^\var \otimes \bm{\eta}_N^\var] \big\rangle \\
    &\quad + 2 \big\langle \Hb, \EE[\bm{\eta}_N^\drift \otimes \bm{\eta}_N^\drift] \big\rangle + 2 \big\langle \Hb, \EE[\bm{\eta}_N^\fluct \otimes \bm{\eta}_N^\fluct] \big\rangle \\
    &= 2 \big\langle \Hb, \Bb_N \big\rangle + 2 \big\langle \Hb, \Vb_N \big\rangle + 2 \big\langle \Hb, \Db_N \big\rangle + 2 \big\langle \Hb, \Fb_N \big\rangle.
\end{align*}

\subsection{Bias upper bound}
Using the defined operators, the update rule of the iterates imply the following
recursive form of $\Bb_t$:
\begin{equation}\label{eq:update_Bt}
    \Bb_t = (\cI - \gamma_t\cT_t)\circ \Bb_{t-1}, \quad \Bb_0 = \betab_0\otimes \betab_0,
\end{equation}
Note that this bias term $\big\langle \Hb, \Bb_N \big\rangle$ is the same as that in \cite{wulast}, so we can apply Lemma in \cite{wulast} to bound it directly:
\begin{lemma}[A bias upper bound]\label{lemma:HB-decay-lr-upper-bound}
Suppose Assumptions \ref{assump:second-moment} and \ref{assump:fourth-moment} hold. Let $\Neff=N/\log N$. 
Consider \eqref{eq:update_Bt}.
Suppose $\gamma < 1/(4\alpha\tr(\Hb))$. We have 
\begin{align*}
\big\langle \Hb, \Bb_N \big\rangle \lesssim \big\| \prod_{t=1}^N(\Ib-\gamma_t \Hb) (\wb_0 - \wb^*) \big\|^2_{\Hb} +\alpha\|\wb_0 - \wb^* \|^2_{\frac{\Ib_{0:k^*}}{\gamma \Neff}+\Hb_{k^*:\infty}} \cdot \frac{\DIM}{\Neff}
\end{align*}
where $k^* = \max \{k: \lambda_k \ge \frac{1}{\gamma \Neff}\}$ and $\DIM := k^* + \gamma^2 \Neff^2 \sum_{i>k^*}\lambda_i^2.$
\end{lemma}
\begin{proof}
See proof of Theorem D.1 in \cite{wu2022power}.
\end{proof}

\subsection{Variance upper bound}
Using the defined operators, the update rule of the iterates imply the following
recursive form of $\Vb_t$:
\begin{equation}\label{eq:update_Ct}
\Vb_t = (\cI-\gamma_t\cT_t) \circ \Vb_{t-1} + \gamma_t^2\bSigma_t,\qquad \Vb_0 = \boldsymbol{0}.
\end{equation}
where $\bSigma_t=\EE[\xi_t^2\xb_t\xb_t^\top]$. Note that $\{\xb_t\}$ are independent of $\{z_t\}$. So we have:
$$\bSigma_t=\EE[\EE[\xi_t^2\xb_t\xb_t^\top]|z_t]=(1-p)\bSigma_1+p\bSigma_2 \preceq((1-p)\sigma_1^2+p\sigma_2^2) \Hb.$$
Then we can apply Lemma C.2 in \cite{wulast} to bound variance term $\big \langle \Hb, \Vb_N \big \rangle$ by simply replace $\sigma^2$ with $(1-p)\sigma_1^2+p\sigma_2^2$.

\begin{lemma}[A variance bound]\label{thm:HC-decay-lr-upper-bound}
Suppose Assumptions \ref{assump:second-moment}, \ref{assump:fourth-moment} and \ref{assump:noise} hold.
Consider \eqref{eq:update_Ct}.
Let $\Neff=N/\log N$.
Suppose $\gamma < 1/(\alpha\tr(\Hb))$.  We have
\begin{equation*}
    \big \langle \Hb, \Vb_N \big \rangle \le \frac{8\sigma^2}{1-\gamma \alpha\tr(\Hb)}\cdot\frac{\DIM}{\Neff},
\end{equation*}
where $k^* = \max \{k: \lambda_k \ge \frac{1}{\gamma \Neff}\}$ and $\DIM := k^* + \gamma^2 \Neff^2 \sum_{i>k^*}\lambda_i^2.$
\end{lemma}
\begin{proof}
See proof of Theorem C.2 in \cite{wulast}.
\end{proof}

\subsection{Drift upper bound}
We now analyze the drift term
\[
\big\langle \Hb,\Db_N\big\rangle
=
\big\langle \Hb,\EE[\bm{\eta}^{\drift}_N\otimes \bm{\eta}^{\drift}_N]\big\rangle .
\]

\begin{lemma}[A drift bound]\label{thm:drift-decay-lr-upper-bound}
Suppose Assumptions \ref{assump:second-moment}, \ref{assump:fourth-moment} and \ref{assump:noise} hold.
Consider \eqref{eq:update_Ct}.
Let $\Neff=N/\log N$.
Suppose $\gamma < 1/(\alpha\tr(\Hb))$.  We have
\begin{equation*}
    \big \langle \Hb, \Db_N \big \rangle \le p^2\Bigl\|
\Bigl(\Ib-\prod_{t=1}^N(\Ib-\gamma_t\Hb)\Bigr)\bm{\delta}
\Bigr\|_{\Hb}^2
+
\frac{8\alpha p^2\|\bm{\delta}\|_{\Hb}^2}{1-\gamma \alpha\tr(\Hb)}
\cdot
\frac{\DIM}{\Neff},
\end{equation*}
where $k^* = \max \{k: \lambda_k \ge \frac{1}{\gamma \Neff}\}$ and $\DIM := k^* + \gamma^2 \Neff^2 \sum_{i>k^*}\lambda_i^2.$
\end{lemma}
\begin{proof}

Recall that $\bm{\eta}^{\drift}_t$ satisfies
\[
\bm{\eta}^{\drift}_t
=
(\Ib-\gamma_t\xb_t\xb_t^\top)\bm{\eta}^{\drift}_{t-1}
+
\gamma_t p\Hb\bm{\delta},
\qquad
\bm{\eta}^{\drift}_0=\boldsymbol{0}.
\]

We first center $\bm{\eta}^{\drift}_t$ by separating its mean and fluctuation parts. Define $\overline{\bm{\eta}}^{\drift}_t:=\EE[\bm{\eta}^{\drift}_t].$
Since $\xb_t$ is independent of $\bm{\eta}^{\drift}_{t-1}$ and $\EE[\xb_t\xb_t^\top]=\Hb$, taking expectation on both sides yields
\begin{equation}\label{eq:drift-mean-recursion}
\overline{\bm{\eta}}^{\drift}_t
=
(\Ib-\gamma_t\Hb)\overline{\bm{\eta}}^{\drift}_{t-1}
+
\gamma_t p\Hb\bm{\delta},
\qquad
\overline{\bm{\eta}}^{\drift}_0=\boldsymbol{0}.
\end{equation}
Define the centered drift fluctuation
$\widetilde{\bm{\eta}}^{\drift}_t:=\bm{\eta}^{\drift}_t-\overline{\bm{\eta}}^{\drift}_t$
and subtracting \eqref{eq:drift-mean-recursion} from the recursion of $\bm{\eta}^{\drift}_t$, we obtain
\begin{align}
\widetilde{\bm{\eta}}^{\drift}_t
&=
(\Ib-\gamma_t\xb_t\xb_t^\top)\bm{\eta}^{\drift}_{t-1}
-
(\Ib-\gamma_t\Hb)\overline{\bm{\eta}}^{\drift}_{t-1} \notag\\
&=
(\Ib-\gamma_t\xb_t\xb_t^\top)\widetilde{\bm{\eta}}^{\drift}_{t-1}
+
\gamma_t(\Hb-\xb_t\xb_t^\top)\overline{\bm{\eta}}^{\drift}_{t-1},
\qquad
\widetilde{\bm{\eta}}^{\drift}_0=\boldsymbol{0}.
\label{eq:drift-centered-recursion}
\end{align}
By definition,
\[
\EE[\widetilde{\bm{\eta}}^{\drift}_t]=\boldsymbol{0},
\qquad \forall t\ge 1.
\]
Using $\bm{\eta}^{\drift}_t=\overline{\bm{\eta}}^{\drift}_t+\widetilde{\bm{\eta}}^{\drift}_t,$
we can expand $\Db_t$ as
\begin{align*}
\Db_t
&=
\EE\!\left[
(\overline{\bm{\eta}}^{\drift}_t+\widetilde{\bm{\eta}}^{\drift}_t)
\otimes
(\overline{\bm{\eta}}^{\drift}_t+\widetilde{\bm{\eta}}^{\drift}_t)
\right] \\
&=
\overline{\bm{\eta}}^{\drift}_t\otimes \overline{\bm{\eta}}^{\drift}_t
+
\EE[\widetilde{\bm{\eta}}^{\drift}_t\otimes \widetilde{\bm{\eta}}^{\drift}_t]
+
\overline{\bm{\eta}}^{\drift}_t\otimes \EE[\widetilde{\bm{\eta}}^{\drift}_t]
+
\EE[\widetilde{\bm{\eta}}^{\drift}_t]\otimes \overline{\bm{\eta}}^{\drift}_t .
\end{align*}
Since $\EE[\widetilde{\bm{\eta}}^{\drift}_t]=\boldsymbol{0}$, the two cross terms vanish. Hence
\begin{equation*}\label{eq:Dt-split}
\Db_t
=
\overline{\bm{\eta}}^{\drift}_t\otimes \overline{\bm{\eta}}^{\drift}_t
+
\widetilde{\Db}_t,
\qquad
\widetilde{\Db}_t
:=
\EE[\widetilde{\bm{\eta}}^{\drift}_t\otimes \widetilde{\bm{\eta}}^{\drift}_t].
\end{equation*}

Therefore the drift contribution to the excess risk can be decomposed into two parts:
\begin{equation}\label{eq:drift-risk-split}
\big\langle \Hb,\Db_N\big\rangle
=
\big\|\overline{\bm{\eta}}^{\drift}_N\big\|_{\Hb}^2
+
\big\langle \Hb,\widetilde{\Db}_N\big\rangle .
\end{equation}

In the following, we will bound the two terms in \eqref{eq:drift-risk-split} separately.

\paragraph{Mean part of the drift term.}
We now derive the explicit form of
\(
\big\|\overline{\bm{\eta}}^{\drift}_N\big\|_{\Hb}^2
\).
Recall from \eqref{eq:drift-mean-recursion} that
\[
\overline{\bm{\eta}}^{\drift}_t
=
(\Ib-\gamma_t\Hb)\overline{\bm{\eta}}^{\drift}_{t-1}
+
\gamma_t p\Hb\bm{\delta},
\qquad
\overline{\bm{\eta}}^{\drift}_0=\boldsymbol{0}.
\]
Unrolling the recursion gives
\begin{align}
\overline{\bm{\eta}}^{\drift}_N
&=
p\sum_{s=1}^N
\gamma_s
\left(
\prod_{t=s+1}^N (\Ib-\gamma_t\Hb)
\right)\Hb\bm{\delta}.
\label{eq:drift-mean-unroll}
\end{align}
Since each factor \((\Ib-\gamma_t\Hb)\) is a polynomial in \(\Hb\), it commutes with \(\Hb\). Hence
\[
\left(
\prod_{t=s+1}^N (\Ib-\gamma_t\Hb)
\right)\Hb
=
\Hb
\left(
\prod_{t=s+1}^N (\Ib-\gamma_t\Hb)
\right),
\]
and therefore
\begin{align*}
\sum_{s=1}^N
\gamma_s
\left(
\prod_{t=s+1}^N (\Ib-\gamma_t\Hb)
\right)\Hb
&=
\Ib-\prod_{t=1}^N(\Ib-\gamma_t\Hb).
\end{align*}
Substituting this identity into \eqref{eq:drift-mean-unroll}, we obtain
\begin{equation*}\label{eq:drift-mean-closed-form}
\overline{\bm{\eta}}^{\drift}_N
=
p\Bigl(\Ib-\prod_{t=1}^N(\Ib-\gamma_t\Hb)\Bigr)\bm{\delta}.
\end{equation*}
As a consequence,
\begin{equation}\label{eq:drift-mean-risk-form}
\big\|\overline{\bm{\eta}}^{\drift}_N\big\|_{\Hb}^2
=
p^2
\Bigl\|
\Bigl(\Ib-\prod_{t=1}^N(\Ib-\gamma_t\Hb)\Bigr)\bm{\delta}
\Bigr\|_{\Hb}^2.
\end{equation}
This is exactly the first drift term appearing in the Lemma.

\paragraph{Fluctuation part of the drift term.}
We now bound the second term in \eqref{eq:drift-risk-split},
namely
\[
\big\langle \Hb,\widetilde{\Db}_N\big\rangle,
\qquad
\widetilde{\Db}_t
:=
\EE[\widetilde{\bm{\eta}}^{\drift}_t\otimes \widetilde{\bm{\eta}}^{\drift}_t].
\]

Recall from \eqref{eq:drift-centered-recursion} that
\[
\widetilde{\bm{\eta}}^{\drift}_t
=
(\Ib-\gamma_t\xb_t\xb_t^\top)\widetilde{\bm{\eta}}^{\drift}_{t-1}
+
\gamma_t(\Hb-\xb_t\xb_t^\top)\overline{\bm{\eta}}^{\drift}_{t-1},
\qquad
\widetilde{\bm{\eta}}^{\drift}_0=\boldsymbol{0}.
\]
Define
\[
\widetilde{\bSigma}^{\drift}_t
:=
\EE\!\left[
\bigl((\Hb-\xb_t\xb_t^\top)\overline{\bm{\eta}}^{\drift}_{t-1}\bigr)
\otimes
\bigl((\Hb-\xb_t\xb_t^\top)\overline{\bm{\eta}}^{\drift}_{t-1}\bigr)
\right].
\]
Since \(\overline{\bm{\eta}}^{\drift}_{t-1}\) is deterministic, the recursion for
\(\widetilde{\Db}_t\) takes the form
\begin{equation}\label{eq:drift-fluct-recursion}
\widetilde{\Db}_t
=
(\cI-\gamma_t\cT_t)\circ \widetilde{\Db}_{t-1}
+
\gamma_t^2 \widetilde{\bSigma}^{\drift}_t,
\qquad
\widetilde{\Db}_0=\boldsymbol{0}.
\end{equation}
Indeed, the cross term vanishes because
\[
\EE\big[(\Hb-\xb_t\xb_t^\top)\overline{\bm{\eta}}^{\drift}_{t-1}\big]
=
(\Hb-\EE[\xb_t\xb_t^\top])\overline{\bm{\eta}}^{\drift}_{t-1}
=
\boldsymbol{0}.
\]

Let $
\Ab_t
:=
\overline{\bm{\eta}}^{\drift}_{t-1}
\otimes
\overline{\bm{\eta}}^{\drift}_{t-1},$
then
\begin{align*}
\widetilde{\bSigma}^{\drift}_t
&=
\EE\!\left[
(\Hb-\xb_t\xb_t^\top)\Ab_t(\Hb-\xb_t\xb_t^\top)
\right] \\
&=
\EE[\xb_t\xb_t^\top \Ab_t \xb_t\xb_t^\top]
-\Hb \Ab_t \Hb \\
&=
(\cM-\widetilde{\cM})\circ \Ab_t .
\end{align*}
Since \(\cM-\widetilde{\cM}\) is a PSD mapping by Lemma~\ref{lemma:operators},
we have \(\widetilde{\bSigma}^{\drift}_t \succeq 0\). On the other hand, by
Assumption~\ref{assump:fourth-moment},
\[
\EE[\xb_t\xb_t^\top \Ab_t \xb_t\xb_t^\top]
=
\cM\circ \Ab_t
\preceq
\alpha\,\tr(\Hb \Ab_t)\,\Hb
=
\alpha\,\|\overline{\bm{\eta}}^{\drift}_{t-1}\|_{\Hb}^2\,\Hb.
\]
Therefore,
\begin{equation}\label{eq:drift-sigma-upper}
\widetilde{\bSigma}^{\drift}_t
\preceq
\alpha\,\|\overline{\bm{\eta}}^{\drift}_{t-1}\|_{\Hb}^2\,\Hb.
\end{equation}

Next, using \eqref{eq:drift-mean-closed-form},
$
\overline{\bm{\eta}}^{\drift}_{t-1}
=
p\Bigl(\Ib-\prod_{s=1}^{t-1}(\Ib-\gamma_s\Hb)\Bigr)\bm{\delta},
$
hence
\begin{align}
\|\overline{\bm{\eta}}^{\drift}_{t-1}\|_{\Hb}^2
&=
p^2\sum_i
\lambda_i
\left(
1-\prod_{s=1}^{t-1}(1-\gamma_s\lambda_i)
\right)^2
\delta_i^2
\notag\\
&\le
p^2\sum_i \lambda_i\delta_i^2
=
p^2\|\bm{\delta}\|_{\Hb}^2 .
\label{eq:drift-mean-H-bound}
\end{align}
Combining \eqref{eq:drift-sigma-upper} and \eqref{eq:drift-mean-H-bound}, we obtain
\begin{equation}\label{eq:drift-sigma-final}
\widetilde{\bSigma}^{\drift}_t
\preceq
\alpha p^2\|\bm{\delta}\|_{\Hb}^2 \Hb.
\end{equation}

The recursion \eqref{eq:drift-fluct-recursion} has exactly the same form as the
variance recursion \eqref{eq:update_Ct}. Therefore, applying
Lemma~\ref{thm:HC-decay-lr-upper-bound} with
\[
\sigma^2
\;\leftarrow\;
\alpha p^2\|\bm{\delta}\|_{\Hb}^2,
\]
we immediately get
\begin{equation}\label{eq:drift-fluctuation-final}
\big\langle \Hb,\widetilde{\Db}_N\big\rangle
\le
\frac{8\alpha p^2\|\bm{\delta}\|_{\Hb}^2}{1-\gamma \alpha\tr(\Hb)}
\cdot
\frac{\DIM}{\Neff},
\end{equation}
where
\[
k^* = \max \left\{k:\lambda_k\ge \frac{1}{\gamma \Neff}\right\},
\qquad
\DIM := k^*+\gamma^2\Neff^2\sum_{i>k^*}\lambda_i^2 .
\]

Combining \eqref{eq:drift-mean-risk-form} and \eqref{eq:drift-fluctuation-final},
we conclude that
\[
\big\langle \Hb,\Db_N\big\rangle
=
\big\|\overline{\bm{\eta}}^{\drift}_N\big\|_{\Hb}^2
+
\big\langle \Hb,\widetilde{\Db}_N\big\rangle
\le
p^2
\Bigl\|
\Bigl(\Ib-\prod_{t=1}^N(\Ib-\gamma_t\Hb)\Bigr)\bm{\delta}
\Bigr\|_{\Hb}^2
+
\frac{8\alpha p^2\|\bm{\delta}\|_{\Hb}^2}{1-\gamma \alpha\tr(\Hb)}
\cdot
\frac{\DIM}{\Neff}.
\]

\end{proof}

\subsection{Fluctuation upper bound}
We now analyze the drift term
\[
\big\langle \Hb,\Fb_N\big\rangle
=
\big\langle \Hb,\EE[\bm{\eta}^{\fluct}_N\otimes \bm{\eta}^{\fluct}_N]\big\rangle .
\]

\begin{lemma}[A fluctuation bound]\label{thm:fluct-decay-lr-upper-bound}
Suppose Assumptions \ref{assump:second-moment}, \ref{assump:fourth-moment} and \ref{assump:noise} hold.
Consider \eqref{eq:update_Ct}.
Let $\Neff=N/\log N$.
Suppose $\gamma < 1/(\alpha\tr(\Hb))$.  We have
\begin{equation*}
    \big \langle \Hb, \Fb_N \big \rangle \le \frac{8Np(1-p)\|\bm{\delta}\|_{\Hb}^2}{(N-1)\bigl(1-\alpha\gamma\tr(\Hb)\bigr)}\cdot \frac{\DIM}{\Neff}
+
\frac{16\alpha p\|\bm{\delta}\|_{\Hb}^2}{1-\gamma \alpha\tr(\Hb)}
\cdot
\frac{\DIM}{\Neff},
\end{equation*}
where $k^* = \max \{k: \lambda_k \ge \frac{1}{\gamma \Neff}\}$ and $\DIM := k^* + \gamma^2 \Neff^2 \sum_{i>k^*}\lambda_i^2.$
\end{lemma}

\begin{proof}

Recall that \(\bm{\eta}^{\fluct}_t\) satisfies
\[
\bm{\eta}^{\fluct}_t
=
(\Ib-\gamma_t\xb_t\xb_t^\top)\bm{\eta}^{\fluct}_{t-1}
+
\gamma_t z_t \xb_t\xb_t^\top \bm{\delta}
-
\gamma_t p \Hb \bm{\delta},
\qquad
\bm{\eta}^{\fluct}_0=\boldsymbol{0}.
\]
We decompose the driving term as
\[
z_t \xb_t\xb_t^\top \bm{\delta}-p\Hb\bm{\delta}
=
z_t(\xb_t\xb_t^\top-\Hb)\bm{\delta}
+
(z_t-p)\Hb\bm{\delta}.
\]
Accordingly, define two auxiliary processes
\[
\bm{\eta}^{\fluct}_t
=
\bm{\eta}^{\fluct,1}_t+\bm{\eta}^{\fluct,2}_t,
\qquad t\ge 0,
\]
where
\[
\begin{cases}
\bm{\eta}^{\fluct,1}_t
=
(\Ib-\gamma_t\xb_t\xb_t^\top)\bm{\eta}^{\fluct,1}_{t-1}
+
\gamma_t z_t(\xb_t\xb_t^\top-\Hb)\bm{\delta},
\\[0.5ex]
\bm{\eta}^{\fluct,1}_0=\boldsymbol{0},
\end{cases}
\quad
\begin{cases}
\bm{\eta}^{\fluct,2}_t
=
(\Ib-\gamma_t\xb_t\xb_t^\top)\bm{\eta}^{\fluct,2}_{t-1}
+
\gamma_t (z_t-p)\Hb\bm{\delta},
\\[0.5ex]
\bm{\eta}^{\fluct,2}_0=\boldsymbol{0}.
\end{cases}
\]
Then by Cauchy--Schwarz,
\begin{align}
\big\langle \Hb,\Fb_N\big\rangle
&=
\big\langle
\Hb,
\EE\big[
(\bm{\eta}^{\fluct,1}_N+\bm{\eta}^{\fluct,2}_N)
\otimes
(\bm{\eta}^{\fluct,1}_N+\bm{\eta}^{\fluct,2}_N)
\big]
\big\rangle
\notag\\
&\le
2\big\langle \Hb,\Fb_N^{(1)}\big\rangle
+
2\big\langle \Hb,\Fb_N^{(2)}\big\rangle ,
\label{eq:fluct-split-risk}
\end{align}
where
\[
\Fb_t^{(1)}
:=
\EE[\bm{\eta}^{\fluct,1}_t\otimes \bm{\eta}^{\fluct,1}_t],
\qquad
\Fb_t^{(2)}
:=
\EE[\bm{\eta}^{\fluct,2}_t\otimes \bm{\eta}^{\fluct,2}_t].
\]

We first analyze the term \(\Fb_t^{(1)}\). 
Expanding the outer product gives
\begin{align*}
\Fb_t^{(1)}
&=
\EE\Big[
\Big((\Ib-\gamma_t\xb_t\xb_t^\top)\bm{\eta}^{\fluct,1}_{t-1}
+
\gamma_t z_t(\xb_t\xb_t^\top-\Hb)\bm{\delta}\Big)
\\
&\qquad\qquad\otimes
\Big((\Ib-\gamma_t\xb_t\xb_t^\top)\bm{\eta}^{\fluct,1}_{t-1}
+
\gamma_t z_t(\xb_t\xb_t^\top-\Hb)\bm{\delta}\Big)
\Big]
\\
&=
\EE\Big[
(\Ib-\gamma_t\xb_t\xb_t^\top)\bm{\eta}^{\fluct,1}_{t-1}
\otimes
(\Ib-\gamma_t\xb_t\xb_t^\top)\bm{\eta}^{\fluct,1}_{t-1}
\Big]
\\
&\quad
+
\gamma_t^2
\EE\Big[
z_t^2\big((\xb_t\xb_t^\top-\Hb)\bm{\delta}\big)
\otimes
\big((\xb_t\xb_t^\top-\Hb)\bm{\delta}\big)
\Big]
\\
&\quad
+
\gamma_t\EE\Big[
(\Ib-\gamma_t\xb_t\xb_t^\top)\bm{\eta}^{\fluct,1}_{t-1}
\otimes
z_t(\xb_t\xb_t^\top-\Hb)\bm{\delta}
\Big]
\\
&\quad
+
\gamma_t\EE\Big[
z_t(\xb_t\xb_t^\top-\Hb)\bm{\delta}
\otimes
(\Ib-\gamma_t\xb_t\xb_t^\top)\bm{\eta}^{\fluct,1}_{t-1}
\Big].
\end{align*}
Define
\begin{equation}\label{eq:Sigma-fluct1-def}
\bSigma_t^{\fluct,1}
:=
\EE\Big[
z_t^2\big((\xb_t\xb_t^\top-\Hb)\bm{\delta}\big)
\otimes
\big((\xb_t\xb_t^\top-\Hb)\bm{\delta}\big)
\Big].
\end{equation}
Then
\begin{align}
\Fb_t^{(1)}
&=
(\cI-\gamma_t\cT_t)\circ \Fb_{t-1}^{(1)}
+
\gamma_t^2\bSigma_t^{\fluct,1}
+
\Cb_t^{(1)}+(\Cb_t^{(1)})^\top,
\label{eq:F1-recursion-with-cross}
\end{align}
where
\[
\Cb_t^{(1)}
:=
\gamma_t\EE\Big[
(\Ib-\gamma_t\xb_t\xb_t^\top)\bm{\eta}^{\fluct,1}_{t-1}
\otimes
z_t(\xb_t\xb_t^\top-\Hb)\bm{\delta}
\Big].
\]

We now prove that \(\Cb_t^{(1)}=\boldsymbol{0}\).
Let
\[
\mathcal Z:=\sigma(z_1,\ldots,z_N)
\]
be the sigma-field generated by the entire source schedule. 
We first claim that
\begin{equation}
\EE\big[
\bm{\eta}^{\fluct,1}_t
\mid
\mathcal Z
\big]
=
\boldsymbol{0},
\qquad
\forall t\ge 0.
\label{eq:fluct1-cond-mean-zero}
\end{equation}
Indeed, the claim is trivial for \(t=0\). Suppose it holds at time \(t-1\).
Using the recursion
\[
\bm{\eta}^{\fluct,1}_t
=
(\Ib-\gamma_t\xb_t\xb_t^\top)\bm{\eta}^{\fluct,1}_{t-1}
+
\gamma_t z_t(\xb_t\xb_t^\top-\Hb)\bm{\delta},
\]
and using the independence of \(\xb_t\) from
\(\mathcal Z\vee\sigma(\xb_1,\ldots,\xb_{t-1})\), we have
\begin{align*}
\EE\big[
\bm{\eta}^{\fluct,1}_t
\mid
\mathcal Z
\big]
&=
\EE\Big[
\EE\big[
(\Ib-\gamma_t\xb_t\xb_t^\top)\bm{\eta}^{\fluct,1}_{t-1}
+
\gamma_t z_t(\xb_t\xb_t^\top-\Hb)\bm{\delta}
\mid
\mathcal Z,\xb_1,\ldots,\xb_{t-1}
\big]
\mid
\mathcal Z
\Big]
\\
&=
\EE\Big[
(\Ib-\gamma_t\Hb)\bm{\eta}^{\fluct,1}_{t-1}
+
\gamma_t z_t\EE\big[
(\xb_t\xb_t^\top-\Hb)\bm{\delta}
\big]
\mid
\mathcal Z
\Big]
\\
&=
(\Ib-\gamma_t\Hb)
\EE\big[
\bm{\eta}^{\fluct,1}_{t-1}
\mid
\mathcal Z
\big]
\\
&=
\boldsymbol{0}.
\end{align*}
This proves \eqref{eq:fluct1-cond-mean-zero} by induction.

We now show that the cross term vanishes. Define, for any deterministic vector
\(\ub\), the linear map
\[
\mathcal L_t(\ub)
:=
\EE_{\xb_t}\Big[
(\Ib-\gamma_t\xb_t\xb_t^\top)\ub
\otimes
(\xb_t\xb_t^\top-\Hb)\bm{\delta}
\Big].
\]
The map \(\mathcal L_t(\cdot)\) is linear in \(\ub\). Since \(z_t\) is
\(\mathcal Z\)-measurable, and since \(\xb_t\) is independent of
\(\mathcal Z\vee\sigma(\xb_1,\ldots,\xb_{t-1})\), we obtain
\begin{align*}
\Cb_t^{(1)}
&=
\gamma_t\EE\Big[
(\Ib-\gamma_t\xb_t\xb_t^\top)
\bm{\eta}^{\fluct,1}_{t-1}
\otimes
z_t(\xb_t\xb_t^\top-\Hb)\bm{\delta}
\Big]
\\
&=
\gamma_t
\EE\Big[
\EE\Big[
(\Ib-\gamma_t\xb_t\xb_t^\top)
\bm{\eta}^{\fluct,1}_{t-1}
\otimes
z_t(\xb_t\xb_t^\top-\Hb)\bm{\delta}
\mid
\mathcal Z,\xb_1,\ldots,\xb_{t-1}
\Big]
\Big]
\\
&=
\gamma_t
\EE\Big[
z_t\,
\mathcal L_t\big(
\bm{\eta}^{\fluct,1}_{t-1}
\big)
\Big]
\\
&=
\gamma_t
\EE\Big[
\EE\Big[
z_t\,
\mathcal L_t\big(
\bm{\eta}^{\fluct,1}_{t-1}
\big)
\mid
\mathcal Z
\Big]
\Big]
\\
&=
\gamma_t
\EE\Big[
z_t\,
\mathcal L_t\Big(
\EE[
\bm{\eta}^{\fluct,1}_{t-1}
\mid
\mathcal Z
]
\Big)
\Big]
\\
&=
\boldsymbol{0},
\end{align*}
where the penultimate equality uses the linearity of \(\mathcal L_t\), and the
last equality follows from \eqref{eq:fluct1-cond-mean-zero}. Hence
\[
\Cb_t^{(1)}=\boldsymbol{0}.
\]
Taking transpose also gives
\[
(\Cb_t^{(1)})^\top=\boldsymbol{0}.
\]

Therefore \eqref{eq:F1-recursion-with-cross} reduces to
\begin{equation}\label{eq:F1-recursion}
\Fb_t^{(1)}
=
(\cI-\gamma_t\cT_t)\circ \Fb_{t-1}^{(1)}
+
\gamma_t^2\bSigma_t^{\fluct,1},
\qquad
\Fb_0^{(1)}=\boldsymbol{0}.
\end{equation}

Since \(z_t\in\{0,1\}\), we have $\EE[z_t^2]=\EE[z_t]=p$. Moreover,
\begin{align*}
\bSigma_t^{\fluct,1}
&=
\EE\Big[
z_t^2\big((\xb_t\xb_t^\top-\Hb)\bm{\delta}\big)
\otimes
\big((\xb_t\xb_t^\top-\Hb)\bm{\delta}\big)
\Big]
\\
&\preceq
p\,\EE\Big[
\big((\xb_t\xb_t^\top-\Hb)\bm{\delta}\big)
\otimes
\big((\xb_t\xb_t^\top-\Hb)\bm{\delta}\big)
\Big].
\end{align*}
Write \(\Ab:=\bm{\delta}\otimes\bm{\delta}\). Then
\begin{align*}
&\EE\Big[
\big((\xb_t\xb_t^\top-\Hb)\bm{\delta}\big)
\otimes
\big((\xb_t\xb_t^\top-\Hb)\bm{\delta}\big)
\Big]
\\
&\qquad=
\EE\Big[
(\xb_t\xb_t^\top-\Hb)\Ab(\xb_t\xb_t^\top-\Hb)
\Big]
\\
&\qquad=
\EE[\xb_t\xb_t^\top \Ab \xb_t\xb_t^\top]
-
\Hb\Ab\Hb.
\end{align*}
Since \(\Hb\Ab\Hb\succeq 0\), it follows that
\begin{align}
\bSigma_t^{\fluct,1}
&\preceq
p\,\EE[\xb_t\xb_t^\top \Ab \xb_t\xb_t^\top].
\label{eq:Sigma-fluct1-first-bound}
\end{align}
Now apply Assumption~\ref{assump:fourth-moment} with the PSD matrix
\(\Ab=\bm{\delta}\otimes\bm{\delta}\):
\begin{align}
\EE[\xb_t\xb_t^\top \Ab \xb_t\xb_t^\top]
&\preceq
\,\alpha\,\tr(\Hb\Ab)\,\Hb
=
\,\alpha\,\|\bm{\delta}\|_{\Hb}^2\,\Hb.
\label{eq:Sigma-fluct1-second-bound}
\end{align}
Combining \eqref{eq:Sigma-fluct1-first-bound} and \eqref{eq:Sigma-fluct1-second-bound}, we obtain
\begin{equation}\label{eq:Sigma-fluct1-final-bound}
\bSigma_t^{\fluct,1}
\preceq
p\,\alpha\|\bm{\delta}\|_{\Hb}^2\,\Hb.
\end{equation}

Comparing \eqref{eq:F1-recursion} with the variance recursion \eqref{eq:update_Ct}, we see that
\(\Fb_t^{(1)}\) satisfies exactly the same recursion form as the variance term, with noise covariance
\(\bSigma_t^{\fluct,1}\) in place of \(\bSigma_t\). By \eqref{eq:Sigma-fluct1-final-bound}, we may apply
Lemma~\ref{thm:HC-decay-lr-upper-bound} with
\[
\sigma^2
\leftarrow
p\alpha\|\bm{\delta}\|_{\Hb}^2.
\]
Therefore,
\begin{equation}\label{eq:F1-final-bound}
\big\langle \Hb,\Fb_N^{(1)}\big\rangle
\le
\frac{8p\alpha\|\bm{\delta}\|_{\Hb}^2}{1-\gamma\alpha\tr(\Hb)}
\cdot
\frac{\DIM}{\Neff},
\end{equation}
where
\[
k^*=\max\left\{k:\lambda_k\ge \frac{1}{\gamma\Neff}\right\},
\qquad
\DIM:=k^*+\gamma^2\Neff^2\sum_{i>k^*}\lambda_i^2.
\]

Consequently, by \eqref{eq:fluct-split-risk},
\[
\big\langle \Hb,\Fb_N\big\rangle
\le
2\big\langle \Hb,\Fb_N^{(1)}\big\rangle
+
2\big\langle \Hb,\Fb_N^{(2)}\big\rangle,
\]
and the first term is controlled by \eqref{eq:F1-final-bound}. The second term
\(\big\langle \Hb,\Fb_N^{(2)}\big\rangle\) is be bounded by Lemma~\ref{lem:fluctuation-second-part-upper}.

\end{proof}

\begin{lemma}[Fixed-budget variance identity]
\label{lem:fixed-budget-variance-identity-wu-direct-final}
Let \(m\in\{0,1,\dots,N\}\), and set
\[
p:=\frac{m}{N}.
\]
Suppose that \((z_1,\dots,z_N)\) is uniformly distributed over all binary vectors in \(\{0,1\}^N\) with exactly \(m\) ones. Let \(\bm{v}_1,\dots,\bm{v}_N\in\mathbb{R}^{\mathcal H}\) be deterministic vectors, and define
\[
\bar{\bm{v}}:=\frac1N\sum_{s=1}^N \bm{v}_s.
\]
Then
\begin{equation}\label{eq:fixed-budget-variance-identity}
\EE\left\|
\sum_{s=1}^N (z_s-p)\bm{v}_s
\right\|_{\Hb}^2
=
\frac{Np(1-p)}{N-1}
\sum_{s=1}^N \|\bm{v}_s-\bar{\bm{v}}\|_{\Hb}^2.
\end{equation}
In particular,
\begin{equation}\label{eq:fixed-budget-variance-identity-upper}
\EE\left\|
\sum_{s=1}^N (z_s-p)\bm{v}_s
\right\|_{\Hb}^2
\le
\frac{Np(1-p)}{N-1}
\sum_{s=1}^N \|\bm{v}_s\|_{\Hb}^2.
\end{equation}
\end{lemma}

\begin{proof}
Since the law of \((z_1,\dots,z_N)\) is exchangeable and exactly \(m\) coordinates are equal to \(1\), we have
\[
\EE[z_s]=\frac{m}{N}=p,
\qquad s=1,\dots,N.
\]
Moreover, because \(z_s\in\{0,1\}\),
\[
\EE[(z_s-p)^2]
=
\mathrm{Var}(z_s)
=
p(1-p).
\]
For \(s\neq r\), using the fixed-budget constraint \(\sum_{t=1}^N z_t=m\), we also have
\[
\EE[z_sz_r]
=
\frac{m(m-1)}{N(N-1)}.
\]
Therefore
\begin{align*}
\EE[(z_s-p)(z_r-p)]
&=
\EE[z_sz_r]-p^2 \\
&=
\frac{m(m-1)}{N(N-1)}-\frac{m^2}{N^2} \\
&=
-\frac{m(N-m)}{N^2(N-1)} \\
&=
-\frac{p(1-p)}{N-1}.
\end{align*}

Now expand the quadratic form:
\begin{align*}
\EE\left\|
\sum_{s=1}^N (z_s-p)\bm{v}_s
\right\|_{\Hb}^2
&=
\EE\left[
\sum_{s=1}^N\sum_{r=1}^N
(z_s-p)(z_r-p)\langle \bm{v}_s,\bm{v}_r\rangle_{\Hb}
\right]
\\
&=
\sum_{s=1}^N \EE[(z_s-p)^2]\|\bm{v}_s\|_{\Hb}^2
+
\sum_{s\neq r}\EE[(z_s-p)(z_r-p)]\langle \bm{v}_s,\bm{v}_r\rangle_{\Hb}
\\
&=
p(1-p)\sum_{s=1}^N \|\bm{v}_s\|_{\Hb}^2
-
\frac{p(1-p)}{N-1}
\sum_{s\neq r}\langle \bm{v}_s,\bm{v}_r\rangle_{\Hb}.
\end{align*}
Thus
\begin{equation}\label{eq:fixed-budget-proof-mid}
\EE\left\|
\sum_{s=1}^N (z_s-p)\bm{v}_s
\right\|_{\Hb}^2
=
\frac{p(1-p)}{N-1}
\left[
(N-1)\sum_{s=1}^N \|\bm{v}_s\|_{\Hb}^2
-
\sum_{s\neq r}\langle \bm{v}_s,\bm{v}_r\rangle_{\Hb}
\right].
\end{equation}

On the other hand,
\begin{align*}
\sum_{s=1}^N \|\bm{v}_s-\bar{\bm{v}}\|_{\Hb}^2
&=
\sum_{s=1}^N \|\bm{v}_s\|_{\Hb}^2
-
2\sum_{s=1}^N \langle \bm{v}_s,\bar{\bm{v}}\rangle_{\Hb}
+
\sum_{s=1}^N \|\bar{\bm{v}}\|_{\Hb}^2
\\
&=
\sum_{s=1}^N \|\bm{v}_s\|_{\Hb}^2
-
N\|\bar{\bm{v}}\|_{\Hb}^2.
\end{align*}
Since
\[
N^2\|\bar{\bm{v}}\|_{\Hb}^2
=
\left\|\sum_{s=1}^N \bm{v}_s\right\|_{\Hb}^2
=
\sum_{s=1}^N \|\bm{v}_s\|_{\Hb}^2
+
\sum_{s\neq r}\langle \bm{v}_s,\bm{v}_r\rangle_{\Hb},
\]
it follows that
\begin{align*}
N\sum_{s=1}^N \|\bm{v}_s-\bar{\bm{v}}\|_{\Hb}^2
&=
N\sum_{s=1}^N \|\bm{v}_s\|_{\Hb}^2
-
\left(
\sum_{s=1}^N \|\bm{v}_s\|_{\Hb}^2
+
\sum_{s\neq r}\langle \bm{v}_s,\bm{v}_r\rangle_{\Hb}
\right)
\\
&=
(N-1)\sum_{s=1}^N \|\bm{v}_s\|_{\Hb}^2
-
\sum_{s\neq r}\langle \bm{v}_s,\bm{v}_r\rangle_{\Hb}.
\end{align*}
Substituting this identity into \eqref{eq:fixed-budget-proof-mid} yields
\[
\EE\left\|
\sum_{s=1}^N (z_s-p)\bm{v}_s
\right\|_{\Hb}^2
=
\frac{Np(1-p)}{N-1}
\sum_{s=1}^N \|\bm{v}_s-\bar{\bm{v}}\|_{\Hb}^2,
\]
which proves \eqref{eq:fixed-budget-variance-identity}.

Finally, since
\[
\sum_{s=1}^N \|\bm{v}_s-\bar{\bm{v}}\|_{\Hb}^2
=
\sum_{s=1}^N \|\bm{v}_s\|_{\Hb}^2
-
N\|\bar{\bm{v}}\|_{\Hb}^2
\le
\sum_{s=1}^N \|\bm{v}_s\|_{\Hb}^2,
\]
we obtain \eqref{eq:fixed-budget-variance-identity-upper}.
\end{proof}

\begin{lemma}[An upper bound for the second fluctuation term]
\label{lem:fluctuation-second-part-upper}
Recall that
\[
\bm{\eta}^{\fluct,2}_t
=
(\Ib-\gamma_t\xb_t\xb_t^\top)\bm{\eta}^{\fluct,2}_{t-1}
+
\gamma_t (z_t-p)\Hb\bm{\delta},
\qquad
\bm{\eta}^{\fluct,2}_0=\boldsymbol{0},
\]
and define
\[
\Fb_t^{(2)}
:=
\EE[\bm{\eta}^{\fluct,2}_t\otimes \bm{\eta}^{\fluct,2}_t].
\]
Suppose Assumption~\ref{assump:fourth-moment} holds and $\gamma<\frac{1}{\alpha\tr(\Hb)}.$ Let
$\Neff:=N/\log N,$
Then
\[
\big\langle \Hb,\Fb_N^{(2)}\big\rangle
\le
\frac{8Np(1-p)}{(N-1)\bigl(1-\alpha\gamma\tr(\Hb)\bigr)}
\|\bm{\delta}\|_{\Hb}^2\cdot \frac{\DIM}{\Neff},
\]
where $k^* = \max \{k: \lambda_k \ge \frac{1}{\gamma \Neff}\}$ and $\DIM := k^* + \gamma^2 \Neff^2 \sum_{i>k^*}\lambda_i^2.$
\end{lemma}

\begin{proof}
For each \(t=1,\dots,N\), define
\[
\Bb_t:=\Ib-\gamma_t\xb_t\xb_t^\top.
\]
Iterating the recursion for \(\bm{\eta}^{\fluct,2}_t\) yields
\[
\bm{\eta}^{\fluct,2}_N
=
\sum_{s=1}^N
\Bb_N\Bb_{N-1}\cdots \Bb_{s+1}\,
\gamma_s(z_s-p)\Hb\bm{\delta},
\]
where the empty product is interpreted as \(\Ib\). Therefore
\[
\bm{\eta}^{\fluct,2}_N
=
\sum_{s=1}^N (z_s-p)\bm{v}_s,
\qquad
\bm{v}_s
:=
\gamma_s\Bb_N\Bb_{N-1}\cdots \Bb_{s+1}\Hb\bm{\delta}.
\]

Conditional on the feature sequence \((\xb_1,\dots,\xb_N)\), the vectors
\(\bm{v}_1,\dots,\bm{v}_N\) are deterministic. Since the source indicators are independent
of the features and follow the fixed-budget law, Lemma~\ref{lem:fixed-budget-variance-identity-wu-direct-final}
gives
\[
\EE\!\left[
\|\bm{\eta}^{\fluct,2}_N\|_{\Hb}^2
\,\middle|\, \xb_1,\dots,\xb_N
\right]
\le
\frac{Np(1-p)}{N-1}\sum_{s=1}^N \|\bm{v}_s\|_{\Hb}^2.
\]
Taking expectation over the features gives
\begin{equation}\label{eq:F2-step1}
\big\langle \Hb,\Fb_N^{(2)}\big\rangle
=
\EE\!\left[\|\bm{\eta}^{\fluct,2}_N\|_{\Hb}^2\right]
\le
\frac{Np(1-p)}{N-1}
\sum_{s=1}^N \EE\!\left[\|\bm{v}_s\|_{\Hb}^2\right].
\end{equation}

For each \(t=0,1,\dots,N\) and each \(s\in\{1,\dots,N\}\), define
\[
\bm{v}_{s,t}
:=
\begin{cases}
\gamma_s \Bb_t\Bb_{t-1}\cdots \Bb_{s+1}\Hb\bm{\delta}, & 1\le s\le t,\\[1mm]
\boldsymbol{0}, & s>t.
\end{cases}
\]
In particular, \(\bm{v}_{s,N}=\bm{v}_s\). Now define
\[
\Sb_t:=\sum_{s=1}^t \bm{v}_{s,t}\bm{v}_{s,t}^\top,
\qquad
t=0,1,\dots,N,
\]
with the convention \(\Sb_0:=\boldsymbol{0}\). Then
\[
\langle \Hb,\Sb_N\rangle
=
\sum_{s=1}^N \langle \Hb,\bm{v}_s\bm{v}_s^\top\rangle
=
\sum_{s=1}^N \|\bm{v}_s\|_{\Hb}^2.
\]
Moreover, for \(s<t\),
\[
\bm{v}_{s,t}=\Bb_t\bm{v}_{s,t-1},
\qquad\text{while}\qquad
\bm{v}_{t,t}=\gamma_t\Hb\bm{\delta}.
\]
Hence
\[
\Sb_t
=
\Bb_t\Sb_{t-1}\Bb_t^\top
+
\gamma_t^2 \Ob,
\qquad
\Ob:=\Hb\bm{\delta}\bm{\delta}^\top\Hb.
\]

Define $\bar{\Sb}_t:=\EE[\Sb_t].$
Since \(\Sb_{t-1}\) is measurable with respect to
\(\sigma(\xb_1,\dots,\xb_{t-1})\) and \(\xb_t\) is independent of the past,
taking conditional expectation yields
\[
\bar{\Sb}_t
=
(\cI-\gamma_t\cT_t)\circ \bar{\Sb}_{t-1}
+
\gamma_t^2 \Ob,
\qquad
\bar{\Sb}_0=\boldsymbol{0}.
\]

We first bound the forcing matrix \(\Ob\) by a multiple of \(\Hb\). For any vector \(\ub\),
\[
\ub^\top \Ob \ub
=
(\bm{\delta}^\top \Hb \ub)^2
\le
(\bm{\delta}^\top \Hb \bm{\delta})(\ub^\top \Hb \ub)
=
\|\bm{\delta}\|_{\Hb}^2\,\ub^\top \Hb \ub,
\]
where we used Cauchy--Schwarz in the inner product induced by \(\Hb\). Therefore
\begin{equation}\label{eq:O-upper-H}
\Ob\preceq \|\bm{\delta}\|_{\Hb}^2\,\Hb.
\end{equation}

Now define an auxiliary matrix sequence \((\Cb_t^\delta)_{t=0}^N\) by
\[
\Cb_0^\delta:=\boldsymbol{0},
\qquad
\Cb_t^\delta
=
(\cI-\gamma_t\cT_t)\circ \Cb_{t-1}^\delta
+
\gamma_t^2 \|\bm{\delta}\|_{\Hb}^2\,\Hb,
\qquad
t=1,\dots,N.
\]
By Lemma~\ref{lemma:operators}, the operator \((\cI-\gamma_t\cT_t)\) is a PSD mapping.
Since \(\bar{\Sb}_0=\Cb_0^\delta=\boldsymbol{0}\) and \eqref{eq:O-upper-H} holds, an induction on \(t\)
shows that
\[
\bar{\Sb}_t\preceq \Cb_t^\delta,
\qquad
t=0,1,\dots,N.
\]
Consequently,
\begin{equation}\label{eq:F2-compare}
\sum_{s=1}^N \EE\!\left[\|\bm{v}_s\|_{\Hb}^2\right]
=
\langle \Hb,\bar{\Sb}_N\rangle
\le
\langle \Hb,\Cb_N^\delta\rangle.
\end{equation}

The recursion defining \(\Cb_t^\delta\) has exactly the same form as the variance recursion
\eqref{eq:update_Ct}, with noise covariance
$
\bSigma^\delta:=\|\bm{\delta}\|_{\Hb}^2\,\Hb.
$ Applying Lemma~\ref{thm:HC-decay-lr-upper-bound} with
\[
\sigma^2\leftarrow \|\bm{\delta}\|_{\Hb}^2
\]
yields
\[
\langle \Hb,\Cb_N^\delta\rangle
\le
\frac{8\|\bm{\delta}\|_{\Hb}^2}{1-\alpha\gamma\tr(\Hb)}
\cdot
\frac{\DIM}{\Neff}.
\]
Combining this with \eqref{eq:F2-step1} and \eqref{eq:F2-compare}, we obtain
\[
\big\langle \Hb,\Fb_N^{(2)}\big\rangle
\le
\frac{Np(1-p)}{N-1}
\cdot
\frac{8\|\bm{\delta}\|_{\Hb}^2}{1-\alpha\gamma\tr(\Hb)}
\cdot
\frac{\DIM}{\Neff},
\]
\end{proof}

\subsection{Proof of Theorem~\ref{thm:generalization_error}}
\begin{proof}
 Combining Bias upper bound, Variance upper bound, Drift upper bound and Fluctuation upper bound yields the desired result.   
\end{proof}

\section{Mixed Training Lower Bound Analysis}\label{ap:mix-lower}

In this section, we derive a lower bound for the mixed-training excess risk.  
The starting point is again the error recursion
\begin{equation}
\bm{\eta}_t
=
(\Ib-\gamma_t\xb_t\xb_t^\top)\bm{\eta}_{t-1}
+
\gamma_t z_t \xb_t\xb_t^\top\bm{\delta}
+
\gamma_t\xi_t\xb_t,
\qquad
\bm{\eta}_0=\wb_0-\wb_1^*.
\label{eq:lower-error-recursion}
\end{equation}
Recall that
\[
\bm{\eta}_t=\wb_t-\wb_1^* .
\]
By definition of excess risk,
\[
\EE[\R(\wb_N)-\R(\wb_1^*)]
=
\frac12\,\big\langle \Hb,\EE[\bm{\eta}_N\otimes\bm{\eta}_N]\big\rangle .
\]

To derive a lower bound, we directly center the total error process \(\bm{\eta}_t\) around its mean. Define
\[
\bar{\bm{\eta}}_t:=\EE[\bm{\eta}_t],
\qquad
\widetilde{\bm{\eta}}_t:=\bm{\eta}_t-\bar{\bm{\eta}}_t.
\]
Then
\[
\bm{\eta}_t=\bar{\bm{\eta}}_t+\widetilde{\bm{\eta}}_t,
\qquad
\EE[\widetilde{\bm{\eta}}_t]=\boldsymbol{0}.
\]

Therefore,
\begin{align*}
\EE[\bm{\eta}_N\otimes\bm{\eta}_N]
&=
\EE\big[
(\bar{\bm{\eta}}_N+\widetilde{\bm{\eta}}_N)
\otimes
(\bar{\bm{\eta}}_N+\widetilde{\bm{\eta}}_N)
\big]
\\
&=
\bar{\bm{\eta}}_N\otimes\bar{\bm{\eta}}_N
+
\EE[\widetilde{\bm{\eta}}_N\otimes\widetilde{\bm{\eta}}_N]
+
\bar{\bm{\eta}}_N\otimes\EE[\widetilde{\bm{\eta}}_N]
+
\EE[\widetilde{\bm{\eta}}_N]\otimes\bar{\bm{\eta}}_N
\\
&=
\bar{\bm{\eta}}_N\otimes\bar{\bm{\eta}}_N
+
\EE[\widetilde{\bm{\eta}}_N\otimes\widetilde{\bm{\eta}}_N].
\end{align*}
Hence
\begin{equation}
\EE[\R(\wb_N)-\R(\wb_1^*)]
=
\frac12\|\bar{\bm{\eta}}_N\|_{\Hb}^2
+
\frac12\big\langle \Hb,\EE[\widetilde{\bm{\eta}}_N\otimes\widetilde{\bm{\eta}}_N]\big\rangle.
\label{eq:lower-bound-by-mean}
\end{equation}

\subsection{The mean recursion}

Taking expectation on both sides of \eqref{eq:lower-error-recursion}, and using that
\[
\EE[\xb_t\xb_t^\top]=\Hb,
\qquad
\EE[z_t]=p,
\qquad
\EE[\xi_t\xb_t]=\boldsymbol{0},
\]
we obtain
\begin{equation*}
\bar{\bm{\eta}}_t
=
(\Ib-\gamma_t\Hb)\bar{\bm{\eta}}_{t-1}
+
\gamma_t p \Hb\bm{\delta},
\qquad
\bar{\bm{\eta}}_0=\wb_0-\wb_1^*.
\label{eq:mean-recursion-lower}
\end{equation*}

Unrolling the recursion gives
\begin{align}
\bar{\bm{\eta}}_N
&=
\prod_{t=1}^N(\Ib-\gamma_t\Hb)(\wb_0-\wb_1^*)
+
p\sum_{s=1}^N
\gamma_s
\left(
\prod_{t=s+1}^N(\Ib-\gamma_t\Hb)
\right)\Hb\bm{\delta}.
\label{eq:mean-unrolled-lower}
\end{align}
Since each factor \((\Ib-\gamma_t\Hb)\) is a polynomial in \(\Hb\), it commutes with \(\Hb\). Therefore,
\[
\sum_{s=1}^N
\gamma_s
\left(
\prod_{t=s+1}^N(\Ib-\gamma_t\Hb)
\right)\Hb
=
\Ib-\prod_{t=1}^N(\Ib-\gamma_t\Hb).
\]
Substituting this identity into \eqref{eq:mean-unrolled-lower}, we get the closed form
\begin{equation*}
\bar{\bm{\eta}}_N
=
\prod_{t=1}^{N}(\Ib-\gamma_t\Hb)(\wb_0-\wb_1^*)
+
p\Bigl(\Ib-\prod_{t=1}^{N}(\Ib-\gamma_t\Hb)\Bigr)\bm{\delta}.
\label{eq:mean-closed-form-lower}
\end{equation*}

Consequently,
\begin{equation}
\|\bar{\bm{\eta}}_N\|_{\Hb}^2
=
\Big\|
\prod_{t=1}^{N}(\Ib-\gamma_t\Hb) (\wb_0 - \wb_1^*)
+
p\big(\Ib-\prod_{t=1}^{N}(\Ib-\gamma_t\Hb)\big)\bm{\delta}
\Big\|^2_{\Hb}.
\label{eq:mean-risk-term-lower}
\end{equation}

The remaining task is to lower bound the centered fluctuation term
\[
\frac12\big\langle \Hb,\EE[\widetilde{\bm{\eta}}_N\otimes\widetilde{\bm{\eta}}_N]\big\rangle,
\]
which will be handled next.

\subsection{The centered fluctuation term}
First, we introduce the lower bound Theorem of $\big \langle\Hb,\Vb_N\big\rangle$ in \cite{wulast}, which will be used in our analysis.
\begin{lemma}[A variance lower bound]\label{lemma:C-lower-bound}
Suppose Assumptions \ref{assump:second-moment} and \ref{assump:noise} hold.
Let $\Neff = N / \log N$.
Suppose $\Neff \ge 10$ and $\gamma < 1/\lambda_1$. We have 
\[
\big \langle\Hb,\Vb_N\big\rangle \ge \frac{\sigma_1^2}{400}\cdot\frac{\DIM}{\Neff},
\]
where $k^*=\max\left\{k:\lambda_k\ge \frac{1}{\gamma\Neff}\right\},
\qquad
\DIM:=k^*+\gamma^2\Neff^2\sum_{i>k^*}\lambda_i^2.$
\end{lemma}
\begin{proof}
See proof of Theorem C.2 in \cite{wulast}.
\end{proof}

We now continue the analysis of the centered fluctuation term
\[
\frac12\big\langle \Hb,\EE[\widetilde{\bm{\eta}}_N\otimes\widetilde{\bm{\eta}}_N]\big\rangle .
\]

Starting from the total recursion
\[
\bm{\eta}_t
=
(\Ib-\gamma_t\xb_t\xb_t^\top)\bm{\eta}_{t-1}
+
\gamma_t z_t\xb_t\xb_t^\top\bm{\delta}
+
\gamma_t\xi_t\xb_t,
\qquad
\bm{\eta}_0=\wb_0-\wb_1^*,
\]
and subtracting the mean recursion
\[
\bar{\bm{\eta}}_t
=
(\Ib-\gamma_t\Hb)\bar{\bm{\eta}}_{t-1}
+
\gamma_t p\Hb\bm{\delta},
\qquad
\bar{\bm{\eta}}_0=\wb_0-\wb_1^*,
\]
we obtain
\begin{align}
\widetilde{\bm{\eta}}_t
&=
(\Ib-\gamma_t\xb_t\xb_t^\top)\widetilde{\bm{\eta}}_{t-1}
+
\gamma_t\xi_t\xb_t
+
\gamma_t\bigl(z_t\xb_t\xb_t^\top\bm{\delta}-p\Hb\bm{\delta}\bigr)
+
\gamma_t(\Hb-\xb_t\xb_t^\top)\bar{\bm{\eta}}_{t-1}.
\label{eq:centered-eta-recursion}
\end{align}

We further decompose
\[
\widetilde{\bm{\eta}}_t
=
\widetilde{\bm{\eta}}^{\noise}_t+\widetilde{\bm{\eta}}^{\res}_t,
\qquad t=0,1,\dots,N,
\]
where
\[
\begin{cases}
\widetilde{\bm{\eta}}^{\noise}_t
=
(\Ib-\gamma_t\xb_t\xb_t^\top)\widetilde{\bm{\eta}}^{\noise}_{t-1}
+
\gamma_t\xi_t\xb_t,
\\[0.5ex]
\widetilde{\bm{\eta}}^{\noise}_0=\boldsymbol{0},
\end{cases}
\begin{cases}
\widetilde{\bm{\eta}}^{\res}_t
=
(\Ib-\gamma_t\xb_t\xb_t^\top)\widetilde{\bm{\eta}}^{\res}_{t-1}
+
\gamma_t\bigl(z_t\xb_t\xb_t^\top\bm{\delta}-p\Hb\bm{\delta}\bigr)
+
\gamma_t(\Hb-\xb_t\xb_t^\top)\bar{\bm{\eta}}_{t-1},
\\[0.5ex]
\widetilde{\bm{\eta}}^{\res}_0=\boldsymbol{0}.
\end{cases}
\]
Thus
\[
\widetilde{\bm{\eta}}_N
=
\widetilde{\bm{\eta}}^{\noise}_N+\widetilde{\bm{\eta}}^{\res}_N.
\]

Expanding the second moment gives
\begin{align}
\EE[\widetilde{\bm{\eta}}_N\otimes\widetilde{\bm{\eta}}_N]
&=
\EE[\widetilde{\bm{\eta}}^{\noise}_N\otimes\widetilde{\bm{\eta}}^{\noise}_N]
+
\EE[\widetilde{\bm{\eta}}^{\res}_N\otimes\widetilde{\bm{\eta}}^{\res}_N]
\notag\\
&\quad
+
\EE[\widetilde{\bm{\eta}}^{\noise}_N\otimes\widetilde{\bm{\eta}}^{\res}_N]
+
\EE[\widetilde{\bm{\eta}}^{\res}_N\otimes\widetilde{\bm{\eta}}^{\noise}_N].
\label{eq:centered-second-moment-expand}
\end{align}
We now prove that
\begin{equation}
\EE[\widetilde{\bm{\eta}}^{\noise}_t\otimes\widetilde{\bm{\eta}}^{\res}_t]
=
\boldsymbol{0},
\qquad
\forall t\ge 0.
\label{eq:cross-zero-claim}
\end{equation}

Let
\[
\Bb_t:=\Ib-\gamma_t\xb_t\xb_t^\top,
\qquad
\mathcal{G}:=\sigma\big((\xb_s,z_s):1\le s\le N\big)
\]
be the sigma-field generated by the entire feature sequence and source schedule.

We first observe that
\[
\widetilde{\bm{\eta}}^{\res}_t
\quad\text{is } \mathcal{G}\text{-measurable for every }t.
\]
Indeed, \(\widetilde{\bm{\eta}}^{\res}_0=\boldsymbol{0}\), and the recursion
\[
\widetilde{\bm{\eta}}^{\res}_t
=
\Bb_t\widetilde{\bm{\eta}}^{\res}_{t-1}
+
\gamma_t\bigl(z_t\xb_t\xb_t^\top\bm{\delta}-p\Hb\bm{\delta}\bigr)
+
\gamma_t(\Hb-\xb_t\xb_t^\top)\bar{\bm{\eta}}_{t-1}
\]
only involves \((\xb_s,z_s)_{s\le t}\) and the deterministic vector
\(\bar{\bm{\eta}}_{t-1}\). Hence, by induction,
\(\widetilde{\bm{\eta}}^{\res}_t\) is \(\mathcal{G}\)-measurable.

Next, we show that
\begin{equation}
\EE\big[
\widetilde{\bm{\eta}}^{\noise}_t
\mid
\mathcal{G}
\big]
=
\boldsymbol{0},
\qquad
\forall t\ge 0.
\label{eq:noise-cond-mean-zero}
\end{equation}
To see this, unroll the noise recursion
\[
\widetilde{\bm{\eta}}^{\noise}_t
=
\Bb_t\widetilde{\bm{\eta}}^{\noise}_{t-1}
+
\gamma_t\xi_t\xb_t,
\qquad
\widetilde{\bm{\eta}}^{\noise}_0=\boldsymbol{0}.
\]
This gives
\[
\widetilde{\bm{\eta}}^{\noise}_t
=
\sum_{s=1}^t
\Bigl(\Bb_t\Bb_{t-1}\cdots \Bb_{s+1}\Bigr)
\gamma_s \xi_s\xb_s,
\]
where the empty product is understood as the identity matrix.
Conditional on \(\mathcal{G}\), all matrices
\(\Bb_j\) and vectors \(\xb_s\) are deterministic. Moreover, by the zero-mean noise
assumption,
\[
\EE[\xi_s\mid \mathcal{G}]=0,
\qquad s=1,\ldots,N.
\]
Therefore,
\[
\EE\big[
\widetilde{\bm{\eta}}^{\noise}_t
\mid
\mathcal{G}
\big]
=
\sum_{s=1}^t
\Bigl(\Bb_t\Bb_{t-1}\cdots \Bb_{s+1}\Bigr)
\gamma_s \xb_s
\EE[\xi_s\mid \mathcal{G}]
=
\boldsymbol{0},
\]
which proves \eqref{eq:noise-cond-mean-zero}.

Now, since \(\widetilde{\bm{\eta}}^{\res}_t\) is \(\mathcal{G}\)-measurable, we have
\begin{align*}
\EE\big[
\widetilde{\bm{\eta}}^{\noise}_t
\otimes
\widetilde{\bm{\eta}}^{\res}_t
\big]
&=
\EE\Big[
\EE\big[
\widetilde{\bm{\eta}}^{\noise}_t
\otimes
\widetilde{\bm{\eta}}^{\res}_t
\mid
\mathcal{G}
\big]
\Big]
\\
&=
\EE\Big[
\EE\big[
\widetilde{\bm{\eta}}^{\noise}_t
\mid
\mathcal{G}
\big]
\otimes
\widetilde{\bm{\eta}}^{\res}_t
\Big]
\\
&=
\boldsymbol{0}.
\end{align*}
This proves \eqref{eq:cross-zero-claim}. Taking transpose also gives
\[
\EE\big[
\widetilde{\bm{\eta}}^{\res}_t
\otimes
\widetilde{\bm{\eta}}^{\noise}_t
\big]
=
\boldsymbol{0}.
\]
Consequently,
\[
\EE[\widetilde{\bm{\eta}}_N\otimes\widetilde{\bm{\eta}}_N]
=
\EE[\widetilde{\bm{\eta}}^{\noise}_N\otimes
\widetilde{\bm{\eta}}^{\noise}_N]
+
\EE[\widetilde{\bm{\eta}}^{\res}_N\otimes
\widetilde{\bm{\eta}}^{\res}_N].
\]

Hence
\begin{align}
\big\langle \Hb,\EE[\widetilde{\bm{\eta}}_N\otimes\widetilde{\bm{\eta}}_N]\big\rangle
&=
\big\langle \Hb,\EE[\widetilde{\bm{\eta}}^{\noise}_N\otimes\widetilde{\bm{\eta}}^{\noise}_N]\big\rangle
+
\big\langle \Hb,\EE[\widetilde{\bm{\eta}}^{\res}_N\otimes\widetilde{\bm{\eta}}^{\res}_N]\big\rangle.
\label{eq:centered-fluct-lower-split}
\end{align}

We now identify the first term with a variance-type recursion. Define
\[
\Vb_t^{\mix}
:=
\EE[\widetilde{\bm{\eta}}^{\noise}_t\otimes\widetilde{\bm{\eta}}^{\noise}_t].
\]
Then \(\Vb_t^{\mix}\) satisfies
\[
\Vb_t^{\mix}
=
(\cI-\gamma_t\cT_t)\circ \Vb_{t-1}^{\mix}
+
\gamma_t^2 \bar{\bSigma},
\qquad
\Vb_0^{\mix}=\boldsymbol{0},
\]
where
\[
\bar{\bSigma}
:=
\EE[\xi_t^2\xb_t\xb_t^\top]
=
(1-p)\bSigma_1+p\bSigma_2.
\]
Under Assumption~\ref{assump:noise}, we have
\[
\bSigma_1=\sigma_1^2\Hb,
\qquad
\bSigma_2=\sigma_2^2\Hb,
\]
and therefore
\[
\bar{\bSigma}
=
\bigl((1-p)\sigma_1^2+p\sigma_2^2\bigr)\Hb.
\]
Denote
\[
\bar{\sigma}^2:=(1-p)\sigma_1^2+p\sigma_2^2.
\]
Then \(\Vb_t^{\mix}\) is exactly the same variance recursion as in Lemma~\ref{lemma:C-lower-bound}, with \(\sigma_1^2\) replaced by \(\bar{\sigma}^2\). Hence
\begin{equation}
\big\langle \Hb,\EE[\widetilde{\bm{\eta}}^{\noise}_N\otimes\widetilde{\bm{\eta}}^{\noise}_N]\big\rangle
=
\big\langle \Hb,\Vb_N^{\mix}\big\rangle
\ge
\frac{\bar{\sigma}^2}{400}\cdot\frac{\DIM}{\Neff}.
\label{eq:variance-part-lower}
\end{equation}

The remaining term
\[
\big\langle \Hb,\EE[\widetilde{\bm{\eta}}^{\res}_N\otimes\widetilde{\bm{\eta}}^{\res}_N]\big\rangle
\]
will be lower bounded separately by the following Lemma.

\begin{lemma}[A lower bound for the residual fluctuation term]
\label{lem:residual-fluctuation-lower-bound}
Recall that
\[
\widetilde{\bm{\eta}}^{\res}_t
=
(\Ib-\gamma_t\xb_t\xb_t^\top)\widetilde{\bm{\eta}}^{\res}_{t-1}
+
\gamma_t\bigl(z_t\xb_t\xb_t^\top\bm{\delta}-p\Hb\bm{\delta}\bigr)
+
\gamma_t(\Hb-\xb_t\xb_t^\top)\bar{\bm{\eta}}_{t-1},
\qquad
\widetilde{\bm{\eta}}^{\res}_0=\boldsymbol{0},
\]
Suppose Assumptions~\ref{assump:second-moment},
\ref{assump:fourth-moment} and \ref{assump:noise} hold.
Let $\Neff:=N/\log N,$ Assume \(N\ge 500\) and \(\gamma<1/\lambda_1\). Then
\begin{equation*}
\big\langle \Hb,\EE[\widetilde{\bm{\eta}}^{\res}_N\otimes\widetilde{\bm{\eta}}^{\res}_N]\big\rangle
\ge
\frac{\beta e^{-8}}{400}\,
\|\wb_0-\wb_1^*-p\bm{\delta}\|_{\Hb_{k^*:\infty}}^2
\cdot
\frac{\DIM}{\Neff}
+
\frac{(e^{-4}-2^{-7})^2}{40}\,
p(1-p)\gamma^2\Neff
\sum_{i>k^*}\lambda_i^3\delta_i^2,
\label{eq:residual-fluctuation-lower}
\end{equation*}
where $k^*=\max\left\{k:\lambda_k\ge \frac{1}{\gamma\Neff}\right\},
\qquad
\DIM:=k^*+\gamma^2\Neff^2\sum_{i>k^*}\lambda_i^2.$
\end{lemma}
\begin{proof}
Let$
\cb_{t-1}=p\bm{\delta}-\bar{\bm{\eta}}_{t-1},$
we can rewrite the driving term in the recursion of \(\widetilde{\bm{\eta}}^{\res}_t\) as
\[
z_t\xb_t\xb_t^\top\bm{\delta}-p\Hb\bm{\delta}
+
(\Hb-\xb_t\xb_t^\top)\bar{\bm{\eta}}_{t-1}
=
(z_t-p)\xb_t\xb_t^\top\bm{\delta}
+
(\xb_t\xb_t^\top-\Hb)\cb_{t-1}.
\]
Hence
\begin{equation}
\widetilde{\bm{\eta}}^{\res}_t
=
(\Ib-\gamma_t\xb_t\xb_t^\top)\widetilde{\bm{\eta}}^{\res}_{t-1}
+
\gamma_t(z_t-p)\xb_t\xb_t^\top\bm{\delta}
+
\gamma_t(\xb_t\xb_t^\top-\Hb)\cb_{t-1},
\qquad
\widetilde{\bm{\eta}}^{\res}_0=\boldsymbol{0}.
\label{eq:res-rec-rewrite}
\end{equation}
On the other hand, from the closed form of the mean process,
\[
\bar{\bm{\eta}}_t
=
(\Ib-\gamma_t\Hb)\bar{\bm{\eta}}_{t-1}
+
\gamma_t p\Hb\bm{\delta},
\]
we obtain
\[
\cb_t
=
p\bm{\delta}-\bar{\bm{\eta}}_t
=
(\Ib-\gamma_t\Hb)(p\bm{\delta}-\bar{\bm{\eta}}_{t-1})
=
(\Ib-\gamma_t\Hb)\cb_{t-1},
\]
In particular,
\[
\cb_0=p\bm{\delta}-(\wb_0-\wb_1^*)
=-(\wb_0-\wb_1^*-p\bm{\delta}).
\]

Let
\[
\mathcal Z:=\sigma(z_1,\dots,z_N),
\qquad
\mathcal X:=\sigma(\xb_1,\dots,\xb_N),
\]
 and define
\begin{equation}
\Sb_Z:=\EE[\widetilde{\bm{\eta}}^{\res}_N\mid \mathcal Z],
\qquad
\Sb_X:=\EE[\widetilde{\bm{\eta}}^{\res}_N\mid \mathcal X],
\qquad
\Rb^\circ:=\widetilde{\bm{\eta}}^{\res}_N-\Sb_Z-\Sb_X.
\label{eq:ZXR-def}
\end{equation}
We first show that
\begin{equation}
\EE[\widetilde{\bm{\eta}}^{\res}_N]=\boldsymbol{0}.
\label{eq:res-mean-zero}
\end{equation}
Indeed, taking expectation in \eqref{eq:res-rec-rewrite} and using
\[
\EE[z_t-p]=0,
\qquad
\EE[\xb_t\xb_t^\top-\Hb]=\boldsymbol{0},
\qquad
\EE[\xb_t\xb_t^\top]=\Hb,
\]
gives
\[
\EE[\widetilde{\bm{\eta}}^{\res}_t]
=
(\Ib-\gamma_t\Hb)\EE[\widetilde{\bm{\eta}}^{\res}_{t-1}],
\]
and since \(\widetilde{\bm{\eta}}^{\res}_0=\boldsymbol{0}\), \eqref{eq:res-mean-zero} follows.

Next, \(\mathcal Z\) and \(\mathcal X\) are independent, and \(\EE[\Sb_X]=\EE[\widetilde{\bm{\eta}}^{\res}_N]=0\), hence
\[
\EE[\Sb_X\mid \mathcal Z]=\EE[\Sb_X]=0.
\]
Therefore,
\[
\EE\langle \Sb_Z,\Hb\Sb_X\rangle
=
\EE\Big[\langle \Sb_Z,\Hb\,\EE[\Sb_X\mid \mathcal Z]\rangle\Big]
=0.
\]
Moreover, by \eqref{eq:ZXR-def},
\[
\EE[\Rb^\circ\mid \mathcal Z]
=
\EE[\widetilde{\bm{\eta}}^{\res}_N\mid \mathcal Z]-\Sb_Z-\EE[\Sb_X\mid \mathcal Z]
=
0,
\]
and similarly \(\EE[\Rb^\circ\mid \mathcal X]=0\). Thus
\[
\EE\langle \Sb_Z,\Hb\Rb^\circ\rangle=0,
\qquad
\EE\langle \Sb_X,\Hb\Rb^\circ\rangle=0.
\]
Consequently,
\begin{equation}
\EE\|\widetilde{\bm{\eta}}^{\res}_N\|_{\Hb}^2
=
\EE\|\Sb_Z\|_{\Hb}^2
+
\EE\|\Sb_X\|_{\Hb}^2
+
\EE\|\Rb^\circ\|_{\Hb}^2
\ge
\EE\|\Sb_Z\|_{\Hb}^2
+
\EE\|\Sb_X\|_{\Hb}^2.
\label{eq:ZXR-orth}
\end{equation}
Since \(\big\langle \Hb,\EE[\widetilde{\bm{\eta}}^{\res}_N\otimes\widetilde{\bm{\eta}}^{\res}_N]\big\rangle=\EE\|\widetilde{\bm{\eta}}^{\res}_N\|_{\Hb}^2\), it remains to lower bound the two terms on the right-hand side of \eqref{eq:ZXR-orth}.

Define $
\fb_t:=\EE[\widetilde{\bm{\eta}}^{\res}_t\mid \mathcal X]. $Taking conditional expectation in \eqref{eq:res-rec-rewrite} and using \(\EE[z_t-p\mid \mathcal X]=0\), we obtain
\begin{equation}
\fb_t
=
(\Ib-\gamma_t\xb_t\xb_t^\top)\fb_{t-1}
+
\gamma_t(\xb_t\xb_t^\top-\Hb)\cb_{t-1},
\qquad
\fb_0=\boldsymbol{0}.
\label{eq:ft-feature}
\end{equation}
Thus \(\Sb_X=\fb_N\). Further define
\[
\bzeta_u:=(\xb_u\xb_u^\top-\Hb)\cb_{u-1},
\qquad
\Kb_{u,N}:=\gamma_u\prod_{j=u+1}^N(\Ib-\gamma_j\xb_j\xb_j^\top).
\]
Unrolling \eqref{eq:ft-feature} gives
\begin{equation}
\Sb_X=\sum_{u=1}^N \Ub_u,
\qquad
\Ub_u:=\Kb_{u,N}\bzeta_u.
\label{eq:SX-decomp}
\end{equation}

We first prove the cross terms vanish. Let \(u<v\). Since \(\Kb_{u,N}\) and \(\Ub_v\) are measurable with respect to \(\sigma(\xb_{u+1},\dots,\xb_N)\), while \(\bzeta_u\) depends only on \(\xb_u\), we have
\[
\EE[\bzeta_u]
=
(\EE[\xb_u\xb_u^\top]-\Hb)\cb_{u-1}
=
0.
\]
Hence
\begin{align*}
\EE\big[\langle \Ub_u,\Hb\Ub_v\rangle \mid \xb_{u+1},\dots,\xb_N\big]
&=
\EE\big[\langle \bzeta_u,\Kb_{u,N}^\top\Hb\Ub_v\rangle \mid \xb_{u+1},\dots,\xb_N\big]
=0,
\end{align*}
and therefore
\[
\EE\langle \Ub_u,\Hb\Ub_v\rangle=0,
\qquad u\neq v.
\]
It follows that
\begin{equation}
\EE\|\Sb_X\|_{\Hb}^2
=
\sum_{u=1}^N \EE\|\Ub_u\|_{\Hb}^2.
\label{eq:SX-orth}
\end{equation}

Now let
\[
\Ab_u:=\cb_{u-1}\cb_{u-1}^\top\succeq 0.
\]
Using the independence of \(\Kb_{u,N}\) and \(\bzeta_u\), we have
\begin{align}
\EE\|\Ub_u\|_{\Hb}^2
&=
\EE\big[\bzeta_u^\top \Kb_{u,N}^\top\Hb\Kb_{u,N}\bzeta_u\big]
=
\big\langle \EE[\Kb_{u,N}^\top\Hb\Kb_{u,N}],\,\Gb_u\big\rangle,
\label{eq:single-u-feature}
\end{align}
where
\[
\Gb_u:=\EE[\bzeta_u\bzeta_u^\top]
=
\EE[(\xb_u\xb_u^\top-\Hb)\Ab_u(\xb_u\xb_u^\top-\Hb)].
\]
By expansion,
\[
\Gb_u
=
\EE[\xb_u\xb_u^\top \Ab_u\xb_u\xb_u^\top]-\Hb \Ab_u\Hb.
\]
Applying Assumption~\ref{assump:fourth-moment} with \(A=\Ab_u\), we obtain
\[
\Gb_u
\succeq
\beta\,\tr(\Hb \Ab_u)\,\Hb
=
\beta\,\|\cb_{u-1}\|_{\Hb}^2\,\Hb.
\]
Substituting this into \eqref{eq:single-u-feature} yields
\begin{equation}
\EE\|\Ub_u\|_{\Hb}^2
\ge
\beta\,\|\cb_{u-1}\|_{\Hb}^2\,
\EE\!\big[\tr(\Hb\Kb_{u,N}\Hb\Kb_{u,N}^\top)\big].
\label{eq:single-u-feature-lb}
\end{equation}

Set
\[
\Mb_u:=\Hb^{1/2}\Kb_{u,N}\Hb^{1/2}.
\]
Then
\[
\tr(\Hb\Kb_{u,N}\Hb\Kb_{u,N}^\top)=\|\Mb_u\|_F^2.
\]
By Jensen's inequality,
\[
\EE\|\Mb_u\|_F^2
\ge
\|\EE[\Mb_u]\|_F^2.
\]
Moreover,
\[
\EE[\Kb_{u,N}]
=
\gamma_u\prod_{j=u+1}^N(\Ib-\gamma_j\Hb).
\]
Hence, in the eigenbasis of \(\Hb\),
\[
\|\EE[\Mb_u]\|_F^2
=
\sum_{i\ge 1}\gamma_u^2\lambda_i^2\prod_{j=u+1}^N(1-\gamma_j\lambda_i)^2.
\]
Define
\begin{equation}
\kappa_u
:=
\sum_{i\ge 1}\gamma_u^2\lambda_i^2\prod_{j=u+1}^N(1-\gamma_j\lambda_i)^2.
\label{eq:kappa-def}
\end{equation}
Then \eqref{eq:single-u-feature-lb} gives
\[
\EE\|\Ub_u\|_{\Hb}^2
\ge
\beta\,\|\cb_{u-1}\|_{\Hb}^2\,\kappa_u.
\]
Combining this with \eqref{eq:SX-orth},
\begin{equation}
\EE\|\Sb_X\|_{\Hb}^2
\ge
\beta\sum_{u=1}^N \|\cb_{u-1}\|_{\Hb}^2\,\kappa_u.
\label{eq:SX-sum}
\end{equation}

We now lower bound \(\|\cb_{u-1}\|_{\Hb}^2\) uniformly on the tail space. In the eigenbasis of \(\Hb\),
\[
(\cb_{u-1})_i
=
-\prod_{j=1}^{u-1}(1-\gamma_j\lambda_i)\,(\wb_0-\wb_1^*-p\bm{\delta})_i.
\]
Hence
\begin{equation}
\|\cb_{u-1}\|_{\Hb_{k^*:\infty}}^2
=
\sum_{i>k^*}\lambda_i\prod_{j=1}^{u-1}(1-\gamma_j\lambda_i)^2
(\wb_0-\wb_1^*-p\bm{\delta})_i^2.
\label{eq:c-tail}
\end{equation}
Since \(i>k^*\), we have \(\gamma\lambda_i<1/\Neff\le 1/10\) for \(\Neff\ge 10\). Also
\[
\sum_{t=1}^N\gamma_t\le 2\gamma\Neff.
\]
Using \(1-a\ge e^{-2a}\) for \(0\le a\le 1/2\), we obtain
\[
\prod_{j=1}^{u-1}(1-\gamma_j\lambda_i)^2
\ge
\exp\!\left(-4\lambda_i\sum_{j=1}^{u-1}\gamma_j\right)
\ge
\exp(-8\gamma\Neff\lambda_i)
\ge e^{-8}.
\]
Substituting this into \eqref{eq:c-tail} gives
\begin{equation}
\|\cb_{u-1}\|_{\Hb}^2
\ge
\|\cb_{u-1}\|_{\Hb_{k^*:\infty}}^2
\ge
e^{-8}\|\wb_0-\wb_1^*-p\bm{\delta}\|_{\Hb_{k^*:\infty}}^2,
\qquad
u=1,\dots,N.
\label{eq:c-force}
\end{equation}
Therefore \eqref{eq:SX-sum} yields
\begin{equation}
\EE\|\Sb_X\|_{\Hb}^2
\ge
\beta e^{-8}
\|\wb_0-\wb_1^*-p\bm{\delta}\|_{\Hb_{k^*:\infty}}^2
\sum_{u=1}^N \kappa_u.
\label{eq:SX-reduce}
\end{equation}

Finally, by applying Lemma~\ref{lemma:C-lower-bound},
\[
\sum_{u=1}^N \kappa_u
=
\left\langle
\Hb,\;
\sum_{u=1}^N \gamma_u^2
\prod_{j=u+1}^N(\Ib-\gamma_j\Hb)^2\Hb
\right\rangle
\ge
\frac1{400}
\left(
\frac{k^*}{\Neff}
+\gamma^2\Neff\sum_{i>k^*}\lambda_i^2
\right).
\]
Substituting this into \eqref{eq:SX-reduce}, we conclude that
\begin{equation}
\EE\|\Sb_X\|_{\Hb}^2
\ge
\frac{\beta e^{-8}}{400}\,
\|\wb_0-\wb_1^*-p\bm{\delta}\|_{\Hb_{k^*:\infty}}^2
\left(
\frac{k^*}{\Neff}
+\gamma^2\Neff\sum_{i>k^*}\lambda_i^2
\right).
\label{eq:SX-final}
\end{equation}

Then we lower bound schedule component \(\Sb_Z\). For each \(t\), define $\ab_t:=\EE[\widetilde{\bm{\eta}}^{\res}_t\mid \mathcal Z].$ Taking conditional expectation in \eqref{eq:res-rec-rewrite} and using independence of \(\xb_t\) from \((\mathcal Z,\xb_1,\dots,\xb_{t-1})\), we obtain
\[
\ab_t
=
(\Ib-\gamma_t\Hb)\ab_{t-1}
+
\gamma_t(z_t-p)\Hb\bm{\delta},
\qquad
\ab_0=0.
\]
Thus \(\Sb_Z=\ab_N\), and unrolling gives
\begin{equation}
\Sb_Z
=
\sum_{u=1}^N (z_u-p)\vb_u,
\qquad
\vb_u:=\gamma_u\Bigl(\prod_{j=u+1}^N(\Ib-\gamma_j\Hb)\Bigr)\Hb\bm{\delta}.
\label{eq:SZ-decomp}
\end{equation}
In the eigenbasis of \(\Hb\),
\[
(\vb_u)_i
=
\gamma_u\lambda_i\prod_{j=u+1}^N(1-\gamma_j\lambda_i)\delta_i.
\]

By the fixed-budget variance identity,
\begin{equation}
\EE\|\Sb_Z\|_{\Hb}^2
=
\frac{p(1-p)}{2(N-1)}
\sum_{u=1}^N\sum_{r=1}^N
\|\vb_u-\vb_r\|_{\Hb}^2.
\label{eq:SZ-fixed}
\end{equation}

Let
\[
I_-:=\{t:\ell_t=0\},
\qquad
I_+:=\{t:\ell_t\ge 7\}.
\]
For \(N\ge 500\), one has
\[
|I_-|\ge \frac{\Neff}{2},
\qquad
|I_+|\ge \frac{N}{10}.
\]
Restricting \eqref{eq:SZ-fixed} to pairs \((u,r)\in I_-\times I_+\) yields
\begin{equation}
\EE\|\Sb_Z\|_{\Hb}^2
\ge
\frac{p(1-p)}{2(N-1)}
\sum_{u\in I_-}\sum_{r\in I_+}\|\vb_u-\vb_r\|_{\Hb}^2.
\label{eq:SZ-restrict}
\end{equation}

Fix \(i>k^*\). Since \(\gamma\lambda_i<1/\Neff\le 1/10<1/2\), for \(u\in I_-\) we have \(\gamma_u=\gamma\), and
\[
|(\vb_u)_i|
=
\gamma\lambda_i\prod_{j=u+1}^N(1-\gamma_j\lambda_i)|\delta_i|
\ge
e^{-4}\gamma\lambda_i|\delta_i|.
\]
For \(r\in I_+\), we have \(\gamma_r\le 2^{-7}\gamma\), so
\[
|(\vb_r)_i|
\le
2^{-7}\gamma\lambda_i|\delta_i|.
\]
Therefore
\[
|(\vb_u-\vb_r)_i|
\ge
(e^{-4}-2^{-7})\gamma\lambda_i|\delta_i|,
\qquad
u\in I_-,
\quad
r\in I_+,
\quad
i>k^*.
\]
Hence
\begin{equation}
\|\vb_u-\vb_r\|_{\Hb}^2
\ge
(e^{-4}-2^{-7})^2\gamma^2
\sum_{i>k^*}\lambda_i^3\delta_i^2.
\label{eq:pair-lb}
\end{equation}
Substituting \eqref{eq:pair-lb} into \eqref{eq:SZ-restrict}, and using \(|I_-|\ge \Neff/2\), \(|I_+|\ge N/10\), we obtain
\[
\EE\|\Sb_Z\|_{\Hb}^2
\ge
\frac{(e^{-4}-2^{-7})^2}{40}\,
p(1-p)\gamma^2\Neff
\sum_{i>k^*}\lambda_i^3\delta_i^2.
\]
That is,
\begin{equation}
\EE\|\Sb_Z\|_{\Hb}^2
\ge
\frac{(e^{-4}-2^{-7})^2}{40}\,
p(1-p)\gamma^2\Neff
\sum_{i>k^*}\lambda_i^3\delta_i^2.
\label{eq:SZ-final}
\end{equation}

Finally, combining \eqref{eq:ZXR-orth}, \eqref{eq:SX-final}, and \eqref{eq:SZ-final}, we obtain
\[
\big\langle \Hb,\EE[\widetilde{\bm{\eta}}^{\res}_N\otimes\widetilde{\bm{\eta}}^{\res}_N]\big\rangle
\ge
\frac{\beta e^{-8}}{400}\,
\|\wb_0-\wb_1^*-p\bm{\delta}\|_{\Hb_{k^*:\infty}}^2
\cdot
\frac{\DIM}{\Neff}
+
\frac{(e^{-4}-2^{-7})^2}{40}\,
p(1-p)\gamma^2\Neff
\sum_{i>k^*}\lambda_i^3\delta_i^2,
\]
This completes the proof.
\end{proof}

\subsection{Proof of Theorem~\ref{thm:tail-decay-lower-bound}}
\begin{proof}
Combining Lemma~\ref{lem:residual-fluctuation-lower-bound} with Eq.~\ref{eq:variance-part-lower} yields the desired result.  
\end{proof}

\subsection{Proof of Corollary~\ref{cor:strong-model-collapse}}
\begin{proof}
Let
\[
T:=M+N,\qquad 
\Pb_T:=\prod_{t=1}^{T}(\Ib-\gamma_t\Hb).
\]
Since the synthetic proportion $p=M/(M+N)$ is fixed, it suffices to study the limit as $T\to\infty$.

We first show that $\Pb_T$ vanishes in $\Hb$-norm. Writing the eigendecomposition
\[
\Hb=\sum_{i}\lambda_i \vb_i\vb_i^\top,
\]
we have
\[
\Pb_T \vb_i = \Big(\prod_{t=1}^T(1-\gamma_t\lambda_i)\Big)\vb_i.
\]
Under the assumption on the stepsize, for every $i$ with $\lambda_i>0$,
\[
0\le 1-\gamma_t\lambda_i <1.
\]
Moreover, under the geometric tail-decay schedule,
\[
\sum_{t=1}^T \gamma_t \asymp \frac{T}{\log T}\sum_{\ell=0}^{\lfloor \log T \rfloor}2^{-\ell}
\asymp \frac{T}{\log T}\xrightarrow[T\to\infty]{}\infty.
\]
Hence for every $\lambda_i>0$,
\[
\prod_{t=1}^T(1-\gamma_t\lambda_i)\longrightarrow 0.
\]
Therefore, for any $\wb$ with $\|\wb\|_{\Hb}<\infty$,
\[
\|\Pb_T\wb\|_{\Hb}^2
=
\sum_i \lambda_i\Big(\prod_{t=1}^T(1-\gamma_t\lambda_i)\Big)^2 \langle \wb,\vb_i\rangle^2
\longrightarrow 0
\]
by dominated convergence.

We now turn to the upper bound in Theorem~\ref{thm:generalization_error} which gives
\begin{align*}
\EE[\mathcal{E}_1(\wb_T)]
&\lesssim
\underbrace{\|\Pb_T(\wb_0-\wb_1^*)\|_{\Hb}^2}_{(I)}
+\underbrace{\alpha\|\wb_0-\wb_1^*\|^2_{\frac{\Ib_{0:k^*}}{\gamma \widetilde N_{\mathtt{eff}}}+\Hb_{k^*:\infty}}
\frac{\widetilde D_{\mathtt{eff}}}{\widetilde N_{\mathtt{eff}}}}_{(II)}\\
&\quad+
\underbrace{\bar\sigma^2\frac{\widetilde D_{\mathtt{eff}}}{\widetilde N_{\mathtt{eff}}}}_{(III)}
+\underbrace{\alpha p\|\bm{\delta}\|_{\Hb}^2\frac{\widetilde D_{\mathtt{eff}}}{\widetilde N_{\mathtt{eff}}}}_{(IV)}
+\underbrace{p^2\|( \Ib-\Pb_T)\bm{\delta}\|_{\Hb}^2}_{(V)}.
\end{align*}
By the assumption $\widetilde D_{\mathtt{eff}}=o(\widetilde N_{\mathtt{eff}})$, terms $(II)$, $(III)$ and $(IV)$ all vanish as $T\to\infty$. By the argument above, $(I)\to 0$. Finally,
\[
\|( \Ib-\Pb_T)\bm{\delta}\|_{\Hb}
\to \|\bm{\delta}\|_{\Hb},
\]
because $\|\Pb_T\bm{\delta}\|_{\Hb}\to 0$. Therefore,
\[
\limsup_{T\to\infty}\EE[\mathcal{E}_1(\wb_T)]
\lesssim p^2\|\bm{\delta}\|_{\Hb}^2.
\]

Next, apply the lower bound in Theorem~\ref{thm:tail-decay-lower-bound}:
\begin{align*}
\EE[\mathcal{E}_1(\wb_T)]
&\gtrsim
\underbrace{\Big\| \Pb_T (\wb_0-\wb_1^*) +p(\Ib-\Pb_T)\bm{\delta}\Big\|_{\Hb}^2}_{(A)}\\
&\quad+
\underbrace{\big(\beta\norm{\wb_0-\wb_1^*-p\bm{\delta}}_{\Hb_{k^*:\infty}}^2+\bar \sigma^2\big)
\frac{\widetilde D_{\mathtt{eff}}}{\widetilde N_{\mathtt{eff}}}}_{(B)}
+\underbrace{p(1-p)\gamma^2 \widetilde N_{\mathtt{eff}} \norm{\bm{\delta}}_{\Hb^3_{k^*:\infty}}^2}_{(C)}.
\end{align*}
Since $(B)\ge 0$ and $(C)\ge 0$, it is enough to keep only $(A)$. Rewrite
\[
\Pb_T (\wb_0-\wb_1^*) +p(\Ib-\Pb_T)\bm{\delta}
=
\Pb_T(\wb_0-\wb_1^*-p\bm{\delta})+p\bm{\delta}.
\]
Again using $\|\Pb_T \wb\|_{\Hb}\to 0$ for every $\wb$ with finite $\Hb$-norm, we obtain
\[
\Big\| \Pb_T(\wb_0-\wb_1^*-p\bm{\delta})+p\bm{\delta}\Big\|_{\Hb}^2
\longrightarrow
p^2\|\bm{\delta}\|_{\Hb}^2.
\]
Hence,
\[
\liminf_{T\to\infty}\EE[\mathcal{E}_1(\wb_T)]
\gtrsim p^2\|\bm{\delta}\|_{\Hb}^2.
\]

Combining the upper and lower bounds yields
\[
\lim_{M+N\to\infty}\EE[\mathcal{E}_1(\wb_{M+N})]
\eqsim p^2\|\bm{\delta}\|_{\Hb}^2.
\]
This proves the corollary.
\end{proof}

\section{Two-stage Training Analysis}\label{ap:two-stage}
First, consider the upper bound for training on pure one-source data from \citep{wulast}, which can also be recovered by Theorem~\ref{thm:generalization_error}.
\begin{theorem}[Learning with only real data]\label{thm:generalization_error_real_only}
Consider last iterate SGD with stepsize scheme \eqref{eq:geometry-tail-decay-lr}.
Suppose Assumptions \ref{assump:second-moment},  \ref{assump:fourth-moment} and \ref{assump:noise} hold.
Let $\Neff := N / \log N $.
Suppose $\gamma < 1/(4\alpha\tr(\Hb))$. Then we have 
\begin{align*}
\EE [\R(\wb_{N})] - \R(\wb_1^*)
&\lesssim  \big\| (\prod_{t=1}^N\Ib-\gamma_t \Hb)(\wb_0^*-\wb_1^*)\big\|^2_{\Hb} + \big(\alpha\|\wb_0 - \wb_1^* \|^2_{\frac{\Ib_{0:k^*}}{\gamma \Neff}+\Hb_{k^*:\infty}}+\sigma_1^2\big)\frac{\DIM}{\Neff}
\end{align*}

with  $k^* = \max \{k: \lambda_k \ge \frac{1}{\gamma \Neff}\}$ and $\DIM := k^* + \gamma^2 \Neff^2 \sum_{i>k^*}\lambda_i^2.$
\end{theorem}

We now turn to the two-stage training procedure. The key idea is to condition on the first-stage training iterate and then apply the pure real-data bound from Theorem~\ref{thm:generalization_error} to the real-data training phase.

Suppose that in the first stage we run SGD on \(M\) synthetic samples and obtain \(\wb_M\). In the second stage, starting from \(\wb_M\), we run SGD on \(N\) real samples and output \(\wb_{M+N}\). Our goal is to upper bound
\[
\EE\big[\R(\wb_{M+N})-\R(\wb_1^*)\big].
\]

Let $\mathcal{F}_M:=\sigma(\xb_1,y_1,\dots,\xb_M,y_M)$ be the sigma-field generated by the first stage. Conditional on \(\mathcal{F}_M\), the iterate \(\wb_M\) is deterministic, and the second-stage procedure is exactly SGD trained on \(N\) real samples, initialized at \(\wb_M\). Therefore, applying Theorem~\ref{thm:generalization_error} conditionally, we obtain
\begin{align}
\EE\big[\R(\wb_{M+N})-\R(\wb_1^*)\mid \mathcal{F}_M\big]
&\lesssim
\Big\|
\Big(\prod_{t=1}^N(\Ib-\gamma_t\Hb)\Big)(\wb_M-\wb_1^*)
\Big\|_{\Hb}^2
\notag\\
&\quad+
\left(
\alpha\|\wb_M-\wb_1^*\|^2_{\frac{\Ib_{0:k^*}}{\gamma \Neff}+\Hb_{k^*:\infty}}
+\sigma_1^2
\right)\frac{\DIM}{\Neff},
\label{eq:two-stage-cond-upper}
\end{align}

Taking expectation with respect to \(\mathcal{F}_M\) on both sides of \eqref{eq:two-stage-cond-upper}, we get
\begin{align}
\EE\big[\R(\wb_{M+N})-\R(\wb_1^*)\big]
&=
\EE\Big[\EE\big[\R(\wb_{M+N})-\R(\wb_1^*)\mid \mathcal{F}_M\big]\Big]
\notag\\
&\lesssim
\EE\Bigg[
\Big\|
\Big(\prod_{t=1}^N(\Ib-\gamma_t\Hb)\Big)(\wb_M-\wb_1^*)
\Big\|_{\Hb}^2
\Bigg]
\notag\\
&\quad+
\left(
\alpha\EE\Big[\|\wb_M-\wb_1^*\|^2_{\frac{\Ib_{0:k^*}}{\gamma \Neff}+\Hb_{k^*:\infty}}\Big]
+\sigma_1^2
\right)\frac{\DIM}{\Neff}.
\label{eq:two-stage-upper-before-expand}
\end{align}

We now relate the terms to the stage-one excess risk under the synthetic distribution. Recall that
\[
\wb_M-\wb_1^*
=
(\wb_M-\wb_2^*)+\bm{\delta}.
\]
Hence
\begin{align}
&\EE\Bigg[
\Big\|
\Big(\prod_{t=1}^N(\Ib-\gamma_t\Hb)\Big)(\wb_M-\wb_1^*)
\Big\|_{\Hb}^2
\Bigg]
\notag\\
&=
\EE\Bigg[
\Big\|
\Big(\prod_{t=1}^N(\Ib-\gamma_t\Hb)\Big)(\wb_M-\wb_2^*)
+
\Big(\prod_{t=1}^N(\Ib-\gamma_t\Hb)\Big)\bm{\delta}
\Big\|_{\Hb}^2
\Bigg]
\notag\\
&\le
2\EE\Bigg[
\Big\|
\Big(\prod_{t=1}^N(\Ib-\gamma_t\Hb)\Big)(\wb_M-\wb_2^*)
\Big\|_{\Hb}^2
\Bigg]
+
2\Big\|
\Big(\prod_{t=1}^N(\Ib-\gamma_t\Hb)\Big)\bm{\delta}
\Big\|_{\Hb}^2 .
\label{eq:two-stage-bias-split-correct}
\end{align}

Now denote
\[
\Pb_N:=\prod_{t=1}^N(\Ib-\gamma_t\Hb).
\]
Since \(\Pb_N\) is a polynomial in \(\Hb\), it commutes with \(\Hb\). Therefore
\[
\|\Pb_N \vb\|_{\Hb}^2
=
\langle \vb,\Pb_N \Hb \Pb_N \vb\rangle
=
\langle \vb,\Hb^{1/2}\Pb_N^2\Hb^{1/2}\vb\rangle
\le
\|\vb\|_{\Hb}^2,
\]
because \(0\preceq \Pb_N \preceq \Ib\) under the stepsize assumption \(\gamma_t<1/\lambda_1\). Hence
\begin{align}
\EE\Big[
\|\Pb_N(\wb_M-\wb_2^*)\|_{\Hb}^2
\Big]
&\le
\EE\Big[
\|\wb_M-\wb_2^*\|_{\Hb}^2
\Big]
\notag\\
&=
2\,\EE[L_2(\wb_M)-L_2(\wb_2^*)].
\label{eq:stage-one-excess-first}
\end{align}

Next, let
\[
\Bb:=\frac{\Ib_{0:k^*}}{\gamma \Neff}+\Hb_{k^*:\infty}.
\]
By the definition of \(k^*\),
\[
\lambda_i\ge \frac{1}{\gamma \Neff},
\qquad i\le k^*.
\]
Therefore
\[
\frac{\Ib_{0:k^*}}{\gamma \Neff}\preceq \Hb_{0:k^*},
\]
and hence
\begin{equation}
\Bb
=
\frac{\Ib_{0:k^*}}{\gamma \Neff}+\Hb_{k^*:\infty}
\preceq
\Hb_{0:k^*}+\Hb_{k^*:\infty}
=
\Hb.
\label{eq:B-less-H}
\end{equation}
It follows that
\begin{align}
\EE\Big[\|\wb_M-\wb_2^*\|_{\Bb}^2\Big]
&\le
\EE\Big[\|\wb_M-\wb_2^*\|_{\Hb}^2\Big]
\notag\\
&=
2\,\EE[L_2(\wb_M)-L_2(\wb_2^*)].
\label{eq:stage-one-excess-second}
\end{align}

On the other hand,
\begin{align}
\EE\Big[\|\wb_M-\wb_1^*\|_{\Bb}^2\Big]
&=
\EE\Big[\|(\wb_M-\wb_2^*)+\bm{\delta}\|_{\Bb}^2\Big]
\notag\\
&\le
2\EE\Big[\|\wb_M-\wb_2^*\|_{\Bb}^2\Big]
+
2\|\bm{\delta}\|_{\Bb}^2
\notag\\
&\le
4\,\EE[L_2(\wb_M)-L_2(\wb_2^*)]
+
2\|\bm{\delta}\|_{\Bb}^2,
\label{eq:two-stage-B-split-correct}
\end{align}
where in the last step we used \eqref{eq:stage-one-excess-second}.

Substituting \eqref{eq:two-stage-bias-split-correct}, \eqref{eq:stage-one-excess-first}, and \eqref{eq:two-stage-B-split-correct} into \eqref{eq:two-stage-upper-before-expand}, we obtain
\begin{align}
\EE\big[\R(\wb_{M+N})-\R(\wb_1^*)\big]
&\lesssim
\EE[\R_2(\wb_M)-\R_2(\wb_2^*)]
+
\big\|\Pb_N\bm{\delta}\big\|_{\Hb}^2
\notag\\
&\quad+
\left(
\alpha\,\EE[\R_2(\wb_M)-\R_2(\wb_2^*)]
+
\alpha\|\bm{\delta}\|_{\Bb}^2
+
\sigma_1^2
\right)\frac{\DIM}{\Neff}.
\label{eq:two-stage-upper-stage1-excess}
\end{align}

Finally, by \eqref{eq:B-less-H},
\[
\|\bm{\delta}\|_{\Bb}^2\le \|\bm{\delta}\|_{\Hb}^2.
\]
Therefore
\begin{align}
\EE\big[\R(\wb_{M+N})-\R(\wb_1^*)\big]
&\lesssim
\EE[\R_2(\wb_M)-\R_2(\wb_2^*)]
+
\big\|\Pb_N\bm{\delta}\big\|_{\Hb}^2
\notag\\
&\quad+
\Big(
\alpha\,\EE[\R_2(\wb_M)-\R_2(\wb_2^*)]
+
\alpha\|\bm{\delta}\|_{\Hb}^2
+
\sigma_1^2
\Big)\frac{\DIM}{\Neff}.
\label{eq:two-stage-upper-final-correct}
\end{align}

Equivalently, writing
\[
\mathtt{Excess}_2(\wb_M)
:=
\EE[\R_2(\wb_M)-\R_2(\wb_2^*)],
\]
we have
\begin{align}
\mathtt{Excess}(\wb_{M+N})
&\lesssim
\mathtt{Excess}_2(\wb_M)
+
\big\|\Pb_N\bm{\delta}\big\|_{\Hb}^2
+
\Big(
\alpha\,\mathtt{Excess}_2(\wb_M)
+
\alpha\|\bm{\delta}\|_{\Hb}^2
+
\sigma_1^2
\Big)\frac{\DIM}{\Neff}.
\label{eq:two-stage-upper-final-clean}
\end{align}
This is the desired two-stage upper bound.

\subsection{Two-stage training without learning-rate restart}
\label{ap:two-stage-no-restart}

The two-stage analysis restarts the geometrically decaying stepsize schedule
at the beginning of the real-data stage. We briefly discuss how the argument changes when a single uninterrupted schedule is used across both stages.

Let $T:=M+N$, and denote the global $T$-step schedule by
\[
    \gamma_t^{(T)}
    :=
    \frac{\gamma}{
        2^{\left\lfloor t/(T/\log T)\right\rfloor}
    },
    \qquad
    t=1,\ldots,T.
\]
Under the restarted protocol analyzed above, the real-data stage instead uses
a fresh $N$-step schedule
\[
    \gamma_s^{(N)}
    :=
    \frac{\gamma}{
        2^{\left\lfloor s/(N/\log N)\right\rfloor}
    },
    \qquad
    s=1,\ldots,N.
\]
Thus, without restarting, the contraction operator appearing in the
two-stage bound changes from
$\Pb_N^{\mathrm{rst}}:=\prod_{s=1}^{N}(\Ib-\gamma_s^{(N)}\Hb)$ to the real-data tail of the global schedule
$\Pb_{M,N}^{\mathrm{tail}}:=\prod_{t=M+1}^{M+N}(\Ib-\gamma_t^{(T)}\Hb).$

The proof follows the same conditional argument as above. Conditional on the
first-stage iterate $\wb_M$, the second stage remains single-source SGD on
real data. The decomposition $\wb_M-\wb_1^*=(\wb_M-\wb_2^*)+\bm{\delta}$
is unchanged, as are the subsequent comparisons between the stage-one excess
risk and the corresponding $\Hb$-weighted errors. The only modification is
that the proof of the single-source bound must be applied to the shifted
stepsize sequence $\{\gamma_{M+1}^{(T)},\ldots,\gamma_{M+N}^{(T)}\}$ rather than to a freshly restarted sequence
$\{\gamma_1^{(N)},\ldots,\gamma_N^{(N)}\}$. Consequently, every occurrence of
$\Pb_N^{\mathrm{rst}}$ is replaced by $\Pb_{M,N}^{\mathrm{tail}}$, and the
spectral cutoff and effective-dimension factor are defined using the
real-stage tail schedule.

In particular, the resulting bound has the same qualitative decomposition:
\[
\begin{aligned}
\EE[\mathcal{E}_1(\wb_{M+N})]
\lesssim{}&
\EE[\mathcal{E}_2(\wb_M)]
+
\big\|
    \Pb_{M,N}^{\mathrm{tail}}\bm{\delta}
\big\|_{\Hb}^{2} \\
&+
\Big(
    \alpha\EE[\mathcal{E}_2(\wb_M)]
    +
    \alpha\|\bm{\delta}\|_{\Hb}^{2}
    +
    \sigma_1^2
\Big)
\frac{D_{\mathrm{eff}}^{\mathrm{tail}}}
     {T_{\mathrm{eff}}},
\end{aligned}
\]
where $T_{\mathrm{eff}}:=T/\log T$,
$k_{\mathrm{tail}}^*:=\max\{k:\lambda_k\ge
1/(\gamma_{M+1}^{(T)}T_{\mathrm{eff}})\}$, and
$D_{\mathrm{eff}}^{\mathrm{tail}}
:=k_{\mathrm{tail}}^*
+(\gamma_{M+1}^{(T)}T_{\mathrm{eff}})^2
\sum_{i>k_{\mathrm{tail}}^*}\lambda_i^2$.

Nevertheless, the qualitative asymptotic conclusion is unchanged. For every
eigendirection with $\lambda_i>0$,
\[
\begin{aligned}
    \prod_{t=M+1}^{M+N}
    \bigl(1-\gamma_t^{(T)}\lambda_i\bigr)
    &\le
    \exp\left(
        -\lambda_i
        \sum_{t=M+1}^{M+N}\gamma_t^{(T)}
    \right).
\end{aligned}
\]
When $p=M/(M+N)\in(0,1)$ is fixed, the real-data tail of the geometric schedule
satisfies $\sum_{t=M+1}^{M+N}\gamma_t^{(T)}\rightarrow\infty.$ It follows that
$\big\|
        \Pb_{M,N}^{\mathrm{tail}}\bm{\delta}
    \big\|_{\Hb}^{2}
    \rightarrow 0.$
Moreover, under the corresponding effective-dimension consistency condition,
$D_{\mathrm{eff}}^{\mathrm{tail}}/T\rightarrow0.$
Therefore, using an uninterrupted schedule slows the finite-sample forgetting
of the synthetic initialization, but it does not introduce a persistent
synthetic drift during the real-data stage and does not create a
non-vanishing excess-risk floor.

\section{Random Sketch Model}\label{ap:random-sketch}

Following \citep{linscaling}, we begin by decomposing the population risk
into three components: irreducible risk, approximation error, and excess risk:
\begin{equation}\label{eq:approx-excess-decomp}
    \R_D(\vb_{M+N}) = \underbrace{\min \R(\cdot)}_{\mathsf{Irreducible}} + \underbrace{\min \R_D(\cdot) - \min \R(\cdot)}_{\mathsf{Approx}} + \underbrace{\R_D(\vb_{M+N}) - \min \R_D(\cdot)}_{\mathsf{Excess}}.
\end{equation}

For the Irreducible risk, under the well-specified model Assumption~\ref{ass:distribution-condition},
\[
\mathsf{Irreducible}=\R(\wb_1^*)=\frac12 \sigma_1^2.
\]
For Approximation Error, as established in Lemma C.4 in \citep{linscaling}, under Assumption~\ref{ass:distribution-condition}, with probability at least $1-e^{\Omega(D)}$,
\[
\EE_{\wb_1^*}\mathsf{Approx} \eqsim D^{1-a}.
\]
So in the following subsections, we focus on derive the bound for $\EE\mathsf{Excess}$ using risk upper bound under Assumption~\ref{ass:distribution-condition} and~\ref{ass:power-law-source}. Unless other illustration, expectations are conditioned on $\Sb$.

\subsection{Mixed Training Upper Bound}
Recall the upper bound in Theorem~\ref{thm:generalization_error}, where the data covariance becomes $\Sb\Hb\Sb^\top$ and the optimal parameters becomes $\vb_1^*,\vb_2^*$.

Suppose Assumptions \ref{assump:fourth-moment}, \ref{ass:distribution-condition} hold. Let $\widetilde{N}_{\mathtt{eff}} := (M+N) / \log(M+N)$, $\bar \sigma^2=(1-p)\sigma_1^2+p\sigma_2^2$ and $\tilde\lambda_j$ be the eigenvalue of $\Sb\Hb\Sb^\top$.
Suppose $\gamma < 1/(4\alpha\tr(\Sb\Hb\Sb^\top))$. Then we have 
\begin{align*}
\mathsf{Excess}
&\lesssim  \mathrm{Bias}+\mathrm{Var} +\underbrace{\alpha\, p \|\vb_2^*-\vb_1^*\|^2_{\Sb\Hb\Sb^\top} \frac{\widetilde{D}_{\mathtt{eff}}}{\widetilde{N}_{\mathtt{eff}}}}_{\mathsf{Fluct}}
+ \underbrace{p^2\big\| \vb_2^*-\vb_1^*\big\|^2_{\Sb\Hb\Sb^\top}}_{\mathsf{Drift}},
\end{align*}
with
\begin{align*}
    \vb^*_\ell:=(\Sb\Hb\Sb^\top)^{-1}\Sb\Hb\wb^*_\ell,\quad \text{and}\quad\DIM:= \#\{\tilde\lambda_j \ge 1/(\widetilde{N}_{\mathtt{eff}} \gamma)\} + (\widetilde{N}_{\mathtt{eff}} \gamma)^2 \sum_{\tilde \lambda_j < 1/(\widetilde{N}_{\mathtt{eff}} \gamma)}\tilde\lambda_j^2.
\end{align*}

Note that when $\sigma_1^2,\sigma_2^2\eqsim1$, Theorem A.4 in \cite{linscaling} shows that 
\[
\mathrm{Bias}+\mathrm{Var} \lesssim \big\| \prod_{t=1}^{M+N}(\Ib-\gamma_t \Sb\Hb\Sb^\top) (\vb_0 - \vb_1^*) \big\|^2_{\Sb\Hb\Sb^\top}+ \frac{\DIM}{\widetilde{N}_{\mathtt{eff}}}
\]

Future assume that Assumption~\ref{ass:power-law-source} holds and zero initialiazation where $\vb_0=\bm{0}$, \cite{linscaling}(Appendix~D and E) shows that with probability at least $1-e^{-\Omega(D)}$ over the randomness of the sketch matrix $\Sb$
\begin{align*}
   &\EE_{\wb_1^*}\big\| \prod_{t=1}^{M+N}(\Ib-\gamma_t \Sb\Hb\Sb^\top) (\vb_1^*) \big\|^2_{\Sb\Hb\Sb^\top}\lesssim \max\big\{ D^{1-a},\ (\widetilde{N}_{\mathtt{eff}} \gamma)^{1/a-1}\big\}, \\
   &\EE_{\wb_1^*}\big\| \prod_{t=1}^{M+N}(\Ib-\gamma_t \Sb\Hb\Sb^\top) (\vb_1^*) \big\|^2_{\Sb\Hb\Sb^\top}\gtrsim (\widetilde{N}_{\mathtt{eff}} \gamma)^{1/a-1}\text { when } (\widetilde{N}_{\mathtt{eff}} \gamma)^{1/a}\leq  D/c \text{\quad for some constant }c>0,
\end{align*}
and 
\begin{align*}
   \frac{\DIM}{\widetilde{N}_{\mathtt{eff}}}\eqsim {\min\big\{ D, \ (\widetilde{N}_{\mathtt{eff}} \gamma)^{1/a}\big\}}/{\widetilde{N}_{\mathtt{eff}}} (\widetilde{N}_{\mathtt{eff}} \gamma)^{1/a-1}.
\end{align*}

For the Fluct term, we have
\[
\mathrm{Fluct}\lesssim \|\vb_2^*-\vb_1^*\|_{\Sb\Hb\Sb^\top}^2\cdot \frac{\DIM}{\widetilde{N}_{\mathtt{eff}}}
\]
We verify that 
\begin{align*}
    \|\vb_2^*-\vb_1^*\|_{\Sb\Hb\Sb^\top}^2 &=  \|\Hb^{\frac{1}{2}} \Sb^\top (\vb_2^*-\vb_1^*)\|^2  \\ 
    &= \| \Hb^{\frac{1}{2}} \Sb^\top (\Sb\Hb\Sb^\top)^{-1} \Sb \Hb \bm{\delta}\|^2 \\ 
    &\leq \|\Hb^{\frac{1}{2}} \bm{\delta}\|^2= \|\bm{\delta}\|^2_\Hb,
\end{align*}
which implies that 
\begin{equation}\label{eq:drift_first}
    \EE_{\bm{\delta}}\|\vb_2^*-\vb_1^*\|_{\Sb\Hb\Sb^\top}^2 \lesssim \EE\|\bm{\delta}\|^2_\Hb \eqsim \sum_i \,i^{-b} \eqsim1
\end{equation}
So that with probability at least $1-e^{-\Omega(D)}$ over the randomness of the sketch matrix $\Sb$
\begin{align*}
    \EE_{\bm{\delta}}\mathrm{Fluct} \lesssim \frac{\DIM}{\widetilde{N}_{\mathtt{eff}}} \eqsim {\min\big\{ D, \ (\widetilde{N}_{\mathtt{eff}} \gamma)^{1/a}\big\}}/{\widetilde{N}_{\mathtt{eff}}} \lesssim D^{1-a}+(\widetilde{N}_{\mathtt{eff}} \gamma)^{1/a-1}.
\end{align*}
For the Drift term $p^2\|\vb_2^*-\vb_1^*\|_{\Sb\Hb\Sb^\top}^2$,
\begin{align*}
    \|\vb_2^*-\vb_1^*\|_{\Sb\Hb\Sb^\top}^2 & = \|(\Sb\Hb\Sb^\top)^{-1}\Sb\Hb(\wb_2^*-\wb_1^*)\|_{\Sb\Hb\Sb^\top}^2\\
    & = \bm{\delta}^\top\Hb\Sb^\top(\Sb\Hb\Sb^\top)^{-1}(\Sb\Hb\Sb^\top)(\Sb\Hb\Sb^\top)^{-1}\Sb\Hb\bm{\delta}\\
    & = \bm{\delta}^\top\Hb\Sb^\top(\Sb\Hb\Sb^\top)^{-1}\Sb\Hb\bm{\delta} + \bm{\delta}^\top\Hb\bm{\delta}-\bm{\delta}\top\Hb\bm{\delta}\\
    & = \bm{\delta}^\top\Hb\bm{\delta}-\bm{\delta}^\top(\Hb-\Hb\Sb^\top(\Sb\Hb\Sb^\top)^{-1}\Sb\Hb)\bm{\delta}\\
    & = \|\bm{\delta}\|^2_{\Hb}-\big\|\big(\Ib-\Hb^{\frac12}\Sb^\top(\Sb\Hb\Sb^\top)^{-1}\Sb^\top\Hb^{\frac12}\big)\Hb^\frac12\bm{\delta}\big\|^2
\end{align*}
For the second minus term, \cite{linscaling}(Lemma~C.5) shows that for any random vector $\xb$, let $(\lambda_i,\vb_i)_{i\ge1}$ be the eigenvalue-eigenvector pairs of $\Hb$, if the Assumption~\ref{ass:power-law-source} (power-law decay) satisfies and $\EE[\langle \vb_i,\xb\rangle\langle \vb_j,\xb\rangle]=0$ for $i\neq j$ and $\EE[\lambda_i\langle \vb_i,\xb\rangle^2]\eqsim i^{-b}$ for some $b>1$, Then with probability at least $1-e^{-\Omega(D)}$ over the randomness of the sketch matrix $\Sb$
\begin{equation*}
    \EE_{\xb}\mathsf{Approx}(\Sb,\Hb,\xb):=\big\|\big(\Ib-\Hb^{\frac12}\Sb^\top(\Sb\Hb\Sb^\top)^{-1}\Sb^\top\Hb^{\frac12}\big)\Hb^\frac12\xb\big\|^2\eqsim D^{1-b}
\end{equation*}
So that under Assumption~\ref{ass:power-law-source}, 
\begin{equation}\label{eq:drift_approx}
    \EE_{\bm{\delta}}\big\|\big(\Ib-\Hb^{\frac12}\Sb^\top(\Sb\Hb\Sb^\top)^{-1}\Sb^\top\Hb^{\frac12}\big)\Hb^\frac12\bm{\delta}\big\|^2\eqsim D^{1-b}
\end{equation}
Taking expectation in the decomposition and using
\eqref{eq:drift_approx}, we obtain
\begin{align}
    \EE_{\bm\delta}
    \|\vb_2^*-\vb_1^*\|_{\Sb\Hb\Sb^\top}^2
    &=
    \EE_{\bm\delta}\|\bm\delta\|_{\Hb}^2
    -
    \Theta(D^{1-b}).
    \label{eq:drift-captured-scaling}
\end{align}

Since
$\EE_{\bm\delta}\|\bm\delta\|_{\Hb}^2\eqsim1$, we suppress this
problem-dependent constant in the scaling-law statement and write, for
sufficiently large $D$,
\begin{equation}
    \EE_{\bm\delta}
    \|\vb_2^*-\vb_1^*\|_{\Sb\Hb\Sb^\top}^2
    \eqsim
    1-D^{1-b}.
    \label{eq:drift-scaling-short}
\end{equation}
Therefore,
\begin{align*}
    \EE_{\bm\delta}\mathsf{Drift}
    &=
    p^2
    \EE_{\bm\delta}
    \|\vb_2^*-\vb_1^*\|_{\Sb\Hb\Sb^\top}^2\\
    &\eqsim
    p^2(1-D^{1-b}).
\end{align*}

Finally, put $\mathsf{Irreducible},\mathsf{Approx}$ and $\mathsf{Excess}$ together and choose $\gamma \eqsim 1$, we have that with probability at least $1-e^{-\Omega(D)}$ over the randomness of the sketch matrix $\Sb$:
\begin{align*}
    \EE\R_D(\vb_{M+N}) &\lesssim \sigma_1^2+\frac{1}{D^{a-1}}+ \frac{1}{(\widetilde{N}_{\mathtt{eff}})^{1-1/a}} +\big(\frac{M}{M+N}\big)^2(1-\frac{1}{D^{b-1}})\\
\end{align*}

\subsection{Mixed Training Lower Bound}

We now establish the matching lower bound in the optimization-saturated
regime $\gamma\widetilde N_{\mathtt{eff}}\gtrsim D^a.$ 

Applying Theorem~\ref{thm:tail-decay-lower-bound} to the sketched problem,
with $\Hb$ replaced by $\Sb\Hb\Sb^\top$ and using $\vb_0=\mathbf 0$, gives
\begin{align*}
    \EE\mathsf{Excess}
    \gtrsim
    \EE_{\wb_1^*,\bm\delta}
    \Bigg\|-\Bigg(\prod_{t=1}^{M+N}(\Ib-\gamma_t\Sb\Hb\Sb^\top)\Bigg)\vb_1^*+p\Bigg(\Ib-\prod_{t=1}^{M+N}(\Ib-\gamma_t\Sb\Hb\Sb^\top)\Bigg)(\vb_2^*-\vb_1^*)\Bigg\|_{\Sb\Hb\Sb^\top}^{2},
\end{align*}
where we have dropped the remaining nonnegative terms in
Theorem~\ref{thm:tail-decay-lower-bound}.

By Assumption~\ref{ass:distribution-condition},
\begin{align*}
    \EE_{\wb_1^*,\bm\delta}
    \big[
        \vb_1^*(\vb_2^*-\vb_1^*)^\top
    \big]
    &=
    (\Sb\Hb\Sb^\top)^{-1}\Sb\Hb
    \EE[\wb_1^*\bm\delta^\top]
    \Hb\Sb^\top(\Sb\Hb\Sb^\top)^{-1}\\
    &=\mathbf 0.
\end{align*}
Therefore, the cross term vanishes after taking expectation over
$(\wb_1^*,\bm\delta)$, and hence
\begin{align}
    \EE\mathsf{Excess}
    &\gtrsim
    p^2\EE_{\bm\delta}
    \Bigg\|
        \Bigg(
            \Ib-
            \prod_{t=1}^{M+N}
            (\Ib-\gamma_t\Sb\Hb\Sb^\top)
        \Bigg)
        (\vb_2^*-\vb_1^*)
    \Bigg\|_{\Sb\Hb\Sb^\top}^{2}.
    \label{eq:mix-sketch-drift-lower}
\end{align}

It remains to control the spectral filter in
\eqref{eq:mix-sketch-drift-lower}. By Lemma~6.2 of
\citet{linscaling}, with probability at least $1-e^{-\Omega(D)}$,
the eigenvalues of $\Sb\Hb\Sb^\top$ satisfy
\[
    \widetilde\lambda_j\eqsim j^{-a},
    \qquad j=1,\ldots,D.
\]
Moreover, the geometric stepsize schedule satisfies $\sum_{t=1}^{M+N}\gamma_t\eqsim\gamma\widetilde N_{\mathtt{eff}}.$
Thus, for every $j=1,\ldots,D$,
\begin{align*}
    \prod_{t=1}^{M+N}
    (1-\gamma_t\widetilde\lambda_j)
    &\leq
    \exp\left(
        -\widetilde\lambda_j
        \sum_{t=1}^{M+N}\gamma_t
    \right)\\
    &\leq
    \exp\left(
        -c\gamma
        \widetilde N_{\mathtt{eff}}D^{-a}
    \right)
    \leq e^{-c},
\end{align*}
where the last inequality follows from
$\gamma\widetilde N_{\mathtt{eff}}\gtrsim D^a$. Let $(\widetilde\lambda_j,\widetilde\vb_j)_{j=1}^D$ be the
eigenvalue-eigenvector pairs of $\Sb\Hb\Sb^\top$. Expanding the norm in
this eigenbasis gives
\begin{align*}
    &\Bigg\|
        \Bigg(
            \Ib-
            \prod_{t=1}^{M+N}
            (\Ib-\gamma_t\Sb\Hb\Sb^\top)
        \Bigg)
        (\vb_2^*-\vb_1^*)
    \Bigg\|_{\Sb\Hb\Sb^\top}^{2}\\
    &=
    \sum_{j=1}^{D}
    \widetilde\lambda_j
    \Bigg(
        1-
        \prod_{t=1}^{M+N}
        (1-\gamma_t\widetilde\lambda_j)
    \Bigg)^2
    \left\langle
        \widetilde\vb_j,
        \vb_2^*-\vb_1^*
    \right\rangle^2\\
    &\geq
    (1-e^{-c})^2
    \sum_{j=1}^{D}
    \widetilde\lambda_j
    \left\langle
        \widetilde\vb_j,
        \vb_2^*-\vb_1^*
    \right\rangle^2\\
    &=
    (1-e^{-c})^2
    \|\vb_2^*-\vb_1^*\|_{\Sb\Hb\Sb^\top}^{2}.
\end{align*}
Therefore,
\begin{align*}
    \EE_{\bm\delta}
    \Bigg\|
        \Bigg(
            \Ib-
            \prod_{t=1}^{M+N}
            (\Ib-\gamma_t\Sb\Hb\Sb^\top)
        \Bigg)
        (\vb_2^*-\vb_1^*)
    \Bigg\|_{\Sb\Hb\Sb^\top}^{2}
    \gtrsim
    \EE_{\bm\delta}
    \|\vb_2^*-\vb_1^*\|_{\Sb\Hb\Sb^\top}^{2}.
\end{align*}

The upper-bound analysis above has already established that
\[
    \EE_{\bm\delta}
    \|\vb_2^*-\vb_1^*\|_{\Sb\Hb\Sb^\top}^{2}
    \eqsim
    1-D^{1-b}.
\]
Substituting this result into
\eqref{eq:mix-sketch-drift-lower} gives
$\EE\mathsf{Excess}\gtrsim p^2(1-D^{1-b}).$
In addition, the approximation-error result stated above gives
$\EE_{\wb_1^*}\mathsf{Approx}\eqsim D^{1-a}.$
Combining this with
\eqref{eq:approx-excess-decomp}, we obtain
\[
    \EE\R_D(\vb_{M+N})-\R(\wb_1^*)
    \gtrsim
    D^{1-a}
    +
    p^2(1-D^{1-b}).
\]

Finally, since $\gamma\eqsim1$ and
$\gamma\widetilde N_{\mathtt{eff}}\gtrsim D^a$, we have $\widetilde N_{\mathtt{eff}}^{1/a-1}\lesssim D^{1-a}.$
Therefore,
\[
    \EE\R_D(\vb_{M+N})-\R(\wb_1^*)
    \gtrsim
    D^{1-a}
    +
    \widetilde N_{\mathtt{eff}}^{1/a-1}
    +
    p^2(1-D^{1-b}),
\]
which matches the upper bound in
Theorem~\ref{thm:mix-scaling} up to constants.

\subsection{Two-stage Training Analysis}
First, Recall the upper bound in Theorem~\ref{thm:two-stage}. For a tighter analysis, we need a intermediate result from its proof. The results are for specific $\wb_1^*,\wb_2^*$ and the expectation is taken with respect to the one-stage M-step SGD.

Suppose Assumptions \ref{assump:second-moment}, \ref{assump:fourth-moment} and \ref{ass:distribution-condition} hold.
Let $\Neff :=  N / \log N $ and $\Meff=M / \log M$.
Suppose $\gamma < 1/(4\alpha\tr(\Sb\Hb\Sb^\top))$. Then we have 
\begin{align*}
\mathtt{Excess}(\vb_{M+N})
&\lesssim \EE\Bigg[
\Big\|
\Big(\prod_{t=1}^N(\Ib-\gamma_t\Sb\Hb\Sb^\top)\Big)(\vb_M-\vb_2^*)
\Big\|_{\Sb\Hb\Sb^\top}^2
\Bigg]+ \big\| (\prod_{t=1}^N\Ib-\gamma_t \Sb\Hb\Sb^\top)(\vb_2^*-\vb_1^*)\big\|^2_{\Sb\Hb\Sb^\top}\\
&+ \alpha\big(\mathsf{Excess}_2(\vb_M)+\|\vb_2^*-\vb_1^*\|^2_{\Sb\Hb\Sb^\top}\big)\frac{\DIM}{\Neff} + \mathsf{Var}
\end{align*}
where 
\begin{align*}
    \vb^*_\ell:=(\Sb\Hb\Sb^\top)^{-1}\Sb\Hb\wb^*_\ell,\quad \text{and}\quad\DIM:= \#\{\tilde\lambda_j \ge 1/(\Neff \gamma)\} + (\Neff \gamma)^2 \sum_{\tilde \lambda_j < 1/(\Neff \gamma)}\tilde\lambda_j^2.
\end{align*}
Suppose $\sigma_1^2,\sigma_2^2 \eqsim1$.
For the $\mathsf{Var}$ term, similar to the analysis in mixed training analysis and \citep{linscaling}(Lemma~E.1), with probability at least $1-e^{-\Omega(D)}$ over the randomness of the sketch matrix $\Sb$
\[
\mathsf{Var}\eqsim \frac{\DIM}{\Neff}\eqsim {\min\big\{ D, \ (\Neff \gamma)^{1/a}\big\}}/{\Neff}. 
\]

\paragraph{Analysis of $\mathsf{Excess}_2(\vb_M)$.}
We now analyze the first-stage excess risk when the teacher is $\wb_2^*$.
Conditioning on the sketch matrix $\Sb$, define
\[
\vb_2^*:=(\Sb\Hb\Sb^\top)^{-1}\Sb\Hb\wb_2^*,
\qquad
\Meff:=M/\log M.
\]
Using Theorem~6.1 in~\cite{linscaling}, for $\gamma<1/(4\alpha\,\tr(\Sb\Hb\Sb^\top))$ and zero initialization,
\[
\EE \mathsf{Excess}_2(\vb_M)\lesssim \mathsf{Bias}_2+\frac{\DIM_1}{\Meff},
\]
where
\[
\DIM_1:=\#\{\tilde\lambda_j\ge 1/(\Meff\gamma)\}+(\Meff\gamma)^2\sum_{\tilde\lambda_j<1/(\Meff\gamma)}\tilde\lambda_j^2,
\]
and the expectation is over the randomness of $\wb_2^*$ and the first-stage SGD.

By Assumption~\ref{ass:power-law-source}, let $(\lambda_i,\vb_i)_{i\ge1}$ be the eigenvalue-eigenvector pairs of $\Hb$, and
$\bm\Sigma$ shares the eigenvectors $(\vb_i)_{i\ge1}$ with $\Hb$, i.e.,
\[
\bm\Sigma \vb_i=\tau_i \vb_i,\qquad \lambda_i\eqsim i^{-a},\qquad \lambda_i\tau_i\eqsim i^{-b},
\]
for some $a>1$ and $b>1$. Then
\[
\EE\langle \vb_i,\wb_2^*\rangle^2
=
\EE\langle \vb_i,\wb_1^*\rangle^2+\EE\langle \vb_i,\bm\delta\rangle^2
=
1+\tau_i,
\]
where we used $\EE[\wb_1^*\bm\delta^\top]=0$.

For the bias term, applying Lemma~D.1 in~\cite{linscaling} with a cutoff $k\le D/3$ yields
\[
\EE \mathsf{Bias}_2
\lesssim
\frac{\sum_{i\le k}\EE\langle \vb_i,\wb_2^*\rangle^2}{\Meff\gamma}
+
\sum_{i>k}\lambda_i\EE\langle \vb_i,\wb_2^*\rangle^2.
\]
Substituting $\EE\langle \vb_i,\wb_2^*\rangle^2=1+\tau_i$ gives
\[
\EE \mathsf{Bias}_2
\lesssim
\frac{\sum_{i\le k}(1+\tau_i)}{\Meff\gamma}
+
\sum_{i>k}\lambda_i(1+\tau_i).
\]
Using $\tau_i\eqsim i^{a-b}$, we obtain
\[
\sum_{i\le k}(1+\tau_i)\eqsim k+k^{1+a-b},
\qquad
\sum_{i>k}\lambda_i(1+\tau_i)\eqsim k^{1-a}+k^{1-b},
\]
and hence
\[
\EE \mathsf{Bias}_2
\lesssim
\frac{k+k^{1+a-b}}{\Meff\gamma}+k^{1-a}+k^{1-b}.
\]
Let $c:=\min\{a,b\}.$ and distinguish two cases. If $b\ge a$, then $k^{1+a-b}\lesssim k$ and $k^{1-b}\lesssim k^{1-a}$, so
\[
\EE \mathsf{Bias}_2\lesssim \frac{k}{\Meff\gamma}+k^{1-a}.
\]
If $b<a$, then $k\lesssim k^{1+a-b}$ and $k^{1-a}\lesssim k^{1-b}$, so
\[
\EE \mathsf{Bias}_2\lesssim \frac{k^{1+a-b}}{\Meff\gamma}+k^{1-b}.
\]
Hence, in both cases,
\[
\EE \mathsf{Bias}_2\lesssim \frac{k^{1+a-c}}{\Meff\gamma}+k^{1-c},
\qquad c=\min\{a,b\}.
\]
Choosing
\[
k\eqsim \min\{D,(\Meff\gamma)^{1/a}\},
\]
yields
\[
\EE \mathsf{Bias}_2
\lesssim
\max\{D^{1-c},(\Meff\gamma)^{(1-c)/a}\}.
\]

For the lower bound, note that $
\EE[\wb_2^*(\wb_2^*)^\top]=\Ib+\bm\Sigma,$
under the assumption $\EE[\wb_1^*\bm\delta^\top]=0$. Since $\bm\Sigma$ shares the eigenvectors
of $\Hb$ and $\lambda_i\tau_i\eqsim i^{-b}$, the eigenvalues of $\Hb(\Ib+\bm\Sigma)$
satisfy
\[
\lambda_i(1+\tau_i)\eqsim i^{-a}+i^{-b}\eqsim i^{-c},
\qquad c=\min\{a,b\}.
\]
Therefore, applying the same argument as in Lemma~D.2 and Lemma~D.4 of~\cite{linscaling}, we obtain
\[
\EE \mathsf{Bias}_2 \gtrsim (\Meff\gamma)^{(1-c)/a}
\qquad\text{when}\qquad
(\Meff\gamma)^{1/a}\le D/c_0
\]
for some constant $c_0>0$.
Consequently, with probability at least $1-e^{-\Omega(D)}$ over the randomness of the sketch matrix $\Sb$,  the upper bound is
\[
\EE \mathsf{Bias}_2
\lesssim
\max\{D^{1-c},(\Meff\gamma)^{(1-c)/a}\},
\qquad c=\min\{a,b\},
\]
Moreover, in the data-limited regime $(\Meff\gamma)^{1/a}\le D/c_0$, this upper bound is tight up to constants.

Finally, by Lemma~E.1 in~\cite{linscaling},
\[
\frac{\DIM_1}{\Meff}\eqsim \frac{\min\{D,(\Meff\gamma)^{1/a}\}}{\Meff}
\]
with probability at least $1-e^{-\Omega(D)}$ over the randomness of the sketch matrix $\Sb$.
Combining the above estimates and choose $\gamma\eqsim1$, we conclude that with probability at least
$1-e^{-\Omega(D)}$ over $\Sb$,
\begin{align*}
\EE \mathsf{Excess}_2(\vb_M)
&\lesssim \max\{D^{1-a\wedge b},(\Meff\gamma)^{(1-a\wedge b)/a}\}+\frac{\min\{D,(\Meff\gamma)^{1/a}\}}{\Meff}\lesssim 1
\end{align*}

For the term $\alpha\big(\mathsf{Excess}_2(\vb_M)+\|\vb_2^*-\vb_1^*\|^2_{\Sb\Hb\Sb^\top}\big)\frac{\DIM}{\Neff},$
we claim that it is of the same order as $\DIM/\Neff$. Indeed, by the analysis in
mixed-training-analysis, with probability at least $1-e^{-\Omega(D)}$
over the randomness of the sketch matrix $\Sb$,
\begin{align*}
    \EE_{\bm{\delta}}\|\vb_2^*-\vb_1^*\|^2_{\Sb\Hb\Sb^\top}\lesssim 1.
\end{align*}

Therefore, on the intersection of the above high-probability events,
\begin{align*}
    \alpha\Big(\EE\mathsf{Excess}_2(\vb_M)+\EE_{\bm{\delta}}\|\vb_2^*-\vb_1^*\|^2_{\Sb\Hb\Sb^\top}\Big)\frac{\DIM}{\Neff}
    &\lesssim \frac{\DIM}{\Neff}
    \eqsim \frac{\min\{D,(\Neff\gamma)^{1/a}\}}{\Neff}.
\end{align*}

For the term $\big\| (\prod_{t=1}^N\Ib-\gamma_t \Sb\Hb\Sb^\top)(\vb_2^*-\vb_1^*)\big\|^2_{\Sb\Hb\Sb^\top}$, under Assumption~\ref{ass:power-law-source}, we can apply Lemma~D.4 in \cite{linscaling} directly and get that probability at least $1-e^{-\Omega(D)}$ over the randomness of the sketch matrix $\Sb$:
\begin{align*}
    \EE_{\bm\delta}\big\| (\prod_{t=1}^N\Ib-\gamma_t \Sb\Hb\Sb^\top)(\vb_2^*-\vb_1^*)\big\|^2_{\Sb\Hb\Sb^\top}\lesssim \max\big\{D^{1-b},(\Neff\gamma)^{(1-b)/a}  \big\},
    \end{align*}and
    \begin{align*}
      \EE_{\bm\delta}\big\| (\prod_{t=1}^N\Ib-\gamma_t \Sb\Hb\Sb^\top)(\vb_2^*-\vb_1^*)\big\|^2_{\Sb\Hb\Sb^\top}\gtrsim      (\Neff\gamma)^{(1-b)/a} 
    \end{align*} when $(\Neff\gamma)^{1/a} \leq D/c$ for some constant $c>0$.
    Moreover, when $b\geq a+1$
    \begin{align*}
     \EE_{\bm\delta}\big\| (\prod_{t=1}^N\Ib-\gamma_t \Sb\Hb\Sb^\top)(\vb_2^*-\vb_1^*)\big\|^2_{\Sb\Hb\Sb^\top}\lesssim \log \Neff\cdot\max\big\{(\Neff\gamma)^{(1-b)/a},\ D^{1-b}  \big\}     
    \end{align*}

For the last term
\[
\EE\Bigg[
\Big\|
\Big(\prod_{t=1}^N(\Ib-\gamma_t\Sb\Hb\Sb^\top)\Big)(\vb_M-\vb_2^*)
\Big\|_{\Sb\Hb\Sb^\top}^2
\Bigg],
\]
let $\Pb_N:=\prod_{t=1}^N(\Ib-\gamma_t\Sb\Hb\Sb^\top).$ Conditioning on the sketch matrix $\Sb$, we diagonalize $\Sb\Hb\Sb^\top=\Ub\diag(\tilde\lambda_1,\dots,\tilde\lambda_D)\Ub^\top$
with $\tilde\lambda_1\ge\cdots\ge\tilde\lambda_D>0$. Then
\[
\Pb_N^\top \Sb\Hb\Sb^\top \Pb_N
=
\Ub\diag\!\Big(\tilde\lambda_i\prod_{t=1}^N(1-\gamma_t\tilde\lambda_i)^2\Big)\Ub^\top
\preceq
\Big(\prod_{t=1}^N(1-\gamma_t\tilde\lambda_D)\Big)^2 \Sb\Hb\Sb^\top.
\]
Hence
\[
\EE\!\left[
\|\Pb_N(\vb_M-\vb_2^*)\|_{\Sb\Hb\Sb^\top}^2
\,\middle|\, \Sb
\right]
\le
\Big(\prod_{t=1}^N(1-\gamma_t\tilde\lambda_D)\Big)^2
\EE\!\left[
\|\vb_M-\vb_2^*\|_{\Sb\Hb\Sb^\top}^2
\,\middle|\, \Sb
\right].
\]
Using $1-u\le e^{-u}$ and $\sum_{t=1}^N\gamma_t\eqsim \Neff\gamma$, we further obtain
\[
\EE\!\left[
\|\Pb_N(\vb_M-\vb_2^*)\|_{\Sb\Hb\Sb^\top}^2
\,\middle|\, \Sb
\right]
\lesssim
\exp\!\big(-c\,\Neff\gamma\,\mu_D(\Sb\Hb\Sb^\top)\big)\,
\EE\!\left[
\|\vb_M-\vb_2^*\|_{\Sb\Hb\Sb^\top}^2
\,\middle|\, \Sb
\right]
\]
for some constant $c>0$. Since
\[
\EE\!\left[
\|\vb_M-\vb_2^*\|_{\Sb\Hb\Sb^\top}^2
\,\middle|\, \Sb
\right]
=
\EE\!\left[
\mathsf{Excess}_2(\vb_M)
\,\middle|\, \Sb
\right],
\]
it follows that
\[
\EE\!\left[
\|\Pb_N(\vb_M-\vb_2^*)\|_{\Sb\Hb\Sb^\top}^2
\,\middle|\, \Sb
\right]
\lesssim
\exp\!\big(-c\,\Neff\gamma\,\mu_D(\Sb\Hb\Sb^\top)\big)\,
\EE\!\left[
\mathsf{Excess}_2(\vb_M)
\,\middle|\, \Sb
\right].
\]

Now let $\mathcal E_{\mathrm{spec}}$ be the event that
\[
\mu_D(\Sb\Hb\Sb^\top)\eqsim D^{-a},
\]
and let $\mathcal E_{\mathrm{ex}}$ be the event that
\[
\EE\!\left[
\mathsf{Excess}_2(\vb_M)
\,\middle|\, \Sb
\right]
\lesssim
\max\{D^{1-a\wedge b},(\Meff\gamma)^{(1-a\wedge b)/a}\}.
\]
By Lemma~6.2 in~\cite{linscaling} and the above analysis of $\mathsf{Excess}_2(\vb_M)$, both events hold
with probability at least $1-e^{-\Omega(D)}$ over the randomness of the sketch matrix $\Sb$.
Therefore, on the intersection event $\mathcal E_{\mathrm{spec}}\cap\mathcal E_{\mathrm{ex}}$, we have
\[
\EE\Bigg[
\Big\|
\Big(\prod_{t=1}^N(\Ib-\gamma_t\Sb\Hb\Sb^\top)\Big)(\vb_M-\vb_2^*)
\Big\|_{\Sb\Hb\Sb^\top}^2
\Bigg]
\lesssim
\exp\!\left(-c\,\frac{\Neff\gamma}{D^a}\right)
\max\{D^{1-a\wedge b},(\Meff\gamma)^{(1-a\wedge b)/a}\}.
\]
Consequently, with probability at least $1-e^{-\Omega(D)}$ over the randomness of the sketch matrix $\Sb$,
\[
\EE\Bigg[
\Big\|
\Big(\prod_{t=1}^N(\Ib-\gamma_t\Sb\Hb\Sb^\top)\Big)(\vb_M-\vb_2^*)
\Big\|_{\Sb\Hb\Sb^\top}^2
\Bigg]
\lesssim
\exp\!\left(-c\,\frac{\Neff\gamma}{D^a}\right)
\max\{D^{1-a\wedge b},(\Meff\gamma)^{(1-a\wedge b)/a}\}.
\]

Finally, put $\mathsf{Irreducible},\mathsf{Approx}$ and $\mathsf{Excess}$ together. Choose $\gamma \eqsim 1$ and assume $b<a+1$ for simplicity and clarity, we have that with probability at least $1-e^{-\Omega(D)}$ over the randomness of the sketch matrix $\Sb$:
\begin{align*}
    \EE\R_D(\vb_{M+N}) &\lesssim \sigma_1^2+\frac{1}{D^{a-1}}+  \underbrace{\frac{\min\{D,(\Neff)^{1/a}\}}{\Neff}}_{\mathsf{Var}}\\
    & +\underbrace{e^{-\Omega(\frac{\Neff}{D^a})}\cdot\max\big\{\frac{1}{D^{a\wedge b-1}},\frac{1}{(\Meff)^{\frac{a\wedge b-1}{a}}}  \big\}+\max\big\{\frac{1}{D^{b-1}},\frac{1}{(\Neff)^{\frac{b-1}{a}}}  \big\}}_{\mathsf{Bias}} 
\end{align*}

\section{Exact Synthetic-Benefit Condition}
\label{ap:synthetic-benefit}

\begin{proof}[Proof of Theorem~\ref{thm:exact-synthetic-benefit}]
Reindex the $N$ real-stage samples as
$\{(\xb_t,y_t)\}_{t=1}^N$, where
\[
    y_t=\langle\xb_t,\wb_1^*\rangle+\xi_{1,t}.
\]
For any real-stage iterate, the error recursion is
\[
    \wb_t-\wb_1^*
    =
    (\Ib-\gamma_t\xb_t\xb_t^\top)
    (\wb_{t-1}-\wb_1^*)
    +
    \gamma_t\xi_{1,t}\xb_t.
\]
Iterating this recursion for real-only training gives
\begin{align}
\wb_N^{\mathrm{real}}-\wb_1^*
&=
\Phi_N(\wb_0-\wb_1^*)
\notag\\
&\quad+
\sum_{t=1}^{N}
\gamma_t
\Big[
    (\Ib-\gamma_N\xb_N\xb_N^\top)
    \cdots
    (\Ib-\gamma_{t+1}\xb_{t+1}\xb_{t+1}^\top)
\Big]
\xi_{1,t}\xb_t,
\label{eq:real-only-exact-recursion}
\end{align}
where the product inside the summation is interpreted as $\Ib$ when $t=N$.

The two-stage procedure uses exactly the same real samples, label noises,
and stepsizes, but starts the real stage from $\wb_M$. Therefore,
\begin{align}
\wb_{M+N}-\wb_1^*
&=
\Phi_N(\wb_M-\wb_1^*)
\notag\\
&\quad+
\sum_{t=1}^{N}
\gamma_t
\Big[
    (\Ib-\gamma_N\xb_N\xb_N^\top)
    \cdots
    (\Ib-\gamma_{t+1}\xb_{t+1}\xb_{t+1}^\top)
\Big]
\xi_{1,t}\xb_t.
\label{eq:two-stage-exact-recursion}
\end{align}
In particular, the real-stage noise accumulation in
\eqref{eq:real-only-exact-recursion} and
\eqref{eq:two-stage-exact-recursion} is identical.

Let $\bm{\nu}_N$ denote the common real-stage noise-accumulation term in
\eqref{eq:real-only-exact-recursion} and
\eqref{eq:two-stage-exact-recursion}. These two recursions can be written as
\[
    \wb_N^{\mathrm{real}}-\wb_1^*
    =
    \Phi_N(\wb_0-\wb_1^*)+\bm{\nu}_N,
    \qquad
    \wb_{M+N}-\wb_1^*
    =
    \Phi_N(\wb_M-\wb_1^*)+\bm{\nu}_N.
\]

Conditional on $(\wb_M,\xb_1,\ldots,\xb_N)$, the coefficients of all
$\xi_{1,t}$ in $\bm{\nu}_N$ are fixed. Since the real label noises are
independent of the covariates and the synthetic stage, and satisfy
$\EE[\xi_{1,t}]=0$, we have
\[
    \EE[
        \bm{\nu}_N
        \mid
        \wb_M,\xb_1,\ldots,\xb_N
    ]
    =
    \mathbf{0}.
\]
Both $\Phi_N(\wb_0-\wb_1^*)$ and
$\Phi_N(\wb_M-\wb_1^*)$ are measurable under this conditioning. Therefore,
by the tower property,
\begin{align*}
\EE\!\left[
    \big(\Phi_N(\wb_0-\wb_1^*)\big)^\top
    \Hb\bm{\nu}_N
\right]
&=0,\\
\EE\!\left[
    \big(\Phi_N(\wb_M-\wb_1^*)\big)^\top
    \Hb\bm{\nu}_N
\right]
&=0.
\end{align*}
Expanding the two squared errors consequently gives
\begin{align*}
\EE\|\wb_N^{\mathrm{real}}-\wb_1^*\|_{\Hb}^2
&=
\EE\|\Phi_N(\wb_0-\wb_1^*)\|_{\Hb}^2
+
\EE\|\bm{\nu}_N\|_{\Hb}^2,\\
\EE\|\wb_{M+N}-\wb_1^*\|_{\Hb}^2
&=
\EE\|\Phi_N(\wb_M-\wb_1^*)\|_{\Hb}^2
+
\EE\|\bm{\nu}_N\|_{\Hb}^2.
\end{align*}
The final terms are identical because the two procedures use the same real
samples, label noises, and stepsizes. They therefore cancel in the risk
difference.

Under Assumption~\ref{assump:noise}, for any $\wb$, $\mathcal{E}_1(\wb)=\frac{1}{2}
\|\wb-\wb_1^*\|_{\Hb}^2.$ Hence,
\begin{align*}
&2\left(
    \EE[\mathcal{E}_1(\wb_N^{\mathrm{real}})]
    -
    \EE[\mathcal{E}_1(\wb_{M+N})]
\right)\\
&\qquad=
\EE\|\Phi_N(\wb_0-\wb_1^*)\|_{\Hb}^2
-
\EE\|\Phi_N(\wb_M-\wb_1^*)\|_{\Hb}^2.
\end{align*}

Since $\wb_0$ is deterministic, the definition of $\mathbf{K}_N$ gives
\begin{align*}
\EE\|\Phi_N(\wb_0-\wb_1^*)\|_{\Hb}^2
&=
(\wb_0-\wb_1^*)^\top
\EE_{\xb_1,\ldots,\xb_N}
[\Phi_N^\top\Hb\Phi_N]
(\wb_0-\wb_1^*)\\
&=
\|\wb_0-\wb_1^*\|_{\mathbf{K}_N}^2.
\end{align*}
Similarly, because $\wb_M$ is independent of the real-stage covariates,
the tower property yields
\begin{align*}
\EE\|\Phi_N(\wb_M-\wb_1^*)\|_{\Hb}^2
&=
\EE_{\wb_M}\!\left[
    (\wb_M-\wb_1^*)^\top
    \EE_{\xb_1,\ldots,\xb_N}
    [\Phi_N^\top\Hb\Phi_N]
    (\wb_M-\wb_1^*)
\right]\\
&=
\EE\|\wb_M-\wb_1^*\|_{\mathbf{K}_N}^2.
\end{align*}

Finally, decompose $\wb_M-\wb_1^*=\big(\wb_M-\EE[\wb_M]\big)+
\big(\EE[\wb_M]-\wb_1^*\big)$ and Expandthe corresponding quadratic form to get
\begin{align*}
\EE\|\wb_M-\wb_1^*\|_{\mathbf{K}_N}^2
&=
\EE\|\wb_M-\EE[\wb_M]\|_{\mathbf{K}_N}^2
+
\|\EE[\wb_M]-\wb_1^*\|_{\mathbf{K}_N}^2\\
&\quad+
2\EE\!\left[
    \big(\wb_M-\EE[\wb_M]\big)^\top
    \mathbf{K}_N
    \big(\EE[\wb_M]-\wb_1^*\big)
\right].
\end{align*}
The last term is zero because
$\EE[\wb_M-\EE[\wb_M]]=\mathbf{0}$. Moreover,
\begin{align*}
\EE\|\wb_M-\EE[\wb_M]\|_{\mathbf{K}_N}^2
&=
\EE\tr\!\left(
    \mathbf{K}_N
    (\wb_M-\EE[\wb_M])
    (\wb_M-\EE[\wb_M])^\top
\right)\\
&=
\tr\!\left(
    \mathbf{K}_N\operatorname{Cov}(\wb_M)
\right).
\end{align*}
Combining the preceding identities yields
\begin{align*}
&2\left(
    \EE[\mathcal{E}_1(\wb_N^{\mathrm{real}})]
    -
    \EE[\mathcal{E}_1(\wb_{M+N})]
\right)\\
&\qquad=
\|\wb_0-\wb_1^*\|_{\mathbf{K}_N}^2
-
\|\EE[\wb_M]-\wb_1^*\|_{\mathbf{K}_N}^2
-
\tr\!\left(
    \mathbf{K}_N\operatorname{Cov}(\wb_M)
\right).
\end{align*}
Therefore,
$\EE[\mathcal{E}_1(\wb_{M+N})]
<\EE[\mathcal{E}_1(\wb_N^{\mathrm{real}})]$
if and only if
\[
\|\EE[\wb_M]-\wb_1^*\|_{\mathbf{K}_N}^2
+
\tr\!\left(
    \mathbf{K}_N\operatorname{Cov}(\wb_M)
\right)
<
\|\wb_0-\wb_1^*\|_{\mathbf{K}_N}^2,
\]
which proves Theorem~\ref{thm:exact-synthetic-benefit}.
\end{proof}

\begin{proof}[Proof of Corollary~\ref{cor:one-dimensional-benefit}]
Since $x_s^2=1$, the $s$-th synthetic update with stepsize $1/s$ satisfies
\begin{align*}
w_s
&=
w_{s-1}
-\frac{1}{s}
\big(
    w_{s-1}x_s-(\theta+\delta)x_s-\xi_{2,s}
\big)x_s\\
&=
\left(1-\frac{1}{s}\right)w_{s-1}
+
\frac{1}{s}
\big(
    \theta+\delta+\xi_{2,s}x_s
\big).
\end{align*}
Induction over $s$ therefore gives
\[
    w_M
    =
    \theta+\delta
    +
    \frac{1}{M}
    \sum_{s=1}^{M}\xi_{2,s}x_s.
\]
Because the synthetic noises are independent, centered, and independent of
the Rademacher covariates,
\begin{equation}
    \EE[w_M]=\theta+\delta,
    \qquad
    \operatorname{Var}(w_M)=\frac{\sigma_2^2}{M}.
\label{eq:one-dimensional-stage-one-moments}
\end{equation}

In one dimension, $\Hb=\EE[x^2]=1$. Moreover, every real-stage update has
multiplicative operator
$1-\gamma_t x_t^2=1-\gamma_t$. Thus,
\[
    \Phi_N=\prod_{t=1}^{N}(1-\gamma_t),
    \qquad
    \mathbf{K}_N=
    \left(
        \prod_{t=1}^{N}(1-\gamma_t)
    \right)^2.
\]
Applying Theorem~\ref{thm:exact-synthetic-benefit} and using
\eqref{eq:one-dimensional-stage-one-moments}, we obtain
\begin{align*}
&2\left(
    \EE[\mathcal{E}_1(w_N^{\mathrm{real}})]
    -
    \EE[\mathcal{E}_1(w_{M+N})]
\right)\\
&\qquad=
\left(
    \prod_{t=1}^{N}(1-\gamma_t)
\right)^2
\left[
    (w_0-\theta)^2
    -
    \delta^2
    -
    \frac{\sigma_2^2}{M}
\right].
\end{align*}
Since $0<\gamma_t<1$, the multiplicative factor is strictly positive. Therefore,
two-stage training has strictly smaller expected risk if and only if
\[
    \delta^2+\frac{\sigma_2^2}{M}
    <
    (w_0-\theta)^2.
\]
If $\delta^2\geq(w_0-\theta)^2$, this inequality cannot hold for any $M$.
Otherwise, it is equivalent to
\[
    M>
    \frac{\sigma_2^2}
    {(w_0-\theta)^2-\delta^2},
\]
which completes the proof.
\end{proof}

\section{constant stepsize SGD with iterate averaging}\label{ap:constant-sgd}
In this section, we also denote total sample size as $N$, where synthetic sample size is $pN$ and real sample size is $(1-p)N$. This light abuse of notation will achieve more simplicity and clarity.
\subsection{Upper bound analysis}
We use the same definition in Appendix~\ref{ap:mix-upper}. 
\begin{align*}
    &\Bb_t = \EE [\bm{\eta}^{\bias}_t \otimes \bm{\eta}^{\bias}_t], \quad
\Vb_t = \EE [\bm{\eta}^{\var}_t \otimes \bm{\eta}^{\var}_t],\\
&\Db_t = \EE [\bm{\eta}^{\drift}_t \otimes \bm{\eta}^{\drift}_t], \quad
\Fb_t = \EE [\bm{\eta}^{\fluct}_t \otimes \bm{\eta}^{\fluct}_t].
\end{align*}

Using the definition of $\Bb_t,\Vb_t,\Db_t,\Fb_t$, we can then decompose Excess Risk using Cauchy--Schwarz Inequality:
\begin{align*}
    \EE [ \R(\bar\wb_N) - \R(\wb_1^*) ] &= \frac{1}{2}\big\langle\Hb, \EE[\bar{\bm{\eta}}_N \otimes \bar{\bm{\eta}}_N] \big\rangle \\
    &= \frac{1}{2}\big\langle\Hb, \EE[(\bar{\bm{\eta}}_N^\bias + \bar{\bm{\eta}}_N^\var + \bar{\bm{\eta}}_N^\drift + \bar{\bm{\eta}}_N^\fluct) \otimes (\bar{\bm{\eta}}_N^\bias + \bar{\bm{\eta}}_N^\var + \bar{\bm{\eta}}_N^\drift + \bar{\bm{\eta}}_N^\fluct)]\big \rangle \\
    &\le 2 \big\langle \Hb, \EE[\bar{\bm{\eta}}_N^\bias \otimes \bar{\bm{\eta}}_N^\bias] \big\rangle + 2 \big\langle \Hb, \EE[\bar{\bm{\eta}}_N^\var \otimes \bar{\bm{\eta}}_N^\var] \big\rangle \\
    &\quad + 2 \big\langle \Hb, \EE[\bar{\bm{\eta}}_N^\drift \otimes \bar{\bm{\eta}}_N^\drift] \big\rangle + 2 \big\langle \Hb, \EE[\bar{\bm{\eta}}_N^\fluct \otimes \bar{\bm{\eta}}_N^\fluct] \big\rangle \\
    &= 2 \big\langle \Hb, \bar\Bb_N \big\rangle + 2 \big\langle \Hb, \bar\Vb_N \big\rangle + 2 \big\langle \Hb, \bar\Db_N \big\rangle + 2 \big\langle \Hb, \bar\Fb_N \big\rangle.
\end{align*}

\paragraph{Bias upper bound}
Using the defined operators, the update rule of the iterates imply the following
recursive form of $\Bb_t$:
\begin{equation}\label{eq:update_Bt_avg}
    \Bb_t = (\cI - \gamma\cT)\circ \Bb_{t-1}, \quad \Bb_0 = \betab_0\otimes \betab_0,
\end{equation}
Note that this bias term $\big\langle \Hb, \bar\Bb_N \big\rangle$ is the same as that in \cite{zou2023benign}, so we can apply Lemma B.11 in \cite{zou2023benign} to bound it directly:
\begin{lemma}[A bias upper bound for Avg-SGD]\label{lemma:HB-constant-upper-bound}
Suppose Assumptions \ref{assump:second-moment} and \ref{assump:fourth-moment} hold. Let $\Neff=N$. 
Consider \eqref{eq:update_Bt_avg}.
Suppose $\gamma < 1/(\alpha\tr(\Hb))$. We have 
\begin{align*}
\frac12\big\langle \Hb, \bar\Bb_N \big\rangle &\le \frac{1}{\gamma^2 \Neff^2}\cdot\|\wb_0-\wb^*\|_{\Hb_{0:k^*}^{-1}}^2+\|\wb_0-\wb^*\|_{\Hb_{k^*:\infty}}^{2}\notag\\
&\quad +\frac{2\alpha\big(\|\wb_0-\wb^*\|_{\Ib_{0:k^*}}^2 + \Neff\gamma\|\wb_0-\wb^*\|_{\Hb_{k^*:\infty}}^2\big)}{\Neff\gamma(1-\gamma \alpha\tr(\Hb))}\cdot\frac{\DIM}{\Neff}
\end{align*}
where $k^* = \max \{k: \lambda_k \ge \frac{1}{\gamma \Neff}\}$ and $\DIM := k^* + \gamma^2 \Neff^2 \sum_{i>k^*}\lambda_i^2.$
\end{lemma}
\begin{proof}
See proof of Lemma B.11 in \cite{zou2023benign}.
\end{proof}

\paragraph{Variance upper bound}
Using the defined operators, the update rule of the variance process implies
\begin{equation}\label{eq:update_Vt_avg}
\Vb_t = (\cI-\gamma\cT)\circ \Vb_{t-1} + \gamma^2\bSigma_t,\qquad \Vb_0 = \boldsymbol{0},
\end{equation}
where
\[
\bSigma_t:=\EE[\xi_t^2 \xb_t\xb_t^\top].
\]
Since $\xb_t$ is independent of $z_t$, conditioning on the source indicator gives
\[
\bSigma_t
=
\EE\!\left[\EE[\xi_t^2 \xb_t\xb_t^\top \mid z_t]\right]
=
(1-p)\bSigma_1+p\bSigma_2
\preceq
\bigl((1-p)\sigma_1^2+p\sigma_2^2\bigr)\Hb.
\]
Define
\[
\bar\sigma^2:=(1-p)\sigma_1^2+p\sigma_2^2.
\]
Then \eqref{eq:update_Vt_avg} has exactly the same form as the variance recursion for
constant-stepsize SGD with iterate averaging analyzed in \cite{zou2023benign}, with noise
level upper bounded by $\sigma_{\mathrm{mix}}^2$. Therefore, by applying the averaged-SGD
variance bound in \cite{zou2023benign}, we obtain the following result. The corresponding sharp
variance characterization for iterate averaging is stated in Theorem~2.1 and developed in
Lemma B.6 of \cite{zou2023benign}.

\begin{lemma}[A variance upper bound for Avg-SGD]\label{lemma:HV-constant-upper-bound}
Suppose Assumptions \ref{assump:second-moment}, \ref{assump:fourth-moment} and
\ref{assump:noise} hold.
Consider \eqref{eq:update_Vt_avg}. Let $\Neff=N$ and $\bar\sigma^2=(1-p)\sigma_1^2+p\sigma_2^2$.
Suppose $\gamma < 1/(\alpha\tr(\Hb))$. Then
\begin{equation*}
\frac12\big\langle \Hb,\bar\Vb_N\big\rangle
\le
\frac{\bar\sigma^2}{1-\gamma\alpha\tr(\Hb)}
\cdot
\frac{\DIM}{\Neff},
\end{equation*}
where
$
k^*=\max\left\{k:\lambda_k\ge \frac{1}{\gamma \Neff}\right\},
$
and
$
\DIM:=k^*+\gamma^2\Neff^2\sum_{i>k^*}\lambda_i^2.
$
\end{lemma}

\begin{proof}
The recursion \eqref{eq:update_Vt_avg} is identical to the variance recursion of
constant-stepsize SGD with iterate averaging, except that the noise covariance is now
$\bSigma_t$ instead of a single-source covariance matrix. Since
\[
\bSigma_t\preceq \bar\sigma^2 \Hb
\qquad \text{for all } t,
\]
the proof of the variance upper bound in \cite{zou2023benign} applies verbatim with
$\sigma^2$ replaced by $\sigma_{\mathrm{mix}}^2$.
Hence
\[
\frac12\big\langle \Hb,\bar\Vb_N\big\rangle
\le
\frac{\bar\sigma^2}{1-\gamma\alpha\tr(\Hb)}
\left(
\frac{k^*}{\Neff}
+\Neff\gamma^2\sum_{i>k^*}\lambda_i^2
\right)
=
\frac{\bar\sigma^2}{1-\gamma\alpha\tr(\Hb)}
\cdot
\frac{\DIM}{\Neff}.
\]
This completes the proof.
\end{proof}

\paragraph{Drift upper bound}
We now analyze the drift contribution
\[
\frac12\big\langle \Hb,\bar\Db_N\big\rangle
=
\frac12\big\langle \Hb,\EE[\bar{\bm{\eta}}^{\drift}_N\otimes \bar{\bm{\eta}}^{\drift}_N]\big\rangle,
\qquad
\bar{\bm{\eta}}^{\drift}_N:=\frac1N\sum_{t=0}^{N-1}\bm{\eta}^{\drift}_t.
\]

\begin{lemma}[A drift upper bound for Avg-SGD]\label{lemma:HD-constant-upper-bound}
Suppose Assumptions \ref{assump:second-moment}, \ref{assump:fourth-moment} and
\ref{assump:noise} hold.
Let $\Neff=N$.
Suppose $\gamma<1/(\alpha\tr(\Hb))$. Then
\begin{align*}
\frac12\big\langle \Hb,\bar\Db_N\big\rangle
&\le
\frac{p^2}{2}
\left\|
\left(
\Ib-\frac{1}{\gamma \Neff}\Hb^{-1}\bigl(\Ib-(\Ib-\gamma\Hb)^{\Neff}\bigr)
\right)\bm{\delta}
\right\|_{\Hb}^2 \\
&\quad +
\frac{\alpha p^2\|\bm{\delta}\|_{\Hb}^2}{1-\gamma\alpha\tr(\Hb)}
\cdot
\frac{\DIM}{\Neff},
\end{align*}
where
\[
k^*=\max\left\{k:\lambda_k\ge \frac{1}{\gamma \Neff}\right\},
\qquad
\DIM:=k^*+\gamma^2\Neff^2\sum_{i>k^*}\lambda_i^2.
\]
\end{lemma}

\begin{proof}
Recall that the drift process satisfies
\[
\bm{\eta}^{\drift}_t
=
(\Ib-\gamma\xb_t\xb_t^\top)\bm{\eta}^{\drift}_{t-1}
+
\gamma p\Hb\bm{\delta},
\qquad
\bm{\eta}^{\drift}_0=\boldsymbol{0}.
\]

We first separate the drift process into its mean part and centered fluctuation part.
Define
\[
\mb_t:=\EE[\bm{\eta}^{\drift}_t].
\]
Since $\xb_t$ is independent of $\bm{\eta}^{\drift}_{t-1}$ and $\EE[\xb_t\xb_t^\top]=\Hb$,
taking expectation on both sides yields
\begin{equation}\label{eq:drift-mean-recursion-avg}
\mb_t
=
(\Ib-\gamma\Hb)\mb_{t-1}
+
\gamma p\Hb\bm{\delta},
\qquad
\mb_0=\boldsymbol{0}.
\end{equation}
Let
\[
\widetilde{\bm{\eta}}^{\drift}_t:=\bm{\eta}^{\drift}_t-\mb_t.
\]
Subtracting \eqref{eq:drift-mean-recursion-avg} from the recursion of
$\bm{\eta}^{\drift}_t$ gives
\begin{equation}\label{eq:drift-centered-recursion-avg}
\widetilde{\bm{\eta}}^{\drift}_t
=
(\Ib-\gamma\xb_t\xb_t^\top)\widetilde{\bm{\eta}}^{\drift}_{t-1}
+
\gamma(\Hb-\xb_t\xb_t^\top)\mb_{t-1},
\qquad
\widetilde{\bm{\eta}}^{\drift}_0=\boldsymbol{0}.
\end{equation}
Moreover, by construction,
\[
\EE[\widetilde{\bm{\eta}}^{\drift}_t]=\boldsymbol{0},
\qquad \forall t\ge 0.
\]

Now define the averaged quantities
\[
\bar{\mb}_N:=\frac1N\sum_{t=0}^{N-1}\mb_t,
\qquad
\bar{\widetilde{\bm{\eta}}}^{\drift}_N:=\frac1N\sum_{t=0}^{N-1}\widetilde{\bm{\eta}}^{\drift}_t.
\]
Then
\[
\bar{\bm{\eta}}^{\drift}_N=\bar{\mb}_N+\bar{\widetilde{\bm{\eta}}}^{\drift}_N.
\]
Since $\EE[\bar{\widetilde{\bm{\eta}}}^{\drift}_N]=\boldsymbol{0}$, the cross term vanishes and
\begin{equation}\label{eq:avg-drift-split}
\bar\Db_N
=
\bar{\mb}_N\otimes \bar{\mb}_N
+
\bar{\widetilde\Db}_N,
\qquad
\bar{\widetilde\Db}_N
:=
\EE\!\left[
\bar{\widetilde{\bm{\eta}}}^{\drift}_N
\otimes
\bar{\widetilde{\bm{\eta}}}^{\drift}_N
\right].
\end{equation}
Therefore,
\begin{equation}\label{eq:avg-drift-risk-split}
\frac12\big\langle \Hb,\bar\Db_N\big\rangle
=
\frac12\|\bar{\mb}_N\|_{\Hb}^2
+
\frac12\big\langle \Hb,\bar{\widetilde\Db}_N\big\rangle.
\end{equation}

Unrolling \eqref{eq:drift-mean-recursion-avg}, we obtain
\[
\mb_t
=
p\bigl(\Ib-(\Ib-\gamma\Hb)^t\bigr)\bm{\delta},
\qquad t\ge 0.
\]
Hence
\begin{align*}
\bar{\mb}_N
&=
\frac{p}{N}\sum_{t=0}^{N-1}\bigl(\Ib-(\Ib-\gamma\Hb)^t\bigr)\bm{\delta} \\
&=
p\left(
\Ib-\frac1N\sum_{t=0}^{N-1}(\Ib-\gamma\Hb)^t
\right)\bm{\delta}.
\end{align*}
Using the matrix geometric-series identity
\[
\sum_{t=0}^{N-1}(\Ib-\gamma\Hb)^t
=
(\gamma\Hb)^{-1}\bigl(\Ib-(\Ib-\gamma\Hb)^N\bigr),
\]
we get
\begin{equation}\label{eq:avg-drift-mean-closed}
\bar{\mb}_N
=
p\left(
\Ib-\frac{1}{\gamma N}\Hb^{-1}\bigl(\Ib-(\Ib-\gamma\Hb)^N\bigr)
\right)\bm{\delta}.
\end{equation}
Therefore,
\begin{equation}\label{eq:avg-drift-mean-risk}
\frac12\|\bar{\mb}_N\|_{\Hb}^2
=
\frac{p^2}{2}
\left\|
\left(
\Ib-\frac{1}{\gamma N}\Hb^{-1}\bigl(\Ib-(\Ib-\gamma\Hb)^N\bigr)
\right)\bm{\delta}
\right\|_{\Hb}^2.
\end{equation}

We next bound
\[
\frac12\big\langle \Hb,\bar{\widetilde\Db}_N\big\rangle.
\]
Define
\[
\widetilde{\Db}_t
:=
\EE[\widetilde{\bm{\eta}}^{\drift}_t\otimes \widetilde{\bm{\eta}}^{\drift}_t].
\]
From \eqref{eq:drift-centered-recursion-avg}, we obtain the recursion
\begin{equation}\label{eq:avg-drift-fluct-recursion}
\widetilde{\Db}_t
=
(\cI-\gamma\cT)\circ \widetilde{\Db}_{t-1}
+
\gamma^2 \widetilde{\bSigma}^{\drift}_t,
\qquad
\widetilde{\Db}_0=\boldsymbol{0},
\end{equation}
where
\[
\widetilde{\bSigma}^{\drift}_t
:=
\EE\!\left[
\bigl((\Hb-\xb_t\xb_t^\top)\mb_{t-1}\bigr)
\otimes
\bigl((\Hb-\xb_t\xb_t^\top)\mb_{t-1}\bigr)
\right].
\]
Indeed, the cross term vanishes because
\[
\EE[(\Hb-\xb_t\xb_t^\top)\mb_{t-1}]
=
(\Hb-\EE[\xb_t\xb_t^\top])\mb_{t-1}
=
\boldsymbol{0}.
\]

Let
\[
\Ab_t:=\mb_{t-1}\otimes \mb_{t-1}.
\]
Then
\begin{align*}
\widetilde{\bSigma}^{\drift}_t
&=
\EE\!\left[(\Hb-\xb_t\xb_t^\top)\Ab_t(\Hb-\xb_t\xb_t^\top)\right] \\
&=
\EE[\xb_t\xb_t^\top \Ab_t \xb_t\xb_t^\top]-\Hb \Ab_t \Hb \\
&=
(\cM-\widetilde{\cM})\circ \Ab_t.
\end{align*}
Since $\cM-\widetilde{\cM}$ is a PSD mapping, we have
$\widetilde{\bSigma}^{\drift}_t\succeq \boldsymbol{0}$.
On the other hand, by Assumption~\ref{assump:fourth-moment}\ref{item:fourth-moement-upper},
\[
\EE[\xb_t\xb_t^\top \Ab_t \xb_t\xb_t^\top]
=
\cM\circ \Ab_t
\preceq
\alpha\,\tr(\Hb\Ab_t)\,\Hb
=
\alpha\,\|\mb_{t-1}\|_{\Hb}^2\,\Hb.
\]
Therefore,
\begin{equation}\label{eq:avg-drift-sigma-upper}
\widetilde{\bSigma}^{\drift}_t
\preceq
\alpha\,\|\mb_{t-1}\|_{\Hb}^2\,\Hb.
\end{equation}

Next, by the explicit formula for $\mb_t$,
\[
\mb_{t-1}=p\bigl(\Ib-(\Ib-\gamma\Hb)^{t-1}\bigr)\bm{\delta},
\]
and hence
\begin{align*}
\|\mb_{t-1}\|_{\Hb}^2
&=
p^2\sum_i \lambda_i
\left(1-(1-\gamma\lambda_i)^{t-1}\right)^2\delta_i^2 \\
&\le
p^2\sum_i \lambda_i \delta_i^2
=
p^2\|\bm{\delta}\|_{\Hb}^2.
\end{align*}
Substituting this into \eqref{eq:avg-drift-sigma-upper} gives the uniform bound
\begin{equation}\label{eq:avg-drift-sigma-final}
\widetilde{\bSigma}^{\drift}_t
\preceq
\alpha p^2\|\bm{\delta}\|_{\Hb}^2\,\Hb,
\qquad \forall t.
\end{equation}

Now \eqref{eq:avg-drift-fluct-recursion} has exactly the same form as the variance
recursion for constant-stepsize SGD with iterate averaging, with noise level upper bounded by
\[
\sigma_{\drift}^2:=\alpha p^2\|\bm{\delta}\|_{\Hb}^2.
\]
Therefore, applying Lemma~\ref{lemma:HV-constant-upper-bound} with
$\bar\sigma^2$ replaced by $\sigma_{\drift}^2$, we obtain
\begin{equation}\label{eq:avg-drift-fluct-final}
\frac12\big\langle \Hb,\bar{\widetilde\Db}_N\big\rangle
\le
\frac{\alpha p^2\|\bm{\delta}\|_{\Hb}^2}{1-\gamma\alpha\tr(\Hb)}
\cdot
\frac{\DIM}{N}.
\end{equation}

Finally, combining \eqref{eq:avg-drift-risk-split}, \eqref{eq:avg-drift-mean-risk}
and \eqref{eq:avg-drift-fluct-final}, we conclude that
\begin{align*}
\frac12\big\langle \Hb,\bar\Db_N\big\rangle
&\le
\frac{p^2}{2}
\left\|
\left(
\Ib-\frac{1}{\gamma N}\Hb^{-1}\bigl(\Ib-(\Ib-\gamma\Hb)^N\bigr)
\right)\bm{\delta}
\right\|_{\Hb}^2 \\
&\quad +
\frac{\alpha p^2\|\bm{\delta}\|_{\Hb}^2}{1-\gamma\alpha\tr(\Hb)}
\cdot
\frac{\DIM}{N}.
\end{align*}
This completes the proof.
\end{proof}

\paragraph{Fluctuation upper bound}
Recall that the fluctuation process satisfies
\[
\bm{\eta}^{\fluct}_t
=
(\Ib-\gamma \xb_t\xb_t^\top)\bm{\eta}^{\fluct}_{t-1}
+
\gamma z_t \xb_t\xb_t^\top\bm{\delta}
-
\gamma p\Hb\bm{\delta},
\qquad
\bm{\eta}^{\fluct}_0=\boldsymbol{0}.
\]
We decompose the driving term into two parts:
\[
z_t \xb_t\xb_t^\top\bm{\delta}-p\Hb\bm{\delta}
=
z_t(\xb_t\xb_t^\top-\Hb)\bm{\delta}
+
(z_t-p)\Hb\bm{\delta}.
\]
Accordingly, define
\[
\bm{\eta}^{\fluct}_t
=
\bm{\eta}^{\cov}_t+\bm{\eta}^{\sch}_t,
\]
where
\[
\begin{cases}
\bm{\eta}^{\cov}_t
=
(\Ib-\gamma \xb_t\xb_t^\top)\bm{\eta}^{\cov}_{t-1}
+
\gamma z_t(\xb_t\xb_t^\top-\Hb)\bm{\delta},
\\[0.5ex]
\bm{\eta}^{\cov}_0=\boldsymbol{0},
\end{cases}
\qquad
\begin{cases}
\bm{\eta}^{\sch}_t
=
(\Ib-\gamma \xb_t\xb_t^\top)\bm{\eta}^{\sch}_{t-1}
+
\gamma(z_t-p)\Hb\bm{\delta},
\\[0.5ex]
\bm{\eta}^{\sch}_0=\boldsymbol{0}.
\end{cases}
\]
Thus, for the averaged iterates,
\[
\bar{\bm{\eta}}^{\fluct}_N
=
\bar{\bm{\eta}}^{\cov}_N+\bar{\bm{\eta}}^{\sch}_N.
\]
Hence, by $(a+b)^2\le 2a^2+2b^2$,
\begin{equation}\label{eq:fluctuation-split-avg}
\frac12\big\langle \Hb,\bar\Fb_N\big\rangle
=
\frac12\EE\big[\|\bar{\bm{\eta}}^{\fluct}_N\|_{\Hb}^2\big]
\le
\EE\big[\|\bar{\bm{\eta}}^{\cov}_N\|_{\Hb}^2\big]
+
\EE\big[\|\bar{\bm{\eta}}^{\sch}_N\|_{\Hb}^2\big].
\end{equation}

Define
$
\Fb_t^{\cov}
:=
\EE\!\left[
\bm{\eta}^{\cov}_t\otimes\bm{\eta}^{\cov}_t
\right].
$
Then $\Fb_t^{\cov}$ satisfies the recursion
\begin{equation}\label{eq:fluct-cov-recursion}
\Fb_t^{\cov}
=
(\cI-\gamma\cT)\circ \Fb_{t-1}^{\cov}
+
\gamma^2\bSigma_t^{\cov},
\qquad
\Fb_0^{\cov}=\boldsymbol{0},
\end{equation}
where
\[
\bSigma_t^{\cov}
:=
\EE\!\left[
\bigl(z_t(\xb_t\xb_t^\top-\Hb)\bm{\delta}\bigr)
\otimes
\bigl(z_t(\xb_t\xb_t^\top-\Hb)\bm{\delta}\bigr)
\right].
\]
Indeed, the cross term vanishes as we have proved in last-iterate analysis.

Next, since $z_t\in\{0,1\}$ and $\EE[z_t]=p$, we have
\begin{align*}
\bSigma_t^{\cov}
&=
\EE\!\left[
z_t^2
\bigl((\xb_t\xb_t^\top-\Hb)\bm{\delta}\bigr)
\otimes
\bigl((\xb_t\xb_t^\top-\Hb)\bm{\delta}\bigr)
\right] \\
&=
p\,
\EE\!\left[
\bigl((\xb_t\xb_t^\top-\Hb)\bm{\delta}\bigr)
\otimes
\bigl((\xb_t\xb_t^\top-\Hb)\bm{\delta}\bigr)
\right].
\end{align*}
Let $\Ab:=\bm{\delta}\otimes\bm{\delta}$. Then
\begin{align*}
\EE\!\left[
\bigl((\xb_t\xb_t^\top-\Hb)\bm{\delta}\bigr)
\otimes
\bigl((\xb_t\xb_t^\top-\Hb)\bm{\delta}\bigr)
\right]
&=
\EE\!\left[
(\xb_t\xb_t^\top-\Hb)\Ab(\xb_t\xb_t^\top-\Hb)
\right] \\
&=
\EE[\xb_t\xb_t^\top \Ab \xb_t\xb_t^\top]-\Hb\Ab\Hb \\
&=
(\cM-\widetilde{\cM})\circ \Ab.
\end{align*}
Since $\cM-\widetilde{\cM}$ is a PSD mapping, and by Assumption~\ref{assump:fourth-moment},
\[
\cM\circ \Ab
=
\EE[\xb_t\xb_t^\top \Ab \xb_t\xb_t^\top]
\preceq
\alpha\,\tr(\Hb\Ab)\,\Hb
=
\alpha\|\bm{\delta}\|_{\Hb}^2\,\Hb,
\]
we obtain the upper bound
\begin{equation}\label{eq:Sigma-cov-upper}
\bSigma_t^{\cov}
\preceq
p\alpha\|\bm{\delta}\|_{\Hb}^2\,\Hb.
\end{equation}

Therefore, the recursion \eqref{eq:fluct-cov-recursion} has exactly the same form as the
variance recursion for constant-stepsize SGD with iterate averaging, with noise level upper bounded by
\[
\sigma_{\cov}^2:=p\alpha\|\bm{\delta}\|_{\Hb}^2.
\]
Applying Lemma~\ref{lemma:HV-constant-upper-bound} with
$\bar\sigma^2$ replaced by $\sigma_{\cov}^2$, we get
\begin{equation}\label{eq:fluct-cov-final}
\frac12
\EE\big[\|\bar{\bm{\eta}}^{\cov}_N\|_{\Hb}^2\big]
=
\frac12\big\langle \Hb,\bar\Fb_N^{\cov}\big\rangle
\le
\frac{p\alpha\|\bm{\delta}\|_{\Hb}^2}{1-\gamma\alpha\tr(\Hb)}
\cdot
\frac{\DIM}{N}.
\end{equation}

We now bound the second component
\[
\bm{\eta}^{\sch}_t
=
(\Ib-\gamma \xb_t\xb_t^\top)\bm{\eta}^{\sch}_{t-1}
+
\gamma(z_t-p)\Hb\bm{\delta},
\qquad
\bm{\eta}^{\sch}_0=\boldsymbol{0}.
\]
Its averaged iterate is
\[
\bar{\bm{\eta}}^{\sch}_N
:=
\frac{1}{N}\sum_{t=1}^N \bm{\eta}^{\sch}_t.
\]

Then we use the following lemma to bound $\langle \Hb,\bar{\bm{\eta}}^{\sch}_N\otimes\bar{\bm{\eta}}^{\sch}_N\rangle$.

\begin{lemma}[A schedule-fluctuation upper bound for Avg-SGD]
\label{lem:schedule-fluctuation-clean}
Suppose Assumptions \ref{assump:second-moment} and \ref{assump:fourth-moment} hold.
Assume the source indicators $(z_1,\dots,z_N)$ follow the fixed-budget model with exactly $m$ ones, and let $p:=\frac{m}{N}.$
Consider the schedule fluctuation recursion
\[
\bm{\eta}^{\sch}_t
=
(\Ib-\gamma \xb_t\xb_t^\top)\bm{\eta}^{\sch}_{t-1}
+
\gamma(z_t-p)\Hb\bm{\delta},
\qquad
\bm{\eta}^{\sch}_0=\boldsymbol{0},
\]
and define its averaged iterate by
$\bar{\bm{\eta}}^{\sch}_N
:=
\frac1N\sum_{t=1}^N \bm{\eta}^{\sch}_t.$
Suppose $\gamma<1/(\alpha\tr(\Hb))$. Let
\[
k^*:=\max\left\{k:\lambda_k\ge \frac{1}{\gamma N}\right\},
\qquad
\DIM:=k^*+\gamma^2N^2\sum_{i>k^*}\lambda_i^2.
\]
Then
\begin{equation}\label{eq:schedule-fluctuation-clean-bound}
\frac12\EE\!\left[\|\bar{\bm{\eta}}^{\sch}_N\|_{\Hb}^2\right]
\le
\frac{p(1-p)\|\bm{\delta}\|_{\Hb}^2}{N-1}
\left(
1+
\frac{2\alpha\gamma\tr(\Hb)}{1-\gamma\alpha\tr(\Hb)}
\DIM
\right).
\end{equation}
\end{lemma}

\begin{proof}
Let
\[
\Bb_t:=\Ib-\gamma \xb_t\xb_t^\top.
\]
By repeated substitution of the recursion
\[
\bm{\eta}^{\sch}_t
=
\Bb_t \bm{\eta}^{\sch}_{t-1}
+
\gamma(z_t-p)\Hb\bm{\delta},
\qquad
\bm{\eta}^{\sch}_0=0,
\]
we obtain, for every $t\in\{1,\dots,N\}$,
\begin{equation}\label{eq:schedule-trajectory-expand}
\bm{\eta}^{\sch}_t
=
\gamma
\sum_{s=1}^t
\left(\prod_{j=s+1}^t \Bb_j\right)
(z_s-p)\Hb\bm{\delta},
\end{equation}
where the empty product is interpreted as the identity operator.

Therefore,
\begin{align}
\bar{\bm{\eta}}^{\sch}_N
&=
\frac1N\sum_{t=1}^N \bm{\eta}^{\sch}_t \notag\\
&=
\frac{\gamma}{N}
\sum_{t=1}^N
\sum_{s=1}^t
\left(\prod_{j=s+1}^t \Bb_j\right)
(z_s-p)\Hb\bm{\delta} \notag\\
&=
\sum_{s=1}^N (z_s-p)\vb_s,
\label{eq:schedule-weighted-sum}
\end{align}
where
\begin{equation}\label{eq:def-vs-schedule}
\vb_s
:=
\frac{\gamma}{N}
\sum_{t=s}^N
\left(\prod_{j=s+1}^t \Bb_j\right)\Hb\bm{\delta}.
\end{equation}

Condition on the sample sequence $(\xb_1,\dots,\xb_N)$. Then $\vb_1,\dots,\vb_N$ are
deterministic vectors. Since the feature sequence is independent of the fixed-budget schedule,
the indicators $(z_1,\dots,z_N)$ still follow the same fixed-budget law conditional on
$(\xb_1,\dots,\xb_N)$.

Under the fixed-budget model, for every $s$,
\[
\EE[z_s]=p,
\qquad
\EE[(z_s-p)^2]=p(1-p),
\]
and for every $s\neq r$,
\[
\EE[(z_s-p)(z_r-p)]
=
-\frac{p(1-p)}{N-1}.
\]
Hence, conditioned on $(\xb_1,\dots,\xb_N)$,
\begin{align}
\EE\!\left[
\left\|
\sum_{s=1}^N (z_s-p)\vb_s
\right\|_{\Hb}^2
\,\middle|\,
\xb_1,\dots,\xb_N
\right]
&=
\sum_{s=1}^N\sum_{r=1}^N
\EE[(z_s-p)(z_r-p)]
\langle \vb_s,\vb_r\rangle_{\Hb}
\notag\\
&=
p(1-p)\sum_{s=1}^N\|\vb_s\|_{\Hb}^2
-
\frac{p(1-p)}{N-1}\sum_{s\neq r}\langle \vb_s,\vb_r\rangle_{\Hb}.
\label{eq:conditional-fixed-budget-expand}
\end{align}
Let
\[
\bar \vb:=\frac1N\sum_{s=1}^N \vb_s.
\]
Using the identity
\[
\sum_{s=1}^N\|\vb_s-\bar\vb\|_{\Hb}^2
=
\sum_{s=1}^N\|\vb_s\|_{\Hb}^2
-\frac1N\left\|\sum_{s=1}^N \vb_s\right\|_{\Hb}^2,
\]
one checks that the right-hand side of \eqref{eq:conditional-fixed-budget-expand} equals
\[
\frac{Np(1-p)}{N-1}\sum_{s=1}^N \|\vb_s-\bar\vb\|_{\Hb}^2
\le
\frac{Np(1-p)}{N-1}\sum_{s=1}^N \|\vb_s\|_{\Hb}^2.
\]
Combining this with \eqref{eq:schedule-weighted-sum}, and taking expectation over the
samples, we obtain
\begin{equation}\label{eq:schedule-after-conditional}
\EE\!\left[\|\bar{\bm{\eta}}^{\sch}_N\|_{\Hb}^2\right]
\le
\frac{Np(1-p)}{N-1}
\sum_{s=1}^N \EE\!\left[\|\vb_s\|_{\Hb}^2\right].
\end{equation}

Fix $s\in\{1,\dots,N\}$ and let
$
n_s:=N-s+1.
$
Define an auxiliary bias-only process initialized at $\gamma\Hb\bm{\delta}$:
\[
\ub^{(s)}_0:=\gamma\Hb\bm{\delta},
\qquad
\ub^{(s)}_r:=\Bb_{s+r}\ub^{(s)}_{r-1},
\qquad
r=1,\dots,n_s-1.
\]
Then
\[
\ub^{(s)}_r
=
\left(\prod_{j=s+1}^{s+r}\Bb_j\right)\gamma\Hb\bm{\delta}.
\]
Therefore,
\begin{align}
\vb_s
&=
\frac{\gamma}{N}
\sum_{t=s}^N
\left(\prod_{j=s+1}^t \Bb_j\right)\Hb\bm{\delta}
\notag\\
&=
\frac1N\sum_{r=0}^{n_s-1}\ub^{(s)}_r
=
\frac{n_s}{N}\,\bar{\ub}^{(s)}_{n_s},
\label{eq:vs-bias-equivalence}
\end{align}
where
\[
\bar{\ub}^{(s)}_{n_s}
:=
\frac1{n_s}\sum_{r=0}^{n_s-1}\ub^{(s)}_r.
\]
Since the samples are i.i.d., the distribution of
\[
(\ub^{(s)}_0,\dots,\ub^{(s)}_{n_s-1})
\]
coincides with the distribution of a length-$n_s$ bias-only SGD trajectory with constant
stepsize $\gamma$ and initialization $\gamma\Hb\bm{\delta}$. Hence
\begin{equation}\label{eq:vs-bias-moment}
\EE[\|\vb_s\|_{\Hb}^2]
=
\left(\frac{n_s}{N}\right)^2
\EE\!\left[\|\bar{\ub}_{n_s}(\gamma\Hb\bm{\delta})\|_{\Hb}^2\right],
\end{equation}
where $\bar{\ub}_{n}(\gamma\Hb\bm{\delta})$ denotes the averaged bias-only SGD iterate of
horizon $n$ initialized at $\gamma\Hb\bm{\delta}$.

Substituting \eqref{eq:vs-bias-moment} into \eqref{eq:schedule-after-conditional}, and
changing variables from $s$ to $n=n_s$, yields
\begin{equation}\label{eq:schedule-reduction-to-bias}
\EE\!\left[\|\bar{\bm{\eta}}^{\sch}_N\|_{\Hb}^2\right]
\le
\frac{Np(1-p)}{N-1}
\sum_{n=1}^N
\left(\frac{n}{N}\right)^2
\EE\!\left[\|\bar{\ub}_{n}(\gamma\Hb\bm{\delta})\|_{\Hb}^2\right].
\end{equation}

For each $n\in\{1,\dots,N\}$, let
\[
k_n:=\max\left\{k:\lambda_k\ge \frac1{\gamma n}\right\},
\qquad
D_n:=k_n+n^2\gamma^2\sum_{i>k_n}\lambda_i^2.
\]
Applying Lemma~\ref{lemma:HB-constant-upper-bound} with horizon $n$ and initialization
$
\ub_0=\gamma\Hb\bm{\delta},
$
we obtain
\begin{align}
\frac12
\EE\!\left[\|\bar{\ub}_{n}(\gamma\Hb\bm{\delta})\|_{\Hb}^2\right]
\le\;&
\frac{1}{\gamma^2 n^2}
\|\gamma\Hb\bm{\delta}\|_{\Hb^{-1}_{0:k_n}}^2
+
\|\gamma\Hb\bm{\delta}\|_{\Hb_{k_n:\infty}}^2
\notag\\
&\quad+
\frac{2\alpha
\bigl(
\|\gamma\Hb\bm{\delta}\|_{\Ib_{0:k_n}}^2
+
n\gamma\|\gamma\Hb\bm{\delta}\|_{\Hb_{k_n:\infty}}^2
\bigr)}
{n\gamma(1-\gamma\alpha\tr(\Hb))}
\cdot
\frac{D_n}{n}.
\label{eq:bias-bound-at-n}
\end{align}

We now simplify each term in the eigenbasis of $\Hb$. Writing
\[
\bm{\delta}=\sum_i \delta_i \vb_i,
\qquad
\gamma\Hb\bm{\delta}
=
\gamma\sum_i \lambda_i\delta_i \vb_i,
\]
we have
\begin{align}
\frac{1}{\gamma^2 n^2}
\|\gamma\Hb\bm{\delta}\|_{\Hb^{-1}_{0:k_n}}^2
&=
\frac{1}{n^2}\sum_{i\le k_n}\lambda_i\delta_i^2,
\label{eq:term1-simplify}
\\
\|\gamma\Hb\bm{\delta}\|_{\Hb_{k_n:\infty}}^2
&=
\gamma^2\sum_{i>k_n}\lambda_i^3\delta_i^2,
\label{eq:term2-simplify}
\\
\|\gamma\Hb\bm{\delta}\|_{\Ib_{0:k_n}}^2
&=
\gamma^2\sum_{i\le k_n}\lambda_i^2\delta_i^2.
\label{eq:term3-simplify}
\end{align}

For the first two terms, since $\lambda_i<1/(\gamma n)$ for every $i>k_n$,
\[
\gamma^2\lambda_i^3
\le
\frac{\lambda_i}{n^2}.
\]
Therefore, by \eqref{eq:term1-simplify} and \eqref{eq:term2-simplify},
\begin{align}
\frac{1}{\gamma^2 n^2}
\|\gamma\Hb\bm{\delta}\|_{\Hb^{-1}_{0:k_n}}^2
+
\|\gamma\Hb\bm{\delta}\|_{\Hb_{k_n:\infty}}^2
&\le
\frac{1}{n^2}
\sum_{i\le k_n}\lambda_i\delta_i^2
+
\frac{1}{n^2}
\sum_{i>k_n}\lambda_i\delta_i^2
\notag\\
&=
\frac{\|\bm{\delta}\|_{\Hb}^2}{n^2}.
\label{eq:first-two-terms-bound}
\end{align}

For the third term, using again $\lambda_i<1/(\gamma n)$ for $i>k_n$, we have
\[
n\gamma \lambda_i^3\le \lambda_i^2.
\]
Hence
\begin{align}
\|\gamma\Hb\bm{\delta}\|_{\Ib_{0:k_n}}^2
+
n\gamma\|\gamma\Hb\bm{\delta}\|_{\Hb_{k_n:\infty}}^2
&=
\gamma^2\sum_{i\le k_n}\lambda_i^2\delta_i^2
+
n\gamma^3\sum_{i>k_n}\lambda_i^3\delta_i^2
\notag\\
&\le
\gamma^2\sum_i \lambda_i^2\delta_i^2
\notag\\
&\le
\gamma^2\tr(\Hb)\sum_i \lambda_i\delta_i^2
=
\gamma^2\tr(\Hb)\|\bm{\delta}\|_{\Hb}^2.
\label{eq:third-numerator-bound}
\end{align}
Substituting \eqref{eq:first-two-terms-bound} and \eqref{eq:third-numerator-bound} into
\eqref{eq:bias-bound-at-n}, we obtain
\begin{equation}\label{eq:bias-bound-clean-at-n}
\frac12
\EE\!\left[\|\bar{\ub}_{n}(\gamma\Hb\bm{\delta})\|_{\Hb}^2\right]
\le
\frac{\|\bm{\delta}\|_{\Hb}^2}{n^2}
\left(
1+
\frac{2\alpha\gamma\tr(\Hb)}{1-\gamma\alpha\tr(\Hb)}\,D_n
\right).
\end{equation}

Substituting \eqref{eq:bias-bound-clean-at-n} into \eqref{eq:schedule-reduction-to-bias}
yields
\begin{align}
\frac12
\EE\!\left[\|\bar{\bm{\eta}}^{\sch}_N\|_{\Hb}^2\right]
&\le
\frac{Np(1-p)}{N-1}
\sum_{n=1}^N
\left(\frac{n}{N}\right)^2
\cdot
\frac{\|\bm{\delta}\|_{\Hb}^2}{n^2}
\left(
1+
\frac{2\alpha\gamma\tr(\Hb)}{1-\gamma\alpha\tr(\Hb)}\,D_n
\right)
\notag\\
&=
\frac{p(1-p)\|\bm{\delta}\|_{\Hb}^2}{N(N-1)}
\sum_{n=1}^N
\left(
1+
\frac{2\alpha\gamma\tr(\Hb)}{1-\gamma\alpha\tr(\Hb)}\,D_n
\right).
\label{eq:before-final-Deff}
\end{align}

It remains to compare $D_n$ with the final effective dimension $\DIM$. Observe that
\[
D_n
=
k_n+n^2\gamma^2\sum_{i>k_n}\lambda_i^2
=
\sum_i \min\{1,n^2\gamma^2\lambda_i^2\}.
\]
Since $n\le N$,
\[
D_n
\le
\sum_i \min\{1,N^2\gamma^2\lambda_i^2\}.
\]
Now let
\[
k^*:=\max\left\{k:\lambda_k\ge \frac1{\gamma N}\right\}.
\]
Then
\[
\sum_i \min\{1,N^2\gamma^2\lambda_i^2\}
\le
k^*+\gamma^2N^2\sum_{i>k^*}\lambda_i^2
=
\DIM.
\]
Hence
\[
D_n\le \DIM,
\qquad \forall n\le N.
\]
Applying this to \eqref{eq:before-final-Deff} gives
\begin{align*}
\frac12
\EE\!\left[\|\bar{\bm{\eta}}^{\sch}_N\|_{\Hb}^2\right]
&\le
\frac{p(1-p)\|\bm{\delta}\|_{\Hb}^2}{N(N-1)}
\sum_{n=1}^N
\left(
1+
\frac{2\alpha\gamma\tr(\Hb)}{1-\gamma\alpha\tr(\Hb)}\,\DIM
\right) \\
&=
\frac{p(1-p)\|\bm{\delta}\|_{\Hb}^2}{N-1}
\left(
1+
\frac{2\alpha\gamma\tr(\Hb)}{1-\gamma\alpha\tr(\Hb)}\,\DIM
\right),
\end{align*}
which proves \eqref{eq:schedule-fluctuation-clean-bound}.
\end{proof}

\paragraph{An upper bound for excess risk.}
Combining the bounds for the bias, variance, drift, and fluctuation terms derived above,
we obtain the following upper bound on the excess risk.
\begin{theorem}[Excess risk upper bound for constant-stepsize SGD with iterate averaging]
\label{thm:avg-sgd-upper-bound}
Suppose Assumptions \ref{assump:second-moment}, \ref{assump:fourth-moment}, and \ref{assump:noise} hold. Let $\bar\sigma^2 := (1-p)\sigma_1^2 + p\sigma_2^2.$ 
Assume $\gamma < 1/(\alpha \tr(\Hb))$. Then
\begin{align*}
\EE [ \R(\bar\wb_N) - \R(\wb_1^*) ]
&\lesssim
\frac{1}{\gamma^2 N^2}\|\wb_0-\wb_1^*\|_{\Hb_{0:k^*}^{-1}}^2
+\|\wb_0-\wb_1^*\|_{\Hb_{k^*:\infty}}^{2}\\
&+\big(\alpha\|\wb_0-\wb_1^*\|_{\Hb}^{2}+\bar\sigma^2+\alpha p\|\bm{\delta}\|_{\Hb}^2\big)
\cdot\frac{\DIM}{N}+ \frac{p(1-p)\|\bm{\delta}\|_{\Hb}^2}{N},\\
&+
p^2
\left\|
\left(
\Ib-\frac{1}{\gamma N}\Hb^{-1}\bigl(\Ib-(\Ib-\gamma\Hb)^N\bigr)
\right)\bm{\delta}
\right\|_{\Hb}^2
\end{align*}
where
$
k^* = \max\left\{k:\lambda_k \ge \frac{1}{\gamma N}\right\},
$
and
$
\DIM := k^* + \gamma^2 N^2 \sum_{i>k^*}\lambda_i^2.
$
\end{theorem}

\subsection{Lower bound analysis}
The starting point is again the error recursion
\begin{equation}
\bm{\eta}_t
=
(\Ib-\gamma\xb_t\xb_t^\top)\bm{\eta}_{t-1}
+
\gamma z_t \xb_t\xb_t^\top\bm{\delta}
+
\gamma\xi_t\xb_t,
\qquad
\bm{\eta}_0=\wb_0-\wb_1^*,
\label{eq:lower-error-recursion-avg}
\end{equation}
where
\[
\bm{\eta}_t=\wb_t-\wb_1^*.
\]
Recall that the averaged iterate is
\[
\bar\wb_N:=\frac1N\sum_{t=0}^{N-1}\wb_t,
\]
and correspondingly
\[
\bar{\bm{\eta}}_N
:=
\frac1N\sum_{t=0}^{N-1}\bm{\eta}_t
=
\bar\wb_N-\wb_1^*.
\]
By definition of the excess risk,
\[
\EE[\R(\bar\wb_N)-\R(\wb_1^*)]
=
\frac12\,\big\langle \Hb,\EE[\bar{\bm{\eta}}_N\otimes\bar{\bm{\eta}}_N]\big\rangle.
\]

To derive a lower bound, we center the averaged error around its mean. Define
\[
\bar{\bm{\mu}}_N:=\EE[\bar{\bm{\eta}}_N],
\qquad
\widetilde{\bar{\bm{\eta}}}_N:=\bar{\bm{\eta}}_N-\bar{\bm{\mu}}_N.
\]
Then
\[
\bar{\bm{\eta}}_N=\bar{\bm{\mu}}_N+\widetilde{\bar{\bm{\eta}}}_N,
\qquad
\EE[\widetilde{\bar{\bm{\eta}}}_N]=\boldsymbol{0}.
\]
Therefore,
\begin{align*}
\EE[\bar{\bm{\eta}}_N\otimes\bar{\bm{\eta}}_N]
&=
\EE\Big[
(\bar{\bm{\mu}}_N+\widetilde{\bar{\bm{\eta}}}_N)
\otimes
(\bar{\bm{\mu}}_N+\widetilde{\bar{\bm{\eta}}}_N)
\Big]
\\
&=
\bar{\bm{\mu}}_N\otimes\bar{\bm{\mu}}_N
+
\EE[\widetilde{\bar{\bm{\eta}}}_N\otimes\widetilde{\bar{\bm{\eta}}}_N]
+
\bar{\bm{\mu}}_N\otimes \EE[\widetilde{\bar{\bm{\eta}}}_N]
+
\EE[\widetilde{\bar{\bm{\eta}}}_N]\otimes \bar{\bm{\mu}}_N
\\
&=
\bar{\bm{\mu}}_N\otimes\bar{\bm{\mu}}_N
+
\EE[\widetilde{\bar{\bm{\eta}}}_N\otimes\widetilde{\bar{\bm{\eta}}}_N].
\end{align*}
Hence
\begin{equation}
\EE[\R(\bar\wb_N)-\R(\wb_1^*)]
=
\frac12\|\bar{\bm{\mu}}_N\|_{\Hb}^2
+
\frac12\big\langle \Hb,\EE[\widetilde{\bar{\bm{\eta}}}_N\otimes\widetilde{\bar{\bm{\eta}}}_N]\big\rangle.
\label{eq:lower-bound-by-mean-avg}
\end{equation}
Since the second term is nonnegative, we immediately obtain
\begin{equation}
\EE[\R(\bar\wb_N)-\R(\wb_1^*)]
\ge
\frac12\|\bar{\bm{\mu}}_N\|_{\Hb}^2.
\label{eq:lower-bound-mean-only-avg}
\end{equation}

We now characterize the mean trajectory of the SGD error process. Let
$
\bm{\mu}_t:=\EE[\bm{\eta}_t].
$
Taking expectation on both sides of \eqref{eq:lower-error-recursion-avg}, and using
\[
\EE[\xb_t\xb_t^\top]=\Hb,
\qquad
\EE[z_t]=p,
\qquad
\EE[\xi_t\xb_t]=\boldsymbol{0},
\]
we obtain
\begin{equation*}
\bm{\mu}_t
=
(\Ib-\gamma\Hb)\bm{\mu}_{t-1}
+
\gamma p\Hb\bm{\delta},
\qquad
\bm{\mu}_0=\wb_0-\wb_1^*.
\label{eq:mean-recursion-lower-avg}
\end{equation*}

Unrolling the recursion gives
\begin{align}
\bm{\mu}_t
&=
(\Ib-\gamma\Hb)^t(\wb_0-\wb_1^*)
+
p\sum_{s=1}^t
\gamma(\Ib-\gamma\Hb)^{t-s}\Hb\bm{\delta}.
\label{eq:mean-unrolled-lower-avg}
\end{align}
Since \((\Ib-\gamma\Hb)\) is a polynomial in \(\Hb\), it commutes with \(\Hb\). Therefore,
\[
\sum_{s=1}^t
\gamma(\Ib-\gamma\Hb)^{t-s}\Hb
=
\Ib-(\Ib-\gamma\Hb)^t.
\]
Substituting this identity into \eqref{eq:mean-unrolled-lower-avg}, we obtain the closed form
\begin{equation*}
\bm{\mu}_t
=
(\Ib-\gamma\Hb)^t(\wb_0-\wb_1^*)
+
p\bigl(\Ib-(\Ib-\gamma\Hb)^t\bigr)\bm{\delta}.
\label{eq:mean-closed-form-lower-avg}
\end{equation*}

Now average over $t=0,\dots,N-1$. Since
\[
\bar{\bm{\mu}}_N
=
\EE[\bar{\bm{\eta}}_N]
=
\frac1N\sum_{t=0}^{N-1}\bm{\mu}_t,
\]
we have
\begin{align*}
\bar{\bm{\mu}}_N
&=
\frac1N\sum_{t=0}^{N-1}(\Ib-\gamma\Hb)^t(\wb_0-\wb_1^*)
+
\frac{p}{N}\sum_{t=0}^{N-1}\bigl(\Ib-(\Ib-\gamma\Hb)^t\bigr)\bm{\delta}.
\label{eq:avg-mean-before-sum}
\end{align*}
Using the matrix geometric-series identity
\[
\sum_{t=0}^{N-1}(\Ib-\gamma\Hb)^t
=
(\gamma\Hb)^{-1}\bigl(\Ib-(\Ib-\gamma\Hb)^N\bigr),
\]
we obtain
\begin{equation*}
\bar{\bm{\mu}}_N
=
\frac{1}{\gamma N}\Hb^{-1}\bigl(\Ib-(\Ib-\gamma\Hb)^N\bigr)(\wb_0-\wb_1^*)
+
p\left(
\Ib-\frac{1}{\gamma N}\Hb^{-1}\bigl(\Ib-(\Ib-\gamma\Hb)^N\bigr)
\right)\bm{\delta}.
\label{eq:avg-mean-closed-form}
\end{equation*}

Consequently,
\begin{equation}
\|\bar{\bm{\mu}}_N\|_{\Hb}^2
=
\left\|
\frac{1}{\gamma N}\Hb^{-1}\bigl(\Ib-(\Ib-\gamma\Hb)^N\bigr)(\wb_0-\wb_1^*)
+
p\left(
\Ib-\frac{1}{\gamma N}\Hb^{-1}\bigl(\Ib-(\Ib-\gamma\Hb)^N\bigr)
\right)\bm{\delta}
\right\|_{\Hb}^2.
\label{eq:avg-mean-risk-term}
\end{equation}

Combining \eqref{eq:lower-bound-mean-only-avg} and \eqref{eq:avg-mean-risk-term}, we arrive at the lower bound
\begin{align*}
\EE [ \R(\bar\wb_N) - \R(\wb_1^*) ]
&\gtrsim
\left\|
\frac{1}{\gamma N}\Hb^{-1}\bigl(\Ib-(\Ib-\gamma\Hb)^N\bigr)(\wb_0-\wb_1^*)
+
p\left(
\Ib-\frac{1}{\gamma N}\Hb^{-1}\bigl(\Ib-(\Ib-\gamma\Hb)^N\bigr)
\right)\bm{\delta}
\right\|_{\Hb}^2.
\end{align*}

\subsection{Strong model collapse behavior}
\begin{corollary}[Strong model collapse in Avg-SGD]\label{cor:strong-model-collapse-avg}
Consider averaged SGD with constant stepsize $\gamma$. Suppose Assumptions~\ref{assump:second-moment}, ~\ref{assump:fourth-moment} and ~\ref{assump:noise} hold. Suppose $\gamma < 1 / \alpha\tr(\Hb)$ and $\DIM=o(M+N)$. Further assume that $\|\wb_0-\wb_1^*\|_{2}^2$ is finite. As the total sample size $M+N$ scales to infinity with a fixed synthetic proportion $p$, 
\[
\lim_{M+N\xrightarrow{}\infty}\EE[\mathcal{E}_1(\bar\wb_{M+N})] \eqsim p^2\|\bm{\delta}\|_{\Hb}^2.
\]
\end{corollary}
\begin{proof}
Let
\[
T:=M+N,
\qquad
\Ab_T:=\frac{1}{\gamma T}\Hb^{-1}\bigl(\Ib-(\Ib-\gamma\Hb)^T\bigr).
\]
Since the synthetic proportion $p=M/(M+N)$ is fixed, it suffices to study the limit as
$T\to\infty$.

We first show that $\Ab_T$ vanishes in $\Hb$-norm. Writing the eigendecomposition
\[
\Hb=\sum_i \lambda_i \vb_i\vb_i^\top,
\]
we have
\[
\Ab_T \vb_i
=
\frac{1-(1-\gamma\lambda_i)^T}{\gamma T\lambda_i}\,\vb_i.
\]
Since $\gamma<1/(\alpha\tr(\Hb))$, we have $0\le \gamma\lambda_i\le \gamma\tr(\Hb)<1$, hence
\[
0\le 1-\gamma\lambda_i<1.
\]
For every $i$ with $\lambda_i>0$, it follows that
\[
\frac{1-(1-\gamma\lambda_i)^T}{\gamma T\lambda_i}\longrightarrow 0
\qquad\text{as }T\to\infty.
\]
Moreover, using the elementary inequality
\[
1-(1-a)^T\le Ta,\qquad 0\le a\le 1,
\]
we get
\[
0\le \frac{1-(1-\gamma\lambda_i)^T}{\gamma T\lambda_i}\le 1.
\]
Therefore, for any $\wb$ with $\|\wb\|_{\Hb}<\infty$,
\[
\|\Ab_T\wb\|_{\Hb}^2
=
\sum_i
\lambda_i
\left(
\frac{1-(1-\gamma\lambda_i)^T}{\gamma T\lambda_i}
\right)^2
\langle \wb,\vb_i\rangle^2
\longrightarrow 0
\]
by dominated convergence.

We now turn to the upper bound. By the upper bound theorem for averaged SGD,
\begin{align*}
\EE[\mathcal E_1(\bar\wb_T)]
&\lesssim
\underbrace{
\frac{1}{\gamma^2T^2}\|\wb_0-\wb_1^*\|_{\Hb_{0:k^*}^{-1}}^2
}_{(I)}
+
\underbrace{
\|\wb_0-\wb_1^*\|_{\Hb_{k^*:\infty}}^2
}_{(II)}
\\
&\quad+
\underbrace{
\big(\alpha\|\wb_0-\wb_1^*\|_{\Hb}^{2}+\bar\sigma^2+\alpha p\|\bm{\delta}\|_{\Hb}^2\big)
\cdot\frac{\DIM}{T}
}_{(III)}
+
\underbrace{
\frac{p(1-p)\|\bm{\delta}\|_{\Hb}^2}{T}
}_{(IV)}
\\
&\quad+
\underbrace{
p^2\|(\Ib-\Ab_T)\bm{\delta}\|_{\Hb}^2
}_{(V)},
\end{align*}
where
\[
k^*=\max\left\{k:\lambda_k\ge \frac{1}{\gamma T}\right\},
\qquad
\DIM:=k^*+\gamma^2T^2\sum_{i>k^*}\lambda_i^2.
\]

By the assumption $\DIM=o(T)$, terms $(III)$ and $(IV)$ vanish as $T\to\infty$.
For term $(I)$, note that
\[
\frac{1}{\gamma^2T^2}\|\wb_0-\wb_1^*\|_{\Hb_{0:k^*}^{-1}}^2
\le
\|\wb_0-\wb_1^*\|_{2}^2\cdot \frac{1}{\gamma T},
\]
because for every $i\le k^*$ we have $\lambda_i\ge 1/(\gamma T)$, hence
$\lambda_i^{-1}\le \gamma T$. Therefore $(I)\to 0$.
For term $(II)$, since $k^*\to\infty$ as $T\to\infty$ and $\|\wb_0-\wb_1^*\|_{\Hb}<\infty$,
we have
\[
\|\wb_0-\wb_1^*\|_{\Hb_{k^*:\infty}}^2
=
\sum_{i>k^*}\lambda_i\langle \wb_0-\wb_1^*,\vb_i\rangle^2
\longrightarrow 0.
\]
Finally,
\[
\|(\Ib-\Ab_T)\bm{\delta}\|_{\Hb}\to \|\bm{\delta}\|_{\Hb},
\]
because
\[
\|(\Ib-\Ab_T)\bm{\delta}-\bm{\delta}\|_{\Hb}
=
\|\Ab_T\bm{\delta}\|_{\Hb}\to 0.
\]
Therefore,
\[
\limsup_{T\to\infty}\EE[\mathcal E_1(\bar\wb_T)]
\lesssim
p^2\|\bm{\delta}\|_{\Hb}^2.
\]

Next, we apply the lower bound obtained from the mean term:
\[
\EE[\mathcal E_1(\bar\wb_T)]
\gtrsim
\left\|
\Ab_T(\wb_0-\wb_1^*)
+
p(\Ib-\Ab_T)\bm{\delta}
\right\|_{\Hb}^2.
\]
Rewrite the term inside the norm as
\[
\Ab_T(\wb_0-\wb_1^*)
+
p(\Ib-\Ab_T)\bm{\delta}
=
\Ab_T(\wb_0-\wb_1^*-p\bm{\delta})
+
p\bm{\delta}.
\]
Again using $\|\Ab_T\wb\|_{\Hb}\to 0$ for every $\wb$ with finite $\Hb$-norm, we obtain
\[
\left\|
\Ab_T(\wb_0-\wb_1^*-p\bm{\delta})
+
p\bm{\delta}
\right\|_{\Hb}^2
\longrightarrow
p^2\|\bm{\delta}\|_{\Hb}^2.
\]
Hence
\[
\liminf_{T\to\infty}\EE[\mathcal E_1(\bar\wb_T)]
\gtrsim
p^2\|\bm{\delta}\|_{\Hb}^2.
\]

Combining the upper and lower bounds yields
\[
\lim_{M+N\to\infty}\EE[\mathcal E_1(\bar\wb_{M+N})]
\eqsim
p^2\|\bm{\delta}\|_{\Hb}^2.
\]
This proves the corollary.
\end{proof}

\vspace{-0.5em}
\section{Additional Experiments}\label{app:real-data}
\subsection{CIFAR10 CNN Experiments}
\paragraph{CIFAR10 experiment setup.}
We conduct experiments on the $\mathtt{CIFAR10}$~\citep{krizhevsky2009learning} dataset, which consists of 50,000 training images and 10,000 test images across 10 classes. All images are of size $32 \times 32$. All experiments are run on a single NVIDIA RTX 3090 GPU.

We use a CNN with five convolutional layers. Each block consists of a convolution layer followed by batch normalization and ReLU activation. Model size is controlled by a width parameter $D$, which scales the number of channels in all convolutional layers proportionally. Specifically, the channel dimensions follow the pattern $(D, D, 2D, 2D, 4D)$. We vary $D$ in $\{8, 16, 32, 64, 128, 256\}$ to study the effect of model capacity.

Synthetic data shares the same input distribution as the real data, but labels are generated by a weak teacher model. The teacher is a smaller CNN (width $16$) trained on a random subset of 5,000 training samples for 5 epochs. For each input $x$, the teacher produces a soft label distribution $q_{\text{teacher}}(y \mid x)$ with temperature scaling. The synthetic label is then defined as
$
\tilde{y} =  q_{\text{teacher}}(y \mid x),
$
, corresponding to fully synthetic labels.

Similar to the simulation experiments, we consider two training protocols as described in Section~\ref{sec:setup}.
For a fixed total training budget $T$, we allocate $pT$ synthetic samples and $(1-p)T$ real samples, where $p \in [0,1]$ is the synthetic data proportion. We vary $p$ over $\{0.0, 0.1, \dots, 0.9\}$ and $T$ over a grid ranging from $200$ to $20{,}000$. To examine how the synthetic-induced floor depends on model size, we fix $p=0.9$ and vary the model width $D$ (while also varying $T$ across the same grid).

\paragraph{Evaluation.}
We evaluate performance on the real test set using both classification error and cross-entropy loss. Each experiment is repeated with 3 random seeds, and we report the mean and standard deviation. To isolate the effect of synthetic data from baseline model performance, we report the metrics difference relative to $p=0$ under the same $T$:
\[
\Delta_{error} = Error(p=0.9) - Error(p=0), \quad \Delta_{loss} = Loss(p=0.9) - Loss(p=0),
\]
which removes the intrinsic advantage of larger models and highlights the synthetic-data-induced degradation.

\begin{figure}[H]
    \centering
    \includegraphics[width=0.8\textwidth]{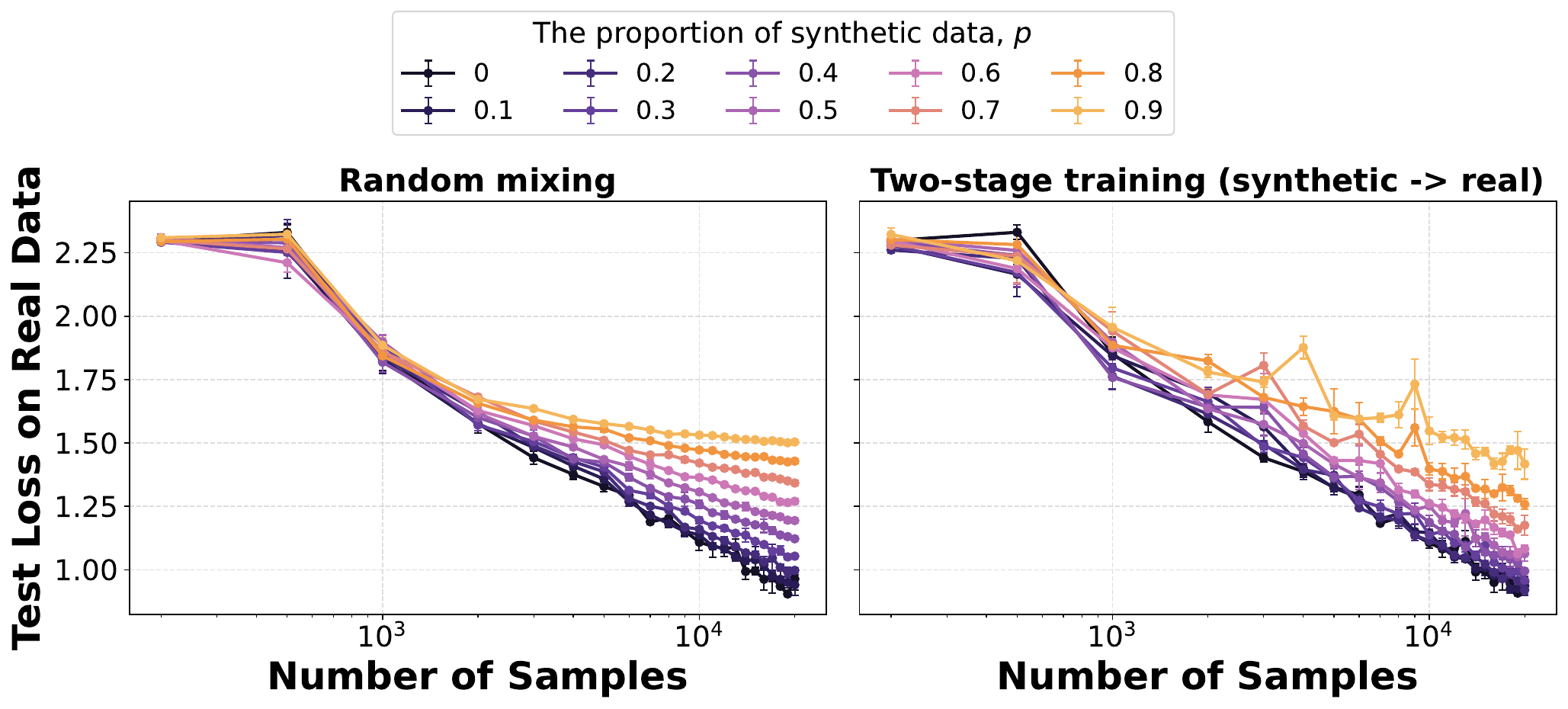}
    \caption{
\textbf{Real-data experiments of different training protocols.}
Test cross-entropy loss on the real distribution as a function of the number of training samples $T$, under different synthetic proportions $p$. 
\textbf{Left:} random mixing of real and synthetic data. 
\textbf{Right:} two-stage training (synthetic-data training followed by real-data training).
}
    \label{fig:real-data-mix_vs_two-stage}
\end{figure}

From Figure~\ref{fig:real-data-mix_vs_two-stage}, we observe the same qualitative behavior predicted by our theory. 
Under random mixing, the test loss quickly saturates as $T$ increases, forming a clear floor that depends on the synthetic proportion $p$. 
In particular, larger values of $p$ lead to systematically higher asymptotic error, indicating that synthetic data introduces an irreducible bias that cannot be removed by additional samples. In contrast, two-stage training mitigates this phenomenon: the test loss continues to decrease with $T$, and consistently achieves slightly better performance than mixing in the large-sample regime. 
These results confirm that the synthetic-data-induced floor is a consequence of mixing, rather than an inherent limitation of the data itself.
\begin{figure}[H]
\centering
\begin{subfigure}[t]{0.42\linewidth}
    \centering
    \includegraphics[width=\linewidth]{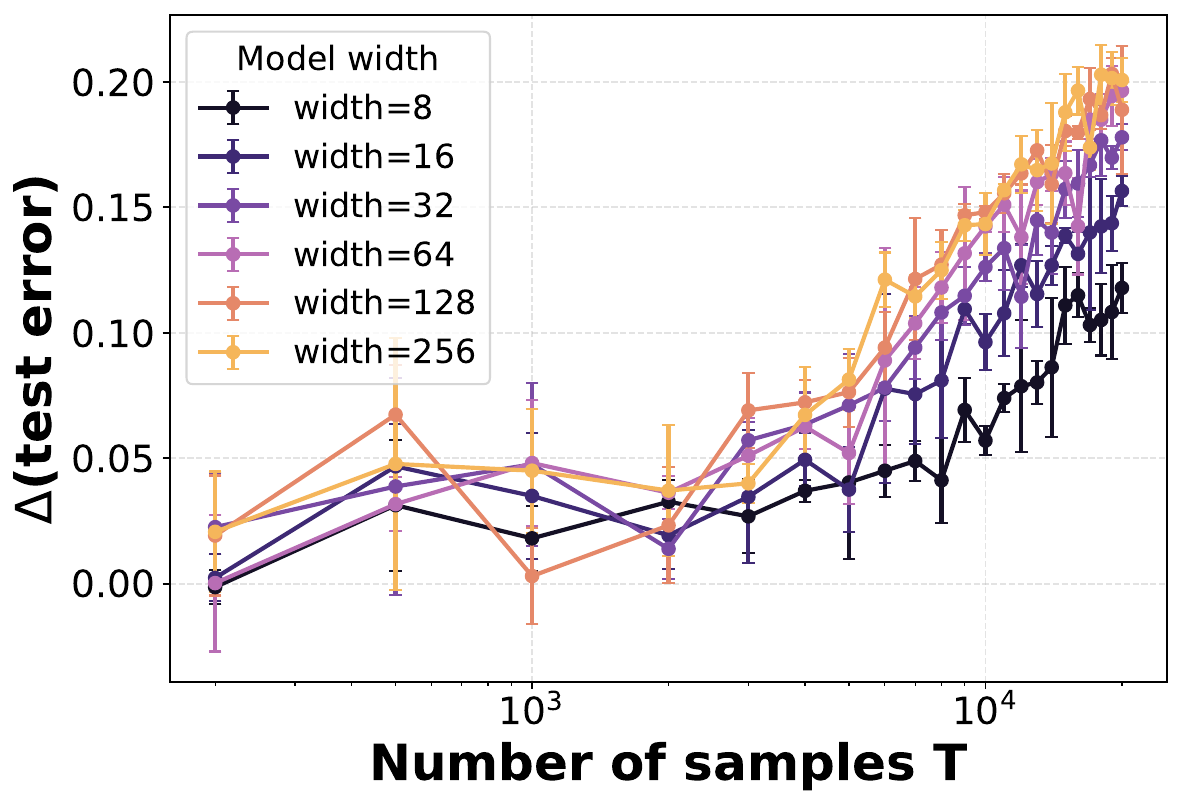}
    \caption{Test error difference}
\end{subfigure}
\begin{subfigure}[t]{0.42\linewidth}
    \centering
    \includegraphics[width=\linewidth]{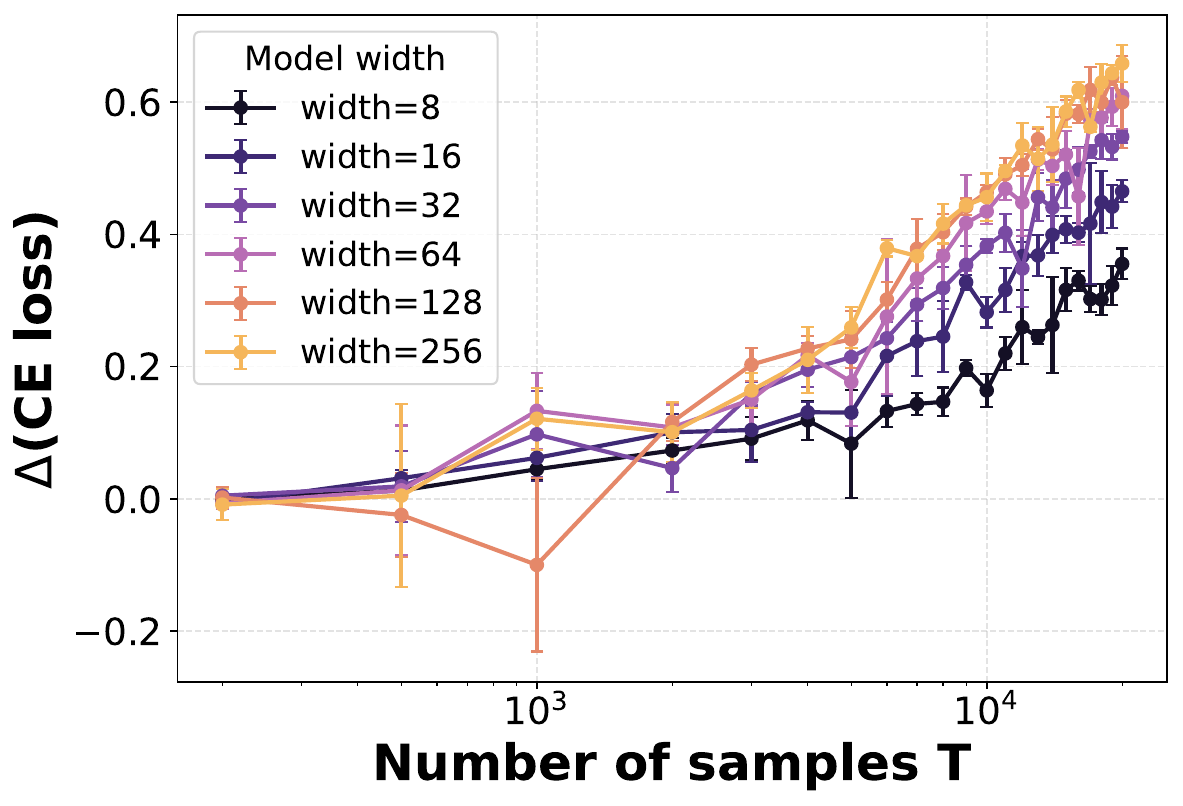}
    \caption{Test cross-entropy difference}
\end{subfigure}

\caption{
\textbf{Effect of model size on synthetic-data-induced degradation.}
We plot the performance gap between high synthetic proportion ($p=0.9$) and pure real-data training ($p=0$), as a function of the number of samples $T$, for different model widths. Across both test error and cross-entropy loss, the gap consistently increases with model size, especially in the large-sample regime. This indicates that larger models are more sensitive to synthetic data, leading to a stronger irreducible bias under mixing.
}
\label{fig:model_size_effect_real}
\end{figure}

Figure~\ref{fig:model_size_effect_real} shows how the impact of synthetic data varies with model size on real data. 
We measure the performance gap between high synthetic proportion ($p=0.9$) and pure real-data training ($p=0$), thereby isolating the effect of synthetic data.
We observe a clear and consistent trend: as model size increases, the gap becomes significantly larger, particularly at large sample sizes. 
While all models exhibit some degradation due to synthetic data, larger models suffer substantially more, indicating that the synthetic-data-induced bias is amplified with model capacity. 
This behavior aligns with our theoretical predictions, suggesting that the floor effect induced by mixing is not only persistent but also worsens with increasing model size.
\vspace{-0.5em}
\subsection{Language Model Experiments}
\paragraph{Experiment setup.}
We conduct language-modeling experiments on WikiText-103, tokenized with the GPT-2 tokenizer. All models are decoder-only Transformers trained with context length 256 using next-token cross-entropy loss, and evaluated by validation loss on the real validation set. Unless otherwise stated, we use batch size 64 and AdamW with learning rate $6\times 10^{-4}$, weight decay $0.1$, training on a single NVIDIA RTX 3090 GPU.

Synthetic text is generated from a DistilGPT-2 teacher. For the synthetic corpus, we sample continuations of length 256 from real WikiText-103 prompts of length 32, using temperature $1.5$, top-$p=0.95$, top-$k=0$, and store only the generated continuations.

We first compare the same two training protocols as before. In \emph{mixed} training, each step uses real or synthetic data so that the cumulative synthetic fraction tracks $p$. In \emph{two-stage} training, for each total budget $T$, the model is first trained for approximately $pT$ synthetic steps and then for the remaining real steps. Our main experiment uses a 4-layer Transformer with 4 attention heads and embedding dimension 256, containing 16.1M parameters. We vary $p\in\{0,0.3,0.6,0.9\}$ and evaluate 12 logarithmically spaced budgets from 100 to 24,000 steps, corresponding to up to 393M trained tokens. Results are shown in Figure~\ref{fig:lm_protocol_comparison}.

\begin{figure}[H]
    \centering
    \includegraphics[width=0.9\linewidth]{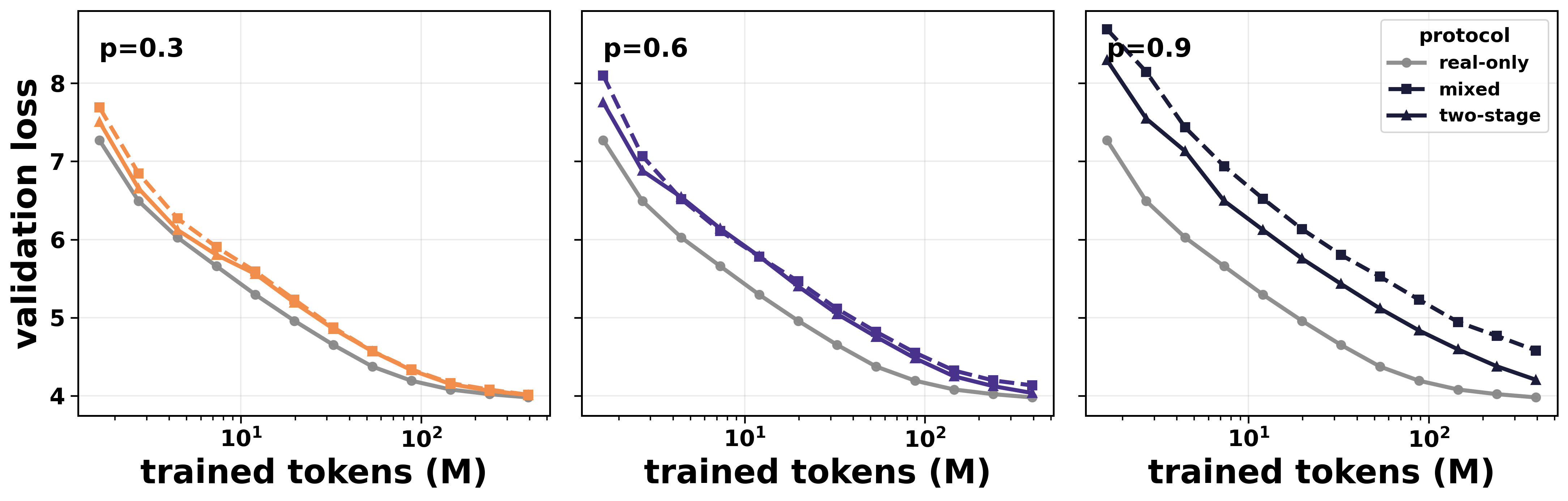}
    \caption{\textbf{Language-model protocol comparison.}
    We compare mixed training and two-stage training on WikiText-103 for synthetic fractions $p\in\{0.3,0.6,0.9\}$, with the real-only baseline shown in gray. Under mixed training, validation loss remains consistently worse than the real-only curve, especially at large $p$, indicating a persistent degradation from training on a fixed fraction of synthetic data. Two-stage training reduces this degradation and approaches the real-only curve more closely at large budgets. This matches the theoretical prediction that mixed training induces a non-vanishing synthetic-data bias floor, whereas using synthetic data only as an initialization stage can mitigate this floor.}
    \label{fig:lm_protocol_comparison}
\end{figure}
\vspace{-1em}
To study the effect of model capacity, we fix a high synthetic fraction $p=0.9$ under the mixed protocol and compare against the matched real-only baseline $p=0$. We train three Transformer models with sizes 4/256/4 (16.1M params), 6/384/6 (30.0M params), and 8/512/8 (51.1M params), where the three numbers denote layers, embedding dimension, and attention heads. We evaluate budgets from 3,000 to 48,000 steps with batch size 32, corresponding to 24.6M to 393.2M trained tokens. We report the degradation
$\Delta \mathrm{loss}=\mathrm{ValLoss}(p=0.9)-\mathrm{ValLoss}(p=0),
$
which isolates the synthetic-data-induced effect from the baseline improvement due to model size. Results are shown in Figure~\ref{fig:lm_model_size}.

\begin{figure}[H]
    \centering
    \includegraphics[width=0.42\linewidth]{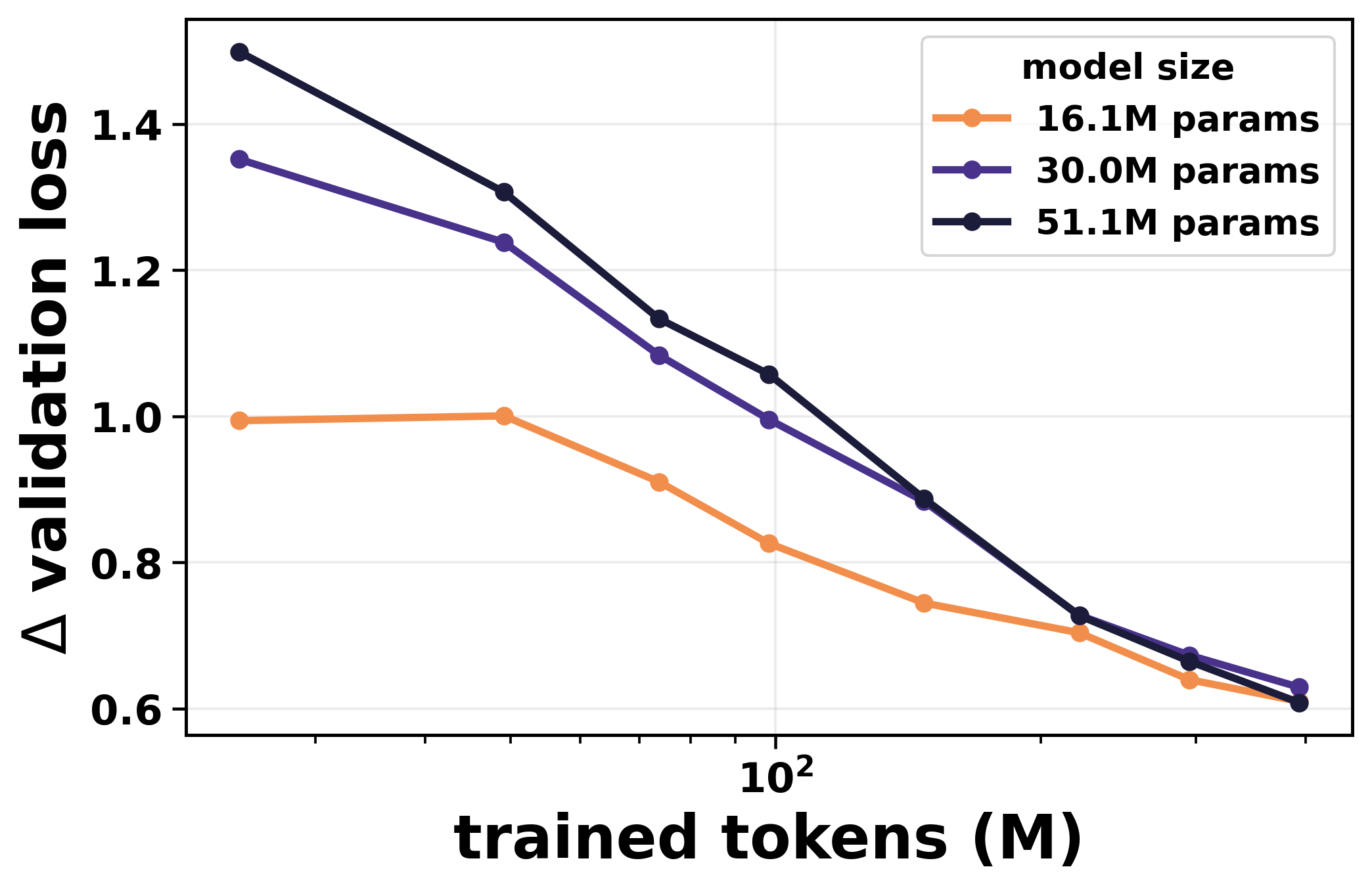}
    \caption{\textbf{Effect of model size on synthetic-data-induced degradation.}
    We plot the validation-loss gap $\Delta\mathrm{loss}=\mathrm{ValLoss}(p=0.9)-\mathrm{ValLoss}(p=0)$ for models with 16.1M, 30.0M, and 51.1M parameters. The larger models exhibit a larger degradation gap at small and intermediate budgets, showing that the effect of synthetic mismatch is amplified by model capacity. This supports the theoretical prediction that higher-capacity models can fit the biased synthetic signal more strongly, leading to a larger synthetic-induced degradation under mixed training.}
    \label{fig:lm_model_size}
\end{figure}

Finally, we test how teacher quality affects the usefulness of synthetic pretraining under a fixed real-data budget. We use the same 4-layer student and compare three synthetic teachers: a high-quality DistilGPT-2 corpus sampled with temperature $0.9$, top-$p=0.95$, top-$k=0$, and prompt length 32; the medium-quality corpus from the main experiment; and a low-quality corpus sampled with temperature $4.0$, top-$p=1.0$, top-$k=0$, and prompt length 4. We pretrain on synthetic budgets of approximately 25M, 50M, 100M, and 200M tokens, then reset the optimizer and train on real-data budgets up to 50M tokens. Performance is evaluated by real validation loss and by the gap relative to the real-only curve at the same real-token budget. Figure~\ref{fig:lm_teacher_quality}.

\begin{figure}[H]
    \centering
    \begin{subfigure}{0.85\linewidth}
        \centering
        \includegraphics[width=\linewidth]{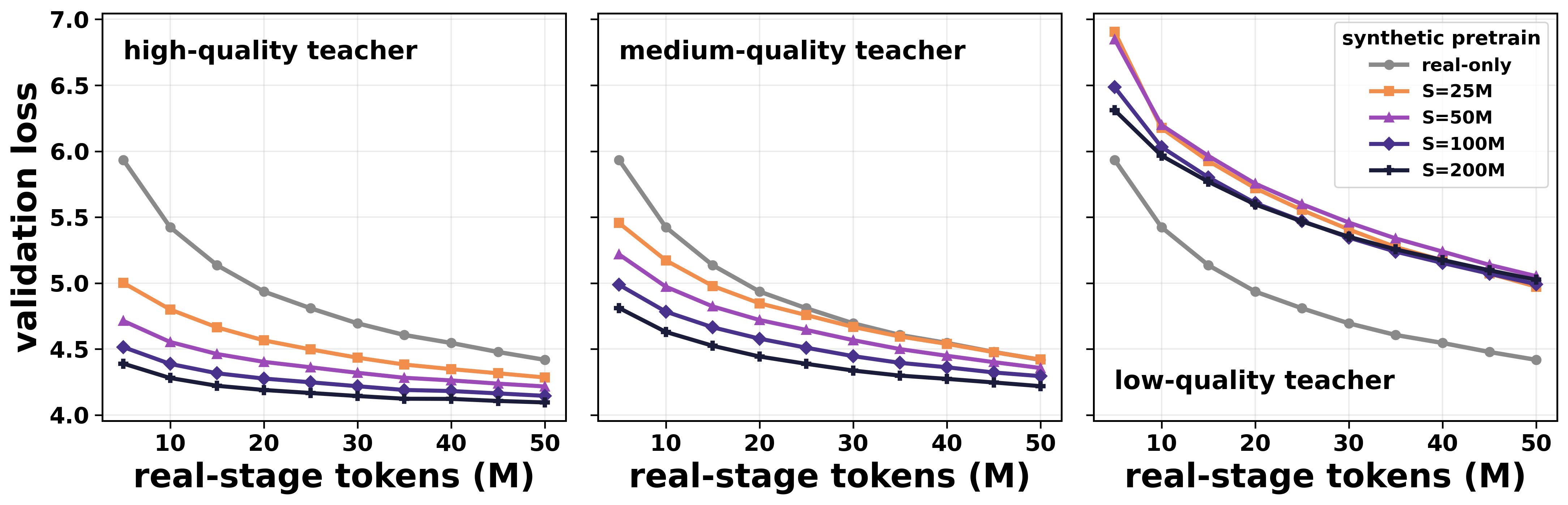}
        \caption{Validation loss after synthetic pretraining followed by fixed real-data training.}
        \label{fig:lm_teacher_quality_loss}
    \end{subfigure}

    \begin{subfigure}{0.85\linewidth}
        \centering
        \includegraphics[width=\linewidth]{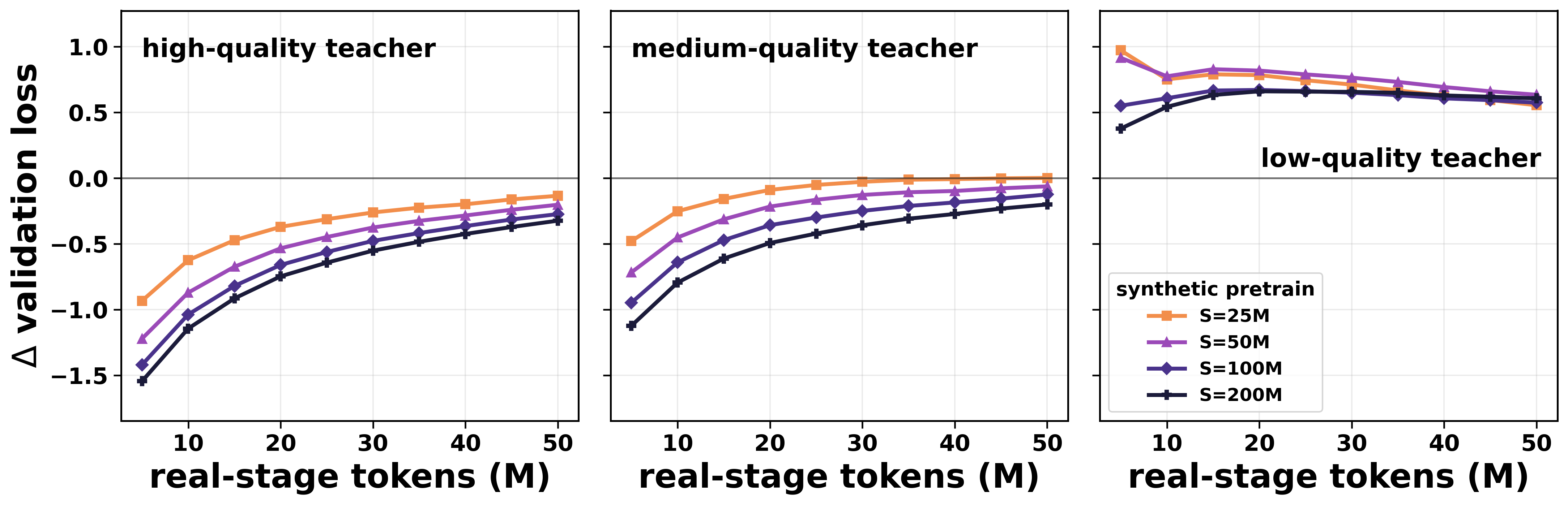}
        \caption{Validation-loss gap relative to the real-only curve at the same real-token budget.}
        \label{fig:lm_teacher_quality_delta}
    \end{subfigure}

    \caption{\textbf{Effect of synthetic teacher quality under a fixed real-data budget.}
    We compare high-, medium-, and low-quality synthetic corpora while keeping the downstream real-data budget fixed. High-quality synthetic pretraining consistently improves validation loss relative to real-only training, and the improvement grows with the synthetic pretraining budget. Medium-quality synthetic data provides weaker but still positive gains at larger synthetic budgets. In contrast, low-quality synthetic data leads to a positive loss gap, meaning that it hurts performance even after subsequent real-data training. These results validate the theoretical bias-quality trade-off: synthetic data can help when its induced bias is small enough to reduce optimization error, but poor-quality synthetic data introduces a mismatch that dominates and causes degradation.}
    \label{fig:lm_teacher_quality}
\end{figure}

\end{document}